\documentclass{article}
\usepackage{iclr2027_conference,times}

\usepackage{amsmath,amsfonts,bm}

\def\eqref#1{equation~\ref{#1}}
\def\Eqref#1{Equation~\ref{#1}}

\def\1{\bm{1}}

\DeclareMathAlphabet{\mathsfit}{\encodingdefault}{\sfdefault}{m}{sl}
\SetMathAlphabet{\mathsfit}{bold}{\encodingdefault}{\sfdefault}{bx}{n}

\newcommand{\E}{\mathbb{E}}

\newcommand{\R}{\mathbb{R}}

\newcommand{\Var}{\mathrm{Var}}

\newcommand{\Cov}{\mathrm{Cov}}

\DeclareMathOperator*{\argmax}{arg\,max}

\usepackage{amsmath}
\usepackage{amsfonts}
\usepackage{amssymb}
\usepackage{amsthm}
\usepackage{bm}
\usepackage{xcolor}
\usepackage{hyperref}
\hypersetup{hidelinks}
\usepackage{todonotes}
\usepackage{booktabs}
\usepackage{multirow}
\usepackage{algorithm}
\usepackage{algpseudocode}
\usepackage{bbm}
\usepackage{url}
\usepackage{enumitem}
\usepackage{makecell}
\usepackage{nicefrac}
\newtheorem{theorem}{Theorem}[section]

\newtheorem{corollary}[theorem]{Corollary}

\newtheorem{proposition}[theorem]{Proposition}

\usepackage{tikz}
\usetikzlibrary{
     arrows.meta,
     calc,
     decorations.pathreplacing,
     positioning
}
\usepackage{wrapfig}

\newcommand{\Sph}{\mathbb S^{D-1}}
\newcommand{\one}{\bm 1}
\providecommand{\tr}{\operatorname{tr}}
\providecommand{\rank}{\operatorname{rank}}
\providecommand{\R}{\mathbb R}
\providecommand{\E}{\mathbb E}
\providecommand{\EI}{\operatorname{EI}}
\providecommand{\Var}{\operatorname{Var}}
\providecommand{\argmax}{\operatorname*{arg\,max}}

\usepackage{comment}

\title{SABLE: Scalable Affine Bayesian Latent-space Exploration}

\author{
\begin{minipage}{\textwidth}
\centering
\vspace{1.5em}
\makebox[\textwidth][s]{%
Erik Bodin\textsuperscript{1,*}\hfill
Lisa Maria Kreusser\textsuperscript{1,*}\hfill
Donney Fan\textsuperscript{2,4}
}\\[0.8em]
\makebox[\textwidth][s]{%
Henry Moss\textsuperscript{3}\hfill
Geoff Pleiss\textsuperscript{2,4}\hfill
Nicolas Durrande\textsuperscript{1}
}\\[1.5em]
\normalfont
\hspace{-11em}
\makebox[\textwidth][c]{%
  \begin{tabular}[c]{@{}l@{\hspace{3em}}l@{}}
    \textsuperscript{1}Monumo &
    \textsuperscript{2}University of British Columbia \\[0.4em]
    \textsuperscript{3}Lancaster University &
    \textsuperscript{4}Vector Institute
  \end{tabular}%
  \rlap{%
    \hspace{2.0em}%
    {\small\textsuperscript{*}Equal contribution.}%
  }%
}
\end{minipage}
}

\iclrfinalcopy
\begin{document}

\maketitle

\begin{abstract}
Complex objects such as engineering designs and proteins are often difficult to optimise because their design spaces exhibit high-dimensional structure and objective evaluations are expensive.
Generative models provide continuous latent representations, but their latent spaces can have thousands to tens of thousands of dimensions, creating a regime where the latent dimension $D$ exceeds the evaluation budget.
We introduce SABLE, a Bayesian optimisation method for high-dimensional latent spaces with an isotropic distribution and limited budget.
SABLE uses a Bayesian affine surrogate for prediction and acquisition, treating higher-order structure as surrogate discrepancy.
After $N$ observations, the slope update has dimension at most $N-1$, yielding a decomposition of each query into an in-span component optimised by expected improvement (EI) in at most $N$ dimensions and an off-span component governed by a scalar norm.
We show that the off-span norm implicitly selected by affine EI can lead to underexploration, and SABLE instead prescribes it using an exploration schedule, without restricting search to a learned subspace. 
An efficient recompression scheme reduces information to
$\mathcal O(M)$ pseudo-observations, bringing the overall
time complexity per step to $\mathcal O(M^2D+M^3+D\log D)$ while keeping parameter inference, prediction, and
acquisition decisions to tens of milliseconds.
Across synthetic benchmarks and latent-space optimisation tasks, SABLE achieves high sample efficiency and reduces decision time by several orders of magnitude at large budgets.
\end{abstract}

\section{Introduction}
\label{sec:introduction}

Scientific and engineering optimisation problems often involve design domains that are difficult to search directly because they are high-dimensional, discrete, or governed by complex physical or chemical constraints. Pretrained generative models can provide simpler continuous latent representations of these domains, making such optimisation problems more tractable. This has enabled applications in protein and peptide design \citep{watson2023rfdiffusion, chen2024ampdiffusion, das2021accelerated}, molecular discovery \citep{chenthamarakshan2023accelerating}, turbulent flows \citep{li2024synthetic}, and particle-detector simulation \citep{amram2023calodiffusion}. In these applications, evaluating candidate quality is typically expensive, for example because it requires simulation, numerical analysis, or experimental validation, so only a limited number of proposals can be tested. Bayesian optimisation (BO) is well suited to this setting because it uses previous evaluations to guide subsequent candidates \citep{garnett2023bayesian}. Applying BO in the latent space of a pretrained generative model is referred to as \emph{Latent Space Bayesian Optimisation} (LSBO) \citep{gomez2018automatic,kusner2017grammar,griffiths2020constrained,tripp2020sampleefficient,maus2022local,moss2025return}. 
Some approaches train task-specific generative models \citep{krishnamoorthy2023}, fine-tune pretrained generators for the objective \citep{fan2023,denker2025,yun2025}, or steer generation using gradients or learned objective models \citep{Gruver2023,yuan2025paretoflow,yun2025}; we instead keep the pretrained generator fixed and use only objective evaluations.

The latent spaces of modern scientific generative models are typically
high-dimensional. For example, a diffusion model for Lagrangian turbulence
\citep{li2024synthetic} has $D=6{,}000$ latent variables, while a
CaloDiffusion model \citep{amram2023calodiffusion} has $D=40{,}500$.
These dimensions can greatly exceed the number $N$ of affordable objective evaluations.
In the regime $N < D$, without additional structural assumptions, the data cannot identify even a global linear model, let alone a flexible global surrogate.
High-dimensional BO~\citep{binois2022survey} often addresses this mismatch through random embeddings
or adaptive subspaces
\citep{wang2013bayesian,papenmeier2022increasing},
additive or sparse models
\citep{kandasamy2015high,eriksson2021high},
or local modelling
\citep{eriksson2019scalable}.
We instead propose SABLE (\emph{Scalable Affine Bayesian Latent-space
Exploration}), which does \emph{not} restrict admissible queries to a learned
low-dimensional subspace and uses a Bayesian affine \emph{working model}
inferred from limited data.
Higher-order variation is treated as surrogate discrepancy, with the variance hyperparameters inferred from data.

Recent work develops high-dimensional affine BO and fast spherical affine EI~\citep{doumont2026,costofsampling}.
However, we show that affine EI can underexplore on
non-affine objectives, and posterior refitting still incurs
$\mathcal O(N^2D)$ cost per round when $N<D$,
becoming prohibitively expensive
at large evaluation budgets. 
SABLE addresses both limitations, improving sample
efficiency and reducing decision time by over four
orders of magnitude at large evaluation budgets.

SABLE exploits the fact that high ambient dimension does not imply high-dimensional data-informed geometry. Conditional on the variance hyperparameters, $N$ observations update the posterior of the affine slope vector in at most $N-1$ dimensions, while the orthogonal directions retain their prior mean and variance. Likewise, the evaluated designs span at most $N$ dimensions. 
This geometry yields an in-span/off-span decomposition of each query. Under the affine working model, off-span exploration is governed by a scalar, the off-span norm, while EI optimisation over the in-span component requires at most $N$ dimensions, independently of $D$.
Standard affine EI implicitly determines this off-span norm, which we show is often driven close to zero. SABLE instead prescribes the norm using an exploration schedule, separating exploration control from the fitted variances and any higher-order surrogate discrepancy they may absorb. The fitting cost would still grow quickly with $N$, so we introduce a fast recompression scheme that reduces information to $\mathcal O(M)$ pseudo-observations, along with exploiting the geometry to avoid costly orthogonalisation. For fixed $M$, this bounds the posterior-fitting representation and per-round model-refitting cost independently of $N$, while keeping parameter inference, prediction, and acquisition decisions to tens of milliseconds across the dimensionalities studied, without restricting queries to a learned subspace.  Figure~\ref{fig:geometry} summarises the decoupled off-span exploration and bounded-cost model refitting.

\begin{figure*}[ht]
\centering

\begin{tikzpicture}[
    >=Latex,
    font=\small,
    vec/.style={-{Latex[length=2mm]}, line width=0.9pt},
    mainvec/.style={-{Latex[length=2.3mm]}, line width=1.1pt},
    ambient/.style={draw=gray!45, fill=gray!4, line width=0.7pt},
    span/.style={draw=blue!60!black, fill=blue!9, line width=0.75pt},
    retained/.style={draw=blue!65!black, fill=blue!35, line width=0.5pt},
    augment/.style={draw=orange!80!black, fill=orange!25, line width=0.5pt},
    note/.style={font=\scriptsize, align=center},
    sidenote/.style={font=\scriptsize, align=left},
    paneltitle/.style={font=\bfseries\footnotesize, align=left},
]

\begin{scope}[shift={(0,0)}]

    \node[paneltitle, anchor=west] at (0,3.55)
        {(a) Decoupled off-span exploration};

    \coordinate (O) at (3.00,2.05);
    \coordinate (P) at (3.84,2.05);
    \coordinate (Q) at (3.84,2.70);

    \draw[draw=gray!40, line width=0.7pt]
        (O) circle (1.10);

    \path[span]
        (0.48,1.46) --
        (4.66,1.46) --
        (4.97,2.15) --
        (0.79,2.15) -- cycle;

    \fill (O) circle (1.2pt);

    \node[note, text=blue!65!black]
        at (1.20,2.975)
        {EI chooses\\
         in-span direction};

    \node[note, text=blue!65!black]
        at (2.2,1.68)
        {$\mathcal V_N = \operatorname{span}\{\bm s_1,\ldots,\bm s_N\}$};

    \draw[vec, blue!70!black]
        (O) -- (P);

    \draw[vec, orange!85!black]
        (P) -- (Q);

    \draw[mainvec]
        (O) -- (Q)
        node[pos=0.62, above left=-1pt]
        {$\bm s$};

    \draw[line width=0.65pt]
        ($(P)+(-0.08,0)$)
        -- ++(0,0.08)
        -- ++(0.08,0);

    \draw[-{Latex[length=1.4mm]}, blue!55!black]
        (1.35,2.7)
        to[out=270,in=200]
        (3.3,2.0);

    \node[note, text=orange!85!black]
        at (5.34,2.98)
        {schedule sets off-span norm,\\
         frame sets off-span direction};

    \draw[-{Latex[length=1.4mm]}, orange!70!black]
        (4.70,2.78)
        to[out=235,in=25]
        (3.89,2.35);

    \node[note, text=gray!70!black]
        at (1.60,1.0)
        {$D$ dimensions};

\end{scope}

\begin{scope}[shift={(7.2,0)}]

    \node[paneltitle, anchor=west] at (0,3.55)
        {(b) Bounded-cost model refitting};

    \node[note, text=gray!70!black]
        at (0.78,2.72)
        {grows with $N$};

    \draw[gray!60, line width=0.7pt]
        (0.16,0.90) rectangle (1.40,2.48);

    \foreach \y in {
        1.03,1.16,1.29,1.42,1.55,1.68,
        1.81,1.94,2.07,2.20,2.33,2.46
    }{
        \draw[gray!50, line width=0.35pt]
            (0.16,\y) -- (1.40,\y);
    }

    \draw[mainvec, gray!70]
        (1.92,1.68) -- (2.48,1.68);

    \node[note] at (2.20,1.96)
        {recompress};

    \node[note]
        at (3.52,2.78)
        {$\mathcal O(M)$\\
         pseudo-observations};

    \draw[gray!60, line width=0.7pt]
        (2.88,1.06) rectangle (4.16,2.42);

    \fill[retained]
        (2.88,2.12) rectangle (4.16,2.31);
    \fill[retained]
        (2.88,1.87) rectangle (4.16,2.06);
    \fill[retained]
        (2.88,1.62) rectangle (4.16,1.81);

    \fill[augment]
        (2.88,1.33) rectangle (4.16,1.49);
    \fill[augment]
        (2.88,1.11) rectangle (4.16,1.27);

    \draw[gray!45, line width=0.35pt]
        (2.88,2.12) -- (4.16,2.12);
    \draw[gray!45, line width=0.35pt]
        (2.88,1.87) -- (4.16,1.87);
    \draw[gray!45, line width=0.35pt]
        (2.88,1.62) -- (4.16,1.62);
    \draw[gray!45, line width=0.35pt]
        (2.88,1.49) -- (4.16,1.49);
    \draw[gray!45, line width=0.35pt]
        (2.88,1.27) -- (4.16,1.27);

    \fill[retained]
        (4.38,1.62) rectangle (4.56,1.74);
    \node[font=\scriptsize, anchor=west, text=blue!65!black]
        at (4.64,1.68)
        {strong retained};
    
    \fill[augment]
        (4.38,1.34) rectangle (4.56,1.46);
    \node[font=\scriptsize, anchor=west, text=orange!85!black]
        at (4.64,1.40)
        {weaker summarised};

\end{scope}

\end{tikzpicture}

\caption{
\textbf{Overview of SABLE.}
(a) Each query is decomposed into an in-span component and
an orthogonal off-span component. EI chooses the in-span
direction, while the exploration schedule sets the off-span
norm and a randomised orthogonal frame supplies its direction.
Because the design span $\mathcal V_N$ has dimension at most
$N$, EI optimisation requires at most $N$ coordinates,
independently of $D$.
(b) As observations accumulate, their representation is recompressed to  $\mathcal O(M)$ pseudo-observations. Strong information is retained and
weaker information is summarised, making model-refitting cost independent of
$N$ for fixed $M$.
}
\label{fig:geometry}
\end{figure*}

Our contributions are:
\begin{itemize}[leftmargin=1em]

\item \textbf{Decoupled off-span exploration.}
We exploit the in-span/off-span geometry of affine EI,
which reduces acquisition optimisation to at most $N$
dimensions, independently of $D$, and makes off-span
orientations predictively equivalent.
This reveals a structural failure mode of EI: its implicit
off-span allocation can cause severe underexploration
on non-affine objectives.
SABLE instead prescribes this allocation through an
exploration schedule, significantly improving sample
efficiency.
Off-span invariance also enables a randomised
orthogonal frame to generate exploration directions
orthogonal to previous queries by construction,
avoiding costly orthogonalisation.

\item \textbf{Bounded-cost model refitting.}
We use fast recompression to reduce observations
to $\mathcal O(M)$ pseudo-observations, making posterior-fitting
storage and per-round model-refitting cost depend on $M$
rather than $N$.
Combined with the orthogonal frame, this keeps the per-step computational cost bounded independently of $N$ for $N\leq D$, enabling scalable optimisation at
large budgets.

\end{itemize}

\section{Problem setup and affine working model}\label{sec:model}

\paragraph{Optimisation problem.}
Let $g:\R^D\to\mathcal X$ be a pretrained generative model and $f:\mathcal X\to\R$ an expensive black-box objective, so that the latent objective is $f(g(\bm z))$. We assume the generative model was trained with an isotropic Gaussian latent prior $\bm z\sim\mathcal N(\bm0,I_D)$. 
Under this prior, $\|\bm z\|/\sqrt D\to1$ in probability as $D\to\infty$
\citep{cover2006elements,vershynin2026hdp,nalisnick2019typicality},
so Gaussian latents concentrate at the characteristic radius $\sqrt{D}$.
This thin-shell geometry is also practically relevant: atypical latent norms can degrade decoded samples~\citep{samuel2023normguided}, while recent LSBO work finds that restricting search to the characteristic-radius sphere keeps queries in decoder-supported regions without materially reducing generative range
~\citep{bodin2025linear,willis2025sample,moss2025return,costofsampling}. 
We therefore re-cast the optimisation problem as one over the unit sphere  
\(
\Sph:=\{\bm s\in\R^D:\|\bm s\|=1\}
\)
and define the objective as
\[
F(\bm s):=f(g(\sqrt{D}\bm s)),
\qquad
\bm s^\star\in\argmax_{\bm s\in\Sph}F(\bm s).
\]
After $N$ evaluations, write $\mathcal D_N=\{(\bm s_i,y_i)\}_{i=1}^N$, where $y_i=F(\bm s_i)$, $X_N=(\bm s_1,\ldots,\bm s_N)^\top\in\R^{N\times D}$, and $\bm y_N=(y_1,\ldots,y_N)^\top$. Throughout, the evaluation budget $N_{\max}$ is specified in advance.

\paragraph{Affine working model.}
The degree-zero and degree-one spherical-harmonic components of $F$ give
\(
F(\bm s)=\beta_0+\bm v_\star^\top\bm s+\Delta(\bm s),
\)
where $\Delta$ collects higher-order structure (Appendix~\ref{app:objective-projection}).
Restricting optimisation to the sphere also removes radial variation in the
input norm, avoiding the boundary-seeking behaviour of linear BO on
conventional high-dimensional domains \citep{doumont2026}.
We use the affine component as a \emph{working model}, providing posterior predictions and acquisitions for BO without assuming that
$F$ is globally affine,
and treat $\Delta$ as surrogate discrepancy:
\[
G(\bm s)=\beta+\bm v^\top\bm s,
\qquad
\bm v\mid\sigma_1^2\sim\mathcal N(\bm0,\sigma_1^2I_D),
\qquad
p(\beta)\propto1,
\]
with likelihood
\(
\bm y_N\mid\beta,\bm v,\sigma_\varepsilon^2
\sim\mathcal N(\beta\one_N+X_N\bm v,\sigma_\varepsilon^2 I_N).
\)
For deterministic objectives, $\sigma_\varepsilon^2$ represents this discrepancy, and independent Gaussian observation noise adds variance to it, with more general noise models in Appendix~\ref{app:general-residual-covariance}.

\section{SABLE: finite-data geometry and controlled exploration}
\label{sec:method}

After $N$ observations, the affine working model updates the slope posterior
in at most $N-1$ directions, independently of $D$.
This reduces EI optimisation to at most $N$ coordinates, with predictively
equivalent off-span orientations.
SABLE schedules the off-span norm to prevent underexploration, while using
an orthogonal frame and $O(M)$ pseudo-observation recompression to keep
per-step cost bounded for $N\le D$.
Algorithm~\ref{alg:overview} summarises the method.

\begin{algorithm}[ht]
\caption{Overview of SABLE}
\label{alg:overview}
\begin{algorithmic}[1]
\Require Evaluation budget $N_{\max}$, compression budget $M$,
         exploration schedule $\{t_N\}_{N=1}^{N_{\max}-1}$
\State Set $N=1$, evaluate an initial orthogonal-frame direction
       $\bm s_1$, and store the observation.
\While{$N<N_{\max}$}
    \State Fit the Bayesian affine model using the current
           observation representation.
    \State Choose $\bm w_N^\star$ by maximising reduced EI
           at the prescribed off-span norm $t_{N}$.
    \State Select the next unused off-span frame direction $\bm{e}$ and reconstruct $\bm s_{N+1}\in\mathbb S^{D-1}$.
    \State Evaluate $\bm s_{N+1}$ and increment $N$.
    \If{the recompression criterion is met}
        \State Recompress to $\mathcal O(M)$
               pseudo-observations and update the
               sufficient statistics.
    \EndIf
\EndWhile
\State \Return the best evaluated candidate.
\end{algorithmic}
\end{algorithm}

\subsection{Finite-data posterior geometry}
\label{sec:posterior-geometry}

After marginalising the intercept, $N$ observations modify the posterior in at most $N-1$ directions.

\begin{proposition}[Posterior geometry, proof in Appendix~\ref{app:data-informed-geometry}]
\label{prop:explored-span}
Let
$\bm\mu_N:=N^{-1}\sum_{i=1}^N\bm s_i$ and
$$
\mathcal S_N
:=
\operatorname{span}\{\bm s_i-\bm\mu_N:i=1,\ldots,N\}\, .
$$
Let $P_N$ be the orthogonal projector onto $\mathcal S_N$. Then, conditionally on
$\bm\theta=(\sigma_1^2,\sigma_\varepsilon^2)$, the posterior mean $\bm m_N(\bm\theta)$ is in $\mathcal S_N$ and its covariance $\Sigma_N(\bm\theta)$ satisfies
$$
\Sigma_N(\bm\theta)
=
\sigma_1^2(I_D-P_N)
+
P_N\Sigma_NP_N\, .
$$
Hence the observations modify the slope posterior only within
$\mathcal S_N$ which has dimension at most $N-1$; on $\mathcal S_N^\perp$ the posterior retains
zero mean and prior variance $\sigma_1^2$ in every unit direction.
\end{proposition}

Thus, all orientations in $\mathcal S_N^\perp$
remain exchangeable under the posterior. We
exploit this symmetry to reduce acquisition
dimension without restricting queries to the informed subspace.
Appendix~\ref{app:post-data-rotation} formalises rotational invariance,
and Appendix~\ref{app:effective-informed-dimension} quantifies the strength
of informed updates.

\subsection{Reduced acquisition geometry}
\label{sec:datasized-bo}

The data-informed subspace $\mathcal S_N$ describes
which slope directions observations can update.
Acquisition optimisation is more naturally expressed
relative to the design span
\[
\mathcal V_N
:=
\operatorname{span}\{\bm s_1,\ldots,\bm s_N\},
\]
whose dimension satisfies $d_N\leq N$.
Let $V\in\R^{D\times d_N}$ have orthonormal columns
spanning $\mathcal V_N$.
For $d_N<D$, every query $\bm s\in\Sph$ can be
decomposed as
\begin{equation}
\label{eq:decomposition}
\bm s=V\bm w+t\bm e,
\qquad
\bm e\in\mathcal V_N^\perp,\quad \|\bm e\|=1,
\qquad
\|\bm w\|^2+t^2=1,
\end{equation}
where $\bm w=V^\top\bm s$ gives the in-span
coordinates, $t=\sqrt{1-\|\bm w\|^2}$ the off-span
norm, and $\bm e$ the off-span orientation.

Since $\mathcal V_N^\perp\subseteq\mathcal S_N^\perp$,
the slope posterior is rotationally invariant on
$\mathcal V_N^\perp$.
Rotating $\bm e$ within this complement therefore
leaves the predictive distribution unchanged.
Prediction depends on the off-span component only
through $t$, so $\bm e$ need not be optimised.
Acquisition optimisation can consequently be performed
in at most $d_N\leq N$ coordinates rather than the ambient dimension $D$; see
Appendix~\ref{app:working-basis}.

\paragraph{Decomposition of EI.}
For fixed hyperparameters
$\bm\theta=(\sigma_1^2,\sigma_\varepsilon^2)$,
let $m_N(\bm s)$ and $s_N(\bm s)$ denote the posterior
predictive mean and standard deviation, respectively.
By rotational invariance on $\mathcal V_N^\perp$,
both depend on $\bm s$ only through $\bm w$, since
$t=\sqrt{1-\|\bm w\|^2}$.
We may therefore write
$m_N(\bm s)=\tilde m_N(\bm w)$ and
$s_N(\bm s)=\tilde s_N(\bm w)$.

Let $\bm w_i=V^\top\bm s_i$,
$\tau_N=\max_{1\leq i\leq N}\tilde m_N(\bm w_i)$,
and $\gamma_N(\bm w)=\tilde m_N(\bm w)-\tau_N$.
With $\Phi$ and $\phi$ denoting the standard normal
CDF and density, respectively, EI becomes
\begin{equation}
\label{eq:reduced-ei}
\EI_N(\bm s)
=
\gamma_N(\bm w)
\Phi\left(
\frac{\gamma_N(\bm w)}{\tilde s_N(\bm w)}
\right)
+
\tilde s_N(\bm w)
\phi\left(
\frac{\gamma_N(\bm w)}{\tilde s_N(\bm w)}
\right)
=:
\widetilde{\EI}_N(\bm w).
\end{equation}
Thus, maximising EI on the $D$-sphere is equivalent to
optimising over at most $d_N$ coordinates:
\[
\max_{\bm s\in\Sph}\EI_N(\bm s)
=
\max_{\|\bm w\|\leq 1}
\widetilde{\EI}_N(\bm w).
\]
Given a maximiser $\bm w_N^{\widetilde{\EI}}$ of the
right-hand side, a maximiser on the
sphere is reconstructed as
\[
\bm s_{N+1}
=
V\bm w_N^{\widetilde{\EI}}
+
\sqrt{1-\|\bm w_N^{\widetilde{\EI}}\|^2}\,\bm e,
\qquad
\bm e\in\mathcal V_N^\perp,\quad \|\bm e\|=1.
\]
The off-span orientation $\bm e$ is arbitrary,
as all such directions are predictively equivalent. 
This invariance also holds under hyperparameter
marginalisation (Corollary~\ref{cor:post-data-rotation}).
When approximately marginalising variance hyperparameters, we average
conditional EI over posterior hyperparameter samples
(Appendix~\ref{app:atoms}), using a common incumbent
defined by the marginal posterior mean
(Appendix~\ref{app:norm-rays}).
The acquisition reduction remains unchanged.

\subsection{Decoupled off-span exploration}

The decomposition above reveals that affine EI implicitly
determines the off-span norm $t$ through its optimisation
over $\|\bm w\|\leq 1$, while remaining invariant to the
off-span orientation $\bm e$.
The norm selected by EI depends on the surrogate uncertainty.
Conditioning reduces uncertainty along data-informed
directions, whereas directions in $\mathcal V_N^\perp$
retain their prior variance.
The resulting uncertainty contrast depends on the fitted
variance ratio $\lambda:=\sigma_1^2/\sigma_\varepsilon^2$.
Since higher-order objective variation contributes to
$\sigma_\varepsilon^2$ as surrogate discrepancy, this
ratio need not provide a reliable basis for controlling
exploration, and may lead to severe underexploration
on non-affine objectives; see
Appendix~\ref{app:effective-informed-dimension}.

\subsubsection{Exploration schedule}
\label{sec:exploration_schedule_sub}
SABLE decouples the off-span allocation from EI
by prescribing norms $t_N\in[0,1]$ for
$N=1,\ldots,N_{\max}-1$, where $t_N$ controls
the next query $\bm s_{N+1}$.
EI then selects the in-span component subject to
\begin{equation}
\label{eq:scheduled-ei}
\bm w_N^\star
\in
\argmax_{\|\bm w\|=\sqrt{1-t_N^2}}
\widetilde{\EI}_N(\bm w).
\end{equation}
The query is reconstructed using $\bm w_N^\star$
and $t_N$ as in \eqref{eq:decomposition}.
Thus, the EI criterion remains unchanged, while
the off-span allocation is prescribed separately.
Larger values of $t_N$ favour exploration of new
directions, whereas smaller values favour exploitation
within the existing design span.
The decoupling does not depend on a particular schedule.
We use a precomputed value-of-information (VoI) proxy
(Appendix~\ref{app:voi-schedule}), independent of the
fitted variance scales, with $\lambda_{\mathrm{ref}}$
set to a low or high reference value depending on
whether evaluations are noisy.
We leave further ablations and schedule design for future work.

\subsubsection{Fast orthogonal frame construction}
\label{sec:orthogonal-exploration}

The invariance of affine EI to off-span orientation
allows us to select exploration directions from an
orthogonal frame without affecting the acquisition value.
After the initial query, each subsequent query is
constructed from previously used frame directions,
augmented by a fresh orthogonal direction when $t_N>0$.
Consequently, unused frame columns remain orthogonal
to all previous queries by construction.
For the bounded-state implementation, we exploit
this property using a seeded randomised frame based
on the type-II discrete cosine transform (DCT-II).
This generates fresh directions in
$\mathcal O(D\log D)$ time with $\mathcal O(D)$
stored state, without explicitly orthogonalising
against previous queries as $N$ grows.
Appendix~\ref{app:orthogonal-frame} details the
construction and its properties.

\subsection{Bounded-cost model refitting}
\label{sec:bounded-state}

Although posterior inference and acquisition optimisation
are data-sized, the fitting representation grows with $N$,
making repeated model refitting computationally costly.
We control this growth by recompressing accumulated
observations into $\mathcal O(M)$ pseudo-observations.
Compression is performed in observation space:
since the affine likelihood has conditional mean
$\beta\one_N+X_N\bm v$, the design matrix,
response vector, and intercept column are compressed
together,
\[
(X_N,\bm y_N,\one_N)
\quad\longrightarrow\quad
(\check X,\check{\bm y},\check{\bm o}).
\]

Recompression exactly preserves the centred
slope--response and intercept cross-statistics,
while discarded spectral information is summarised
by pseudo-observations.
Spectral analysis in
Appendix~\ref{app:low-rank-posterior-proof}
motivates retaining directions with stronger data
support, which produce the largest posterior updates
and have the greatest effect on EI.
The representation has size $\mathcal O(M)$
independently of $N$ and can be reused across
variance-hyperparameter values.
A log-determinant correction accounts for discarded
spectral information in approximate marginal-likelihood
evaluation; see Appendix~\ref{app:bounded-refitting}.

Appendix~\ref{app:bounded-refitting-stability}
establishes stability bounds for the conditional
posterior and EI under repeated recompression,
including extensions to hyperparameter-marginalised EI
with approximate marginal-likelihood weights.
These guarantees are conditional on the realised
compression and log-evidence errors and do not cover
the finite-atom or candidate-generation approximations.
For fixed $M$, the posterior-fitting representation
and per-round model-refitting cost remain independent
of $N$.

\subsection{Computational scaling}
\label{sec:complexity}

SABLE addresses two computational bottlenecks beyond
fast acquisition optimisation: the growing cost of
posterior refitting and the cost of constructing fresh
orthogonal exploration directions.
Observation-space recompression bounds the
posterior-fitting representation, while the randomised
orthogonal frame avoids orthogonalisation against
previous queries.
These improvements complement the exploration schedule,
which targets sample rather than computational efficiency.

Without compression, posterior fitting requires
$\mathcal O(ND)$ storage and
$\mathcal O(N^2D+N^3)$ time per refit.
With compression budget $M$, these become
$\mathcal O(MD)$ and
$\mathcal O(M^2D+M^3)$, respectively.
Thus, for fixed $M$, posterior-fiFinite-Data Geometry and Controlled Explorationtting storage
and refitting cost are independent of $N$
once recompression is active.
For $D\gg M$, compressed refitting is dominated
by $\mathcal O(M^2D)$, compared with
$\mathcal O(N^2D)$ without compression.

The orthogonal frame uses $\mathcal O(D)$ state and
generates fresh directions in $\mathcal O(D\log D)$ time.
The EI incumbent uses an $\mathcal O(NM)$ ledger,
maintained in $\mathcal O(NM^2)$ time per round
(Appendix~\ref{app:frame-ledger}).
For $N\leq D$, the overall per-round complexity is
$\mathcal O(M^2D+M^3+D\log D)$ time and
$\mathcal O(MD)$ memory, for fixed numerical settings.

\section{Experiments}
\label{sec:experiments}

We evaluate SABLE from its individual components to the complete method:
state compression, prescribed off-span exploration, and scientific
latent-space optimisation. Unless otherwise stated, SABLE marginalises the
variance scales, uses $M=300$, $\lambda_\text{ref}=10^5$, and optimises on the $\sqrt D$-rescaled unit
sphere. The primary optimisation metric is the best observed value; 
curves show means with $5$th--$95$th percentiles. Experimental settings
and additional results are given in Appendices~\ref{app:experimental-setup},~\ref{app:experiments}.

\paragraph{State compression.}

On ZonalFunctionBench (Appendix~\ref{app:zonal-bench}) at $D=10{,}000$,
we vary only the compression budget $M$ across an affine control and
three objectives with mild, moderate, and severe non-affine
structure, using ten repeats. ZonalFunctionBench provides random spherical objectives
with controlled higher-order structure and no preferred coordinate system. All
variants use marginalisation and the prescribed off-span schedule.
\begin{figure}[ht]
\centering
\includegraphics[width=\textwidth]{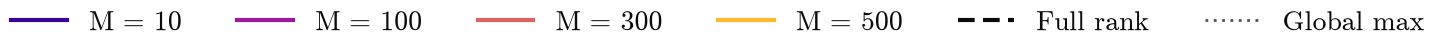}
\vspace{-0.9em}
\begin{minipage}[t]{0.598\textwidth}
\vspace{0pt}
\centering
\setlength{\tabcolsep}{0pt}%
\renewcommand{\arraystretch}{0}%
\begin{tabular}{@{}c@{}c@{}}
\includegraphics[width=0.5\linewidth]{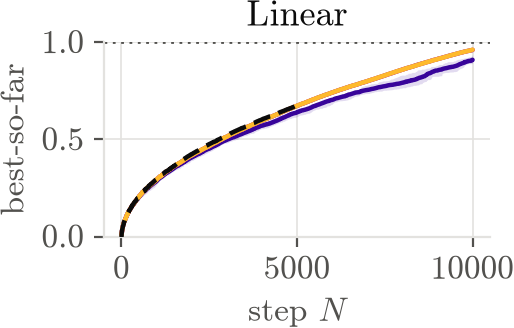} &
\includegraphics[width=0.5\linewidth]{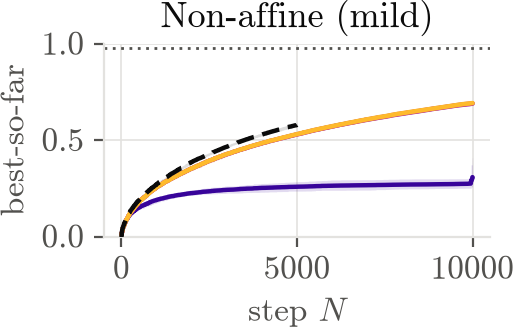} \\
\vspace{0.5em} \\
\includegraphics[width=0.5\linewidth]{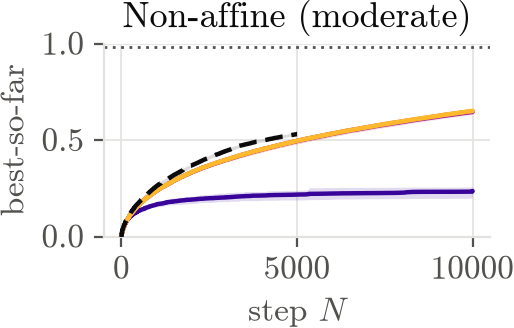} &
\includegraphics[width=0.5\linewidth]{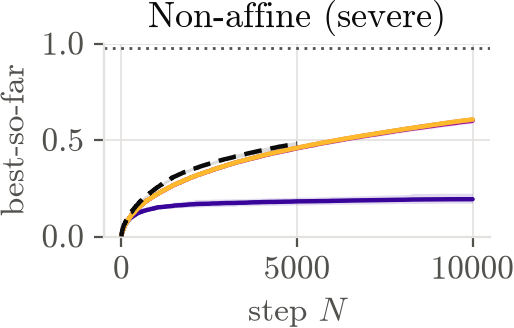} \\
\end{tabular}
\end{minipage}%
\hfill
\begin{minipage}[t]{0.387\textwidth}
\centering
\vspace{0.7em}
\includegraphics[width=\linewidth]{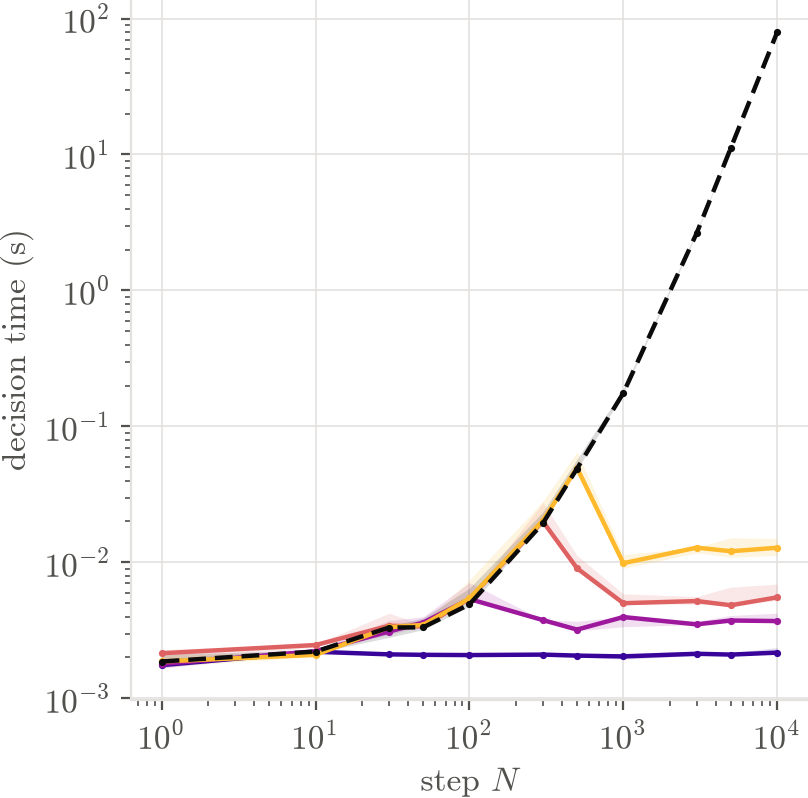}
\end{minipage}
\vspace{1em}
\caption{
\textbf{Effect of state compression.}
\emph{Left:} best observed value for compression budgets
$M\in\{10,100,300,500\}$ at $D=10{,}000$, compared with the uncapped
representation (dashed) and 
the mean (exact or estimated) maximum of the ten objective draws. 
Lines show means over ten runs (and objective draws) and bands show
$5$th--$95$th percentiles.
\emph{Right:} wall-clock time per decision.
}
\label{fig:c3-rank}
\end{figure}
Figure~\ref{fig:c3-rank} shows that at $N=5{,}000$, the largest shared
budget, the $M\in\{100,300,500\}$ variants differ from one another by at
most $0.003$ in best-so-far value. Relative to the uncapped representation,
their mean gaps are $0.003$, $0.047$, $0.035$, and $0.022$ from the affine
through severe cases, with no systematic growth over $N\in[1000,5000]$.
By contrast, $M=10$ is too
aggressive, trailing the uncapped representation by $0.30$--$0.32$ on the
non-affine objectives.
The compressed variants continue to $N=D=10{,}000$, providing a
long-horizon stress test. 
These results support the compression approach of
Section~\ref{sec:bounded-state}: retaining strongly data-informed directions is sufficient to preserve optimisation performance.

\paragraph{Separating exploration from model fit.}
We now examine the exploration behaviour of standard
affine EI on the four controlled objectives at $N=D=1000$,
using MLE, MAP, and hyperparameter marginalisation.
All runs use the uncapped representation without
compression to isolate the effects of model fitting
and exploration.
Figure~\ref{fig:affine_ei_underexplores} shows the
off-span norm selected by EI and the corresponding
fitted discrepancy variance scale.

\begin{figure}[ht]
    \centering
    \includegraphics[width=0.846\textwidth]
    {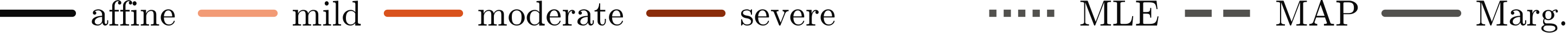}\\[0.6ex]
    \vspace{0.3em}
    \includegraphics[width=0.49\textwidth]
    {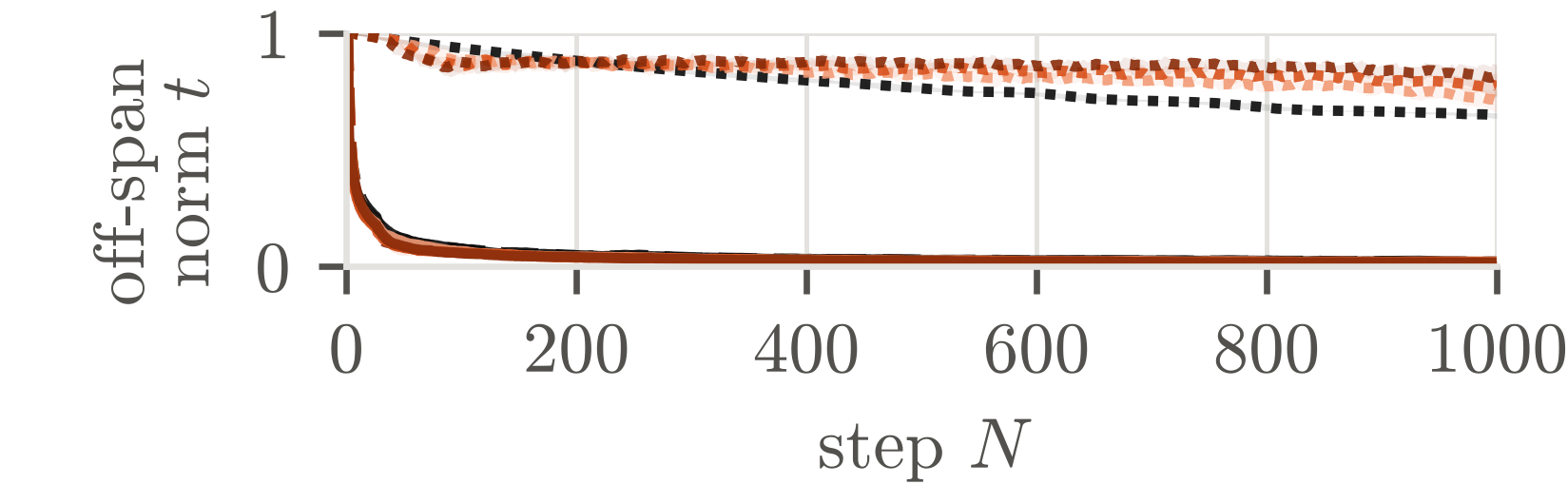}\hfill
    \includegraphics[width=0.49\textwidth]
    {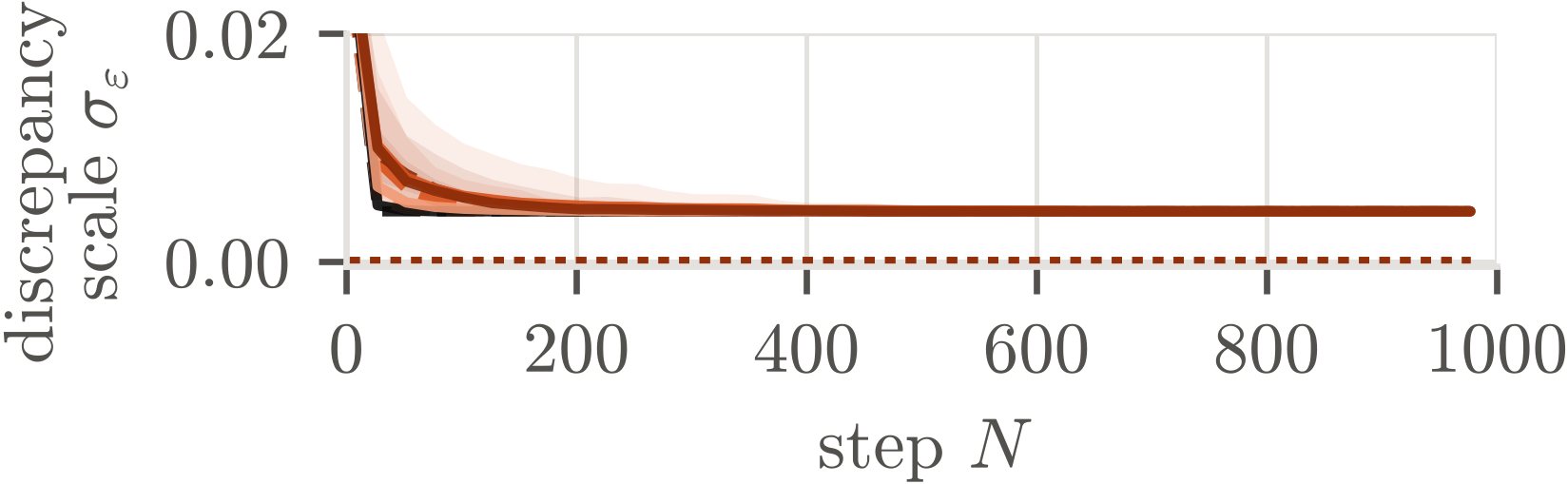}
    \caption{
    \textbf{Implicit exploration under standard affine EI.}
    \emph{Left:} off-span norm $t$ selected by standard affine EI.
    \emph{Right:} fitted discrepancy scale
    $\sigma_\varepsilon$.
    Colours distinguish the affine control and three
    levels of non-affine variation; line styles indicate
    MLE, MAP, and hyperparameter marginalisation.
    }
    \label{fig:affine_ei_underexplores}
\end{figure}

Under MLE, the fitted discrepancy scale approaches zero
even on non-affine objectives, consistent with surrogate
overfitting. EI consequently allocates most query norm
to fresh off-span directions.
MAP and marginalisation exhibit the opposite behaviour:
the off-span norm rapidly approaches zero, largely
confining exploration to the span established early.
These contrasting behaviours illustrate how EI's
implicit exploration allocation can be distorted by
fitted surrogate uncertainty.

To test whether decoupling resolves this problem, we
cross the three variance-scale treatments with three
exploration rules: standard EI (no schedule), a naive
schedule using the fitted signal-to-discrepancy ratio,
and the VoI schedule with $\lambda_{\mathrm{ref}}=10^5$,
as prescribed for these smooth, noise-free objectives
by Appendix~\ref{sec:local-resolvability}.
The MAP configuration with standard EI and no compression
recovers the method of \cite{costofsampling}.
Table~\ref{tab:c2-degree-one-main} reports attained
values and alignment with the reference degree-one
direction, measured by $\cos(\hat{\bm s},\bm v_\star)$.

\begin{table}[ht]
\centering
\caption{
\textbf{Model fitting and off-span exploration} at $N=D=1000$ with the
uncapped representation. 
\emph{Left:} alignment with the reference degree-one
direction, $\cos(\hat{\bm s},\bm v_\star)$. \emph{Right:} attained value: maximum of the best observed point and the posterior-mean recommendation.
Entries are means over ten seeds. Bold: the column best and every
entry within two standard deviations.
}
\label{tab:c2-degree-one-main}
\footnotesize
\begin{tabular}{llcccccccc}
\toprule
& & \multicolumn{4}{c}{$\cos(\hat{\bm s},\bm v_\star)$}
  & \multicolumn{4}{c}{attained value} \\
\cmidrule(lr){3-6}\cmidrule(lr){7-10}
& & affine & mild & mod. & severe & affine & mild & mod. & severe \\
\midrule
\multirow{3}{*}{Standard EI (no schedule)}
& MLE   & \textbf{1.00} & -0.00 & -0.01 & -0.01 & \textbf{1.00} & 0.48 & 0.32 & 0.16 \\
& MAP   & 0.78 & 0.44 & 0.30 & 0.16 & 0.78 & 0.60 & 0.54 & 0.50 \\
& Marg. & 0.77 & 0.44 & 0.29 & 0.19 & 0.77 & 0.59 & 0.54 & 0.49 \\
\midrule
\multirow{3}{*}{
\makecell[l]{naive model-derived\\schedule: $\lambda_k = \nicefrac{\sigma_{1, k}^2}{\sigma_{\varepsilon, k}^2}$}
}
& MLE   & 0.13 & 0.00 & 0.00 & 0.01 & 0.49 & 0.49 & 0.36 & 0.23 \\
& MAP   & 0.98 & \textbf{0.59} & \textbf{0.41} & \textbf{0.23} & 0.98 & 0.67 & 0.55 & 0.45 \\
& Marg. & 0.98 & \textbf{0.59} & \textbf{0.42} & \textbf{0.25} & 0.98 & 0.66 & 0.56 & 0.45 \\
\midrule
\multirow{3}{*}{
\makecell[l]{VoI schedule: $\lambda_{\mathrm{ref}}=10^5$}
}
& MLE   & 0.90 & -0.00 & 0.01 & 0.00 & \textbf{1.00} & 0.55 & 0.37 & 0.24 \\
& MAP   & \textbf{0.99} & \textbf{0.56} & \textbf{0.38} & \textbf{0.21} & \textbf{0.99} & \textbf{0.75} & \textbf{0.68} & \textbf{0.63} \\
& Marg. & \textbf{0.99} & \textbf{0.57} & \textbf{0.38} & \textbf{0.23} & \textbf{0.99} & \textbf{0.74} & \textbf{0.67} & \textbf{0.63} \\
\bottomrule
\end{tabular}
\end{table}

Under standard EI, MLE recovers essentially none of the
degree-one component on non-affine objectives, consistent
with excessive exploration of fresh directions.
MAP and marginalisation recover this component better,
but their limited off-span exploration restricts
optimisation performance.
A schedule based on the fitted variance ratio improves
degree-one recovery under MAP and marginalisation,
yet its attained values remain below those of the
VoI schedule on non-affine objectives.

Prescribing off-span exploration independently of
model fit substantially improves attained values under
MAP and marginalisation, without changing the posterior
or in-span EI criterion.
However, it cannot compensate for poor surrogate fitting:
under MLE, degree-one recovery remains negligible
on non-affine objectives.
These results highlight two requirements:
a useful surrogate fit and exploration control independent
of the fitted discrepancy scale.
Additional experiments with negligible degree-one signal
show that prescribed exploration remains beneficial
even without a recoverable affine component
(Appendix~\ref{app:zero-linear-signal}).
Appendix~\ref{sec:exp-hyperparameters} provides
diagnostics, fixed-$\lambda_{\mathrm{ref}}$ ablations,
and noisy-observation results;
Appendix~\ref{sec:local-resolvability} discusses
the choice of $\lambda_{\mathrm{ref}}$.

\begin{figure}[ht]
    \centering
    \includegraphics[width=\textwidth]
    {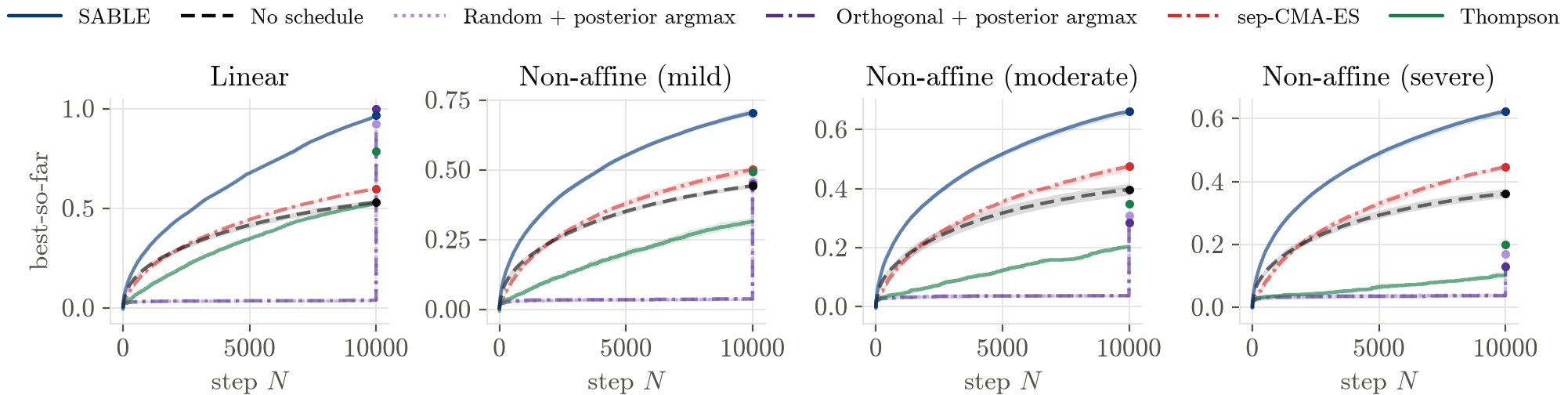}
    \vspace{-1.2em}
    \caption{
    \textbf{Controlled long-horizon evaluation, $D=N=10{,}000$.}
    Lines show mean best observed value over ten runs, with
    $5$th--$95$th percentile bands. Markers at $N=10{,}000$ show attained
    value, the better of the best observed point and the posterior-mean recommendation.
    }
    \label{fig:c4-rank-lambda}
\end{figure}

\paragraph{Controlled long-horizon evaluation.}

Figure~\ref{fig:c4-rank-lambda} evaluates SABLE at
$D=N=10{,}000$, to cover the regime $N \leq D$ up to its boundary $N=D$, using $M=300$,
marginalised variance scales, and the prescribed
schedule. We compare against the no-schedule
ablation, Thompson sampling using the same compressed posterior, random and
orthogonal design-then-fit baselines, and sep-CMA-ES
\citep{ros2008simple}. Curves report best observed value, while endpoint
markers report attained value, accounting for the posterior-mean
recommendation where available. SABLE attains the highest endpoint value on
all three non-affine objectives.
See Appendix~\ref{app:c4-endpoint} for fixed-$\lambda_{\mathrm{ref}}$ ablations.

\begin{figure}[ht]
    \centering
    \includegraphics[width=\textwidth]{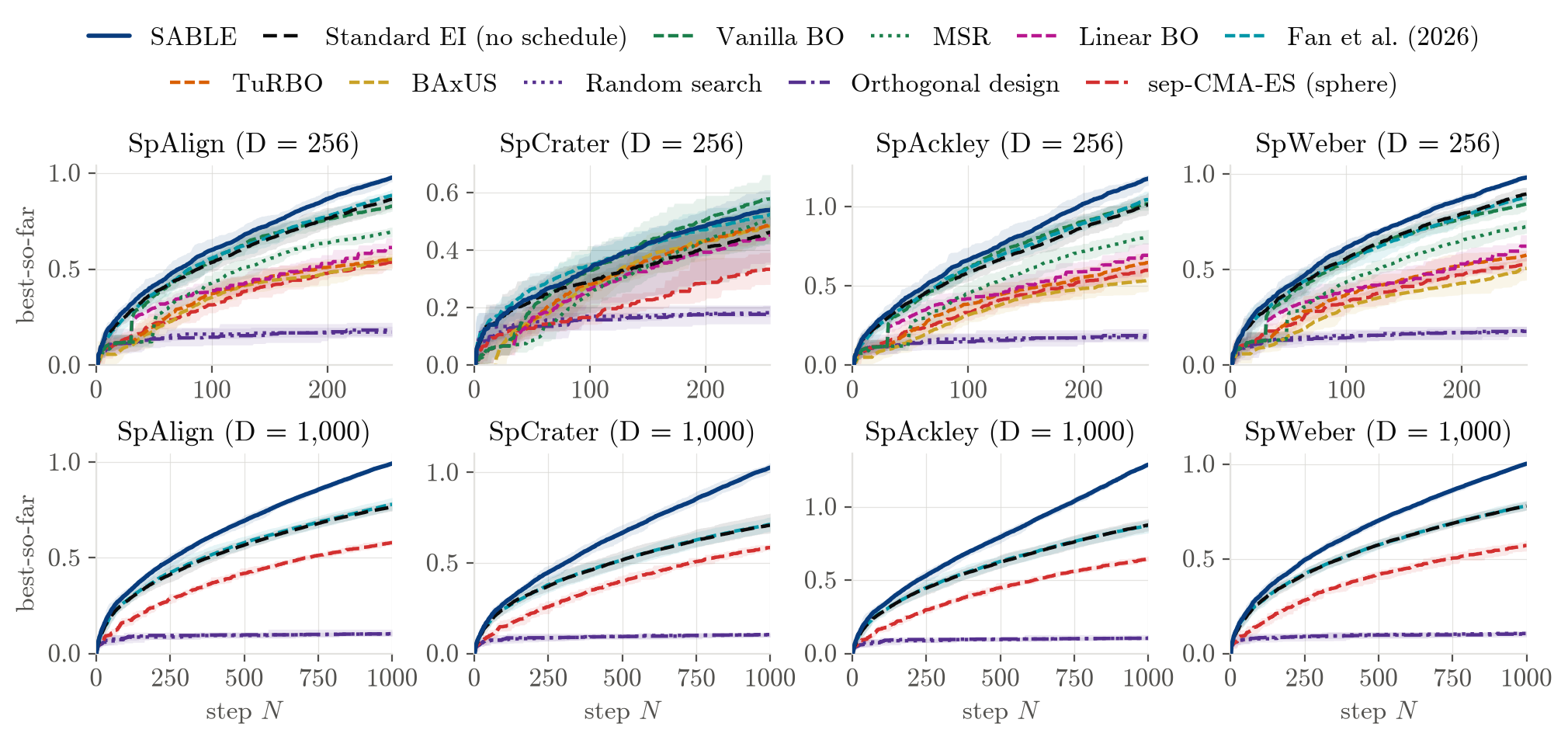}
    \vspace{-1.2em}
    \caption{
    \textbf{Spherical objectives of \citet{costofsampling}.}
    Best observed value against evaluation count at $D=N=256$ (\emph{top})
    and $D=N=1{,}000$ (\emph{bottom}). Lines show means and bands
    $5$th--$95$th percentiles over ten runs. The GP-based baselines are run
    at $D=256$ only. 
    }
    \label{fig:c12-spherical}
\end{figure}

\paragraph{Spherical objectives.}
We evaluate SABLE on the spherical objectives of
\citet{costofsampling}: SpAlign, SpCrater, SpAckley, and SpWeber,
at $D=N\in\{256,1{,}000\}$ over ten runs.
Objectives are negated and rescaled to have maxima of
order one (exactly $1$ for SpAlign).
We compare SABLE with the method of \citet{costofsampling},
sep-CMA-ES \citep{ros2008simple}, our no-schedule ablation,
and random and orthogonal search.
At $D=256$, we also include
Vanilla BO \citep{hvarfner2024vanilla},
Linear BO \citep{doumont2026},
TuRBO \citep{eriksson2019scalable},
MSR \citep{papenmeier2025understanding},
and BAxUS \citep{papenmeier2022increasing}.
Figure~\ref{fig:c12-spherical} shows that SABLE achieves the
highest mean best observed value in seven of eight settings,
leading the next-best method by $0.09$--$0.41$.
The exception is SpCrater at $D=256$, where Vanilla BO
attains $0.58$ versus SABLE's $0.54$.
The margin over \citet{costofsampling} increases from
$0.02$--$0.13$ at $D=256$ to $0.21$--$0.42$ at $D=1{,}000$.

\begin{figure}[t]
\centering
\scriptsize
\setlength{\tabcolsep}{4.0pt}%
\begin{tabular*}{\textwidth}{@{}c@{\hspace{2\tabcolsep}}l@{\extracolsep{\fill}\hspace{2\tabcolsep}}ccccc@{}}
\toprule
& & Turbulence & Calorimeter DS2 & Peptides & Calorimeter DS3 & decision time (s) \\
& & $D = 6{,}000$ & $D = 6{,}480$ & $D = 13{,}440$ & $D = 40{,}500$ & at $N = 3{,}000$,  \\
& evaluation time (s) & 4.6 & 1.1 & 0.49 & 8.2 & $D = 6{,}480$ \\
\midrule
\multirow{8}{*}{\rotatebox[origin=c]{90}{\emph{Model-free}}}
& Random search
& 0.89 {\scriptsize [0.88, 0.90]}
& 0.43 {\scriptsize [0.40, 0.45]}
& 1.01 {\scriptsize [0.86, 1.11]}
& 0.21 {\scriptsize [0.19, 0.22]}
& $<$0.01 \\
& Orthogonal design
& 0.89 {\scriptsize [0.88, 0.91]}
& 0.42 {\scriptsize [0.39, 0.46]}
& 0.95 {\scriptsize [0.87, 1.03]}
& 0.22 {\scriptsize [0.21, 0.24]}
& $<$0.01 \\
& sep-CMA-ES (sphere)
& 0.92 {\scriptsize [0.91, 0.93]}
& 0.76 {\scriptsize [0.74, 0.78]}
& \textbf{1.25} {\scriptsize [1.20, 1.27]}
& 0.55 {\scriptsize [0.53, 0.57]}
& $<$0.01 \\
& sep-CMA-ES ($\R^D$)
& 0.84 {\scriptsize [0.82, 0.86]}
& 0.55 {\scriptsize [0.53, 0.58]}
& \textbf{1.25} {\scriptsize [1.22, 1.28]}
& 0.28 {\scriptsize [0.27, 0.31]}
& $<$0.01 \\
& CMA-ES (sphere)
& 0.92 {\scriptsize [0.91, 0.93]}
& 0.76 {\scriptsize [0.75, 0.77]}
& ---
& ---
& 3.0 \\
& CMA-ES ($\R^D$)
& 0.84 {\scriptsize [0.82, 0.86]}
& 0.55 {\scriptsize [0.53, 0.56]}
& ---
& ---
& 3.0 \\
& CMA-ES (O3, $K = 90$)
& 0.94 {\scriptsize [0.93, 0.95]}
& 0.51 {\scriptsize [0.46, 0.54]}
& 0.96 {\scriptsize [0.76, 1.18]}
& 0.34 {\scriptsize [0.32, 0.37]}
& $<$0.01 \\
& CMA-ES (O3, $K = 300$)
& 0.90 {\scriptsize [0.89, 0.91]}
& 0.49 {\scriptsize [0.47, 0.50]}
& 1.05 {\scriptsize [0.85, 1.22]}
& 0.34 {\scriptsize [0.32, 0.35]}
& 0.02 \\
\midrule
\multirow{8}{*}{\rotatebox[origin=c]{90}{\emph{GP}}}
& Vanilla BO
& 0.90$^{\dagger}$ {\scriptsize [0.89, 0.92]}
& 0.58$^{\dagger}$ {\scriptsize [0.54, 0.60]}
& 0.99$^{\dagger}$ {\scriptsize [0.76, 1.18]}
& ---
& 500$^{\dagger}$ at $N = 546$ \\
& MSR
& 0.89$^{\dagger}$ {\scriptsize [0.88, 0.91]}
& 0.60$^{\dagger}$ {\scriptsize [0.54, 0.67]}
& 0.95$^{\dagger}$ {\scriptsize [0.75, 1.12]}
& ---
& 900$^{\dagger}$ at $N = 303$ \\
& Linear BO
& 0.89$^{\dagger}$ {\scriptsize [0.88, 0.90]}
& 0.52$^{\dagger}$ {\scriptsize [0.48, 0.58]}
& 1.13$^{\dagger}$ {\scriptsize [0.95, 1.23]}
& ---
& 730$^{\dagger}$ at $N = 395$ \\
& TuRBO
& 0.93 {\scriptsize [0.91, 0.94]}
& 0.83 {\scriptsize [0.82, 0.84]}
& 0.68 {\scriptsize [0.57, 0.80]}
& ---
& 7.8 \\
& TuRBO (O3, $K = 90$)
& 0.94 {\scriptsize [0.93, 0.95]}
& 0.54 {\scriptsize [0.52, 0.55]}
& 0.97 {\scriptsize [0.81, 1.21]}
& 0.36 {\scriptsize [0.35, 0.37]}
& 3.7 \\
& TuRBO (O3, $K = 300$)
& 0.94 {\scriptsize [0.92, 0.95]}
& 0.57 {\scriptsize [0.55, 0.58]}
& 0.95 {\scriptsize [0.73, 1.14]}
& 0.40 {\scriptsize [0.39, 0.41]}
& 5.7 \\
& BAxUS
& 0.92$^{\dagger}$ {\scriptsize [0.90, 0.94]}
& 0.70$^{\dagger}$ {\scriptsize [0.65, 0.74]}
& 0.75$^{\dagger}$ {\scriptsize [0.51, 1.08]}
& 0.45$^{\dagger}$ {\scriptsize [0.43, 0.48]}
& 430$^{\dagger}$ at $N = 1{,}132$ \\
& \cite{costofsampling}
& 0.90 {\scriptsize [0.85, 0.94]}
& 0.83 {\scriptsize [0.80, 0.85]}
& 0.77 {\scriptsize [0.60, 0.98]}
& 0.39 {\scriptsize [0.38, 0.40]}
& 2.5 \\
\midrule
\multirow{2}{*}{\rotatebox[origin=c]{90}{\emph{Ours}}}
& Standard EI (no schedule)
& 0.92 {\scriptsize [0.85, 0.94]}
& 0.83 {\scriptsize [0.81, 0.85]}
& 0.81 {\scriptsize [0.64, 1.12]}
& 0.39 {\scriptsize [0.39, 0.40]}
& 0.027 \\
& SABLE
& \textbf{0.95} {\scriptsize [0.94, 0.96]}
& \textbf{0.86} {\scriptsize [0.84, 0.87]}
& 1.20 {\scriptsize [1.15, 1.23]}
& \textbf{0.66} {\scriptsize [0.64, 0.68]}
& 0.027 \\
\bottomrule
\end{tabular*}
\vspace{1.0em}\\
\includegraphics[width=\textwidth]{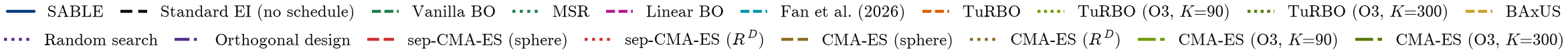}
\vspace{-1.0em}\\
\includegraphics[width=\textwidth]{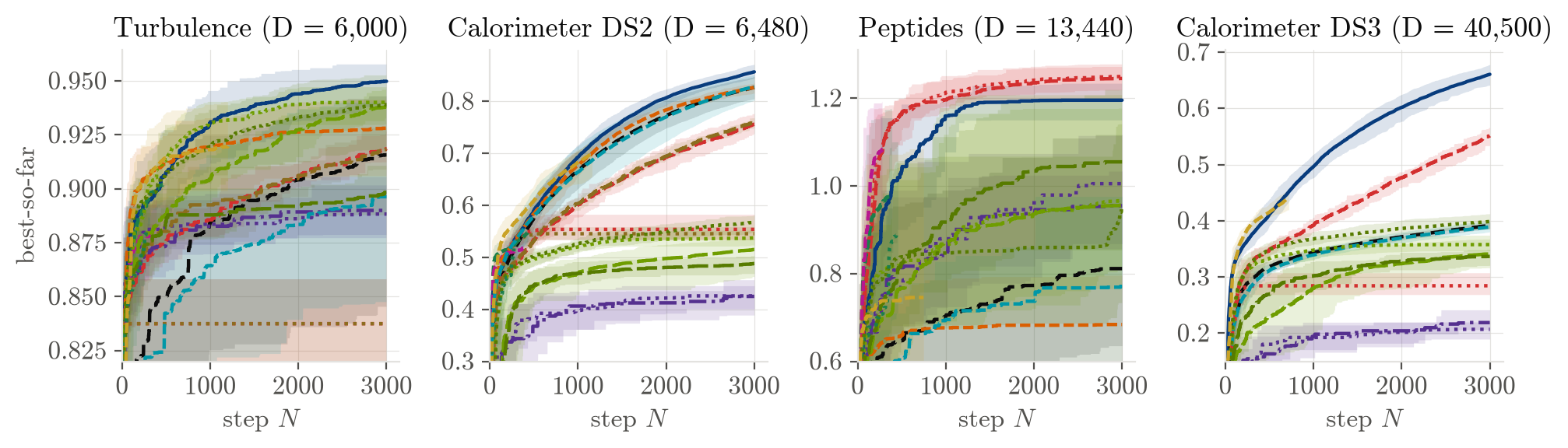}
\vspace{-1.2em}\\
\includegraphics[width=\textwidth]{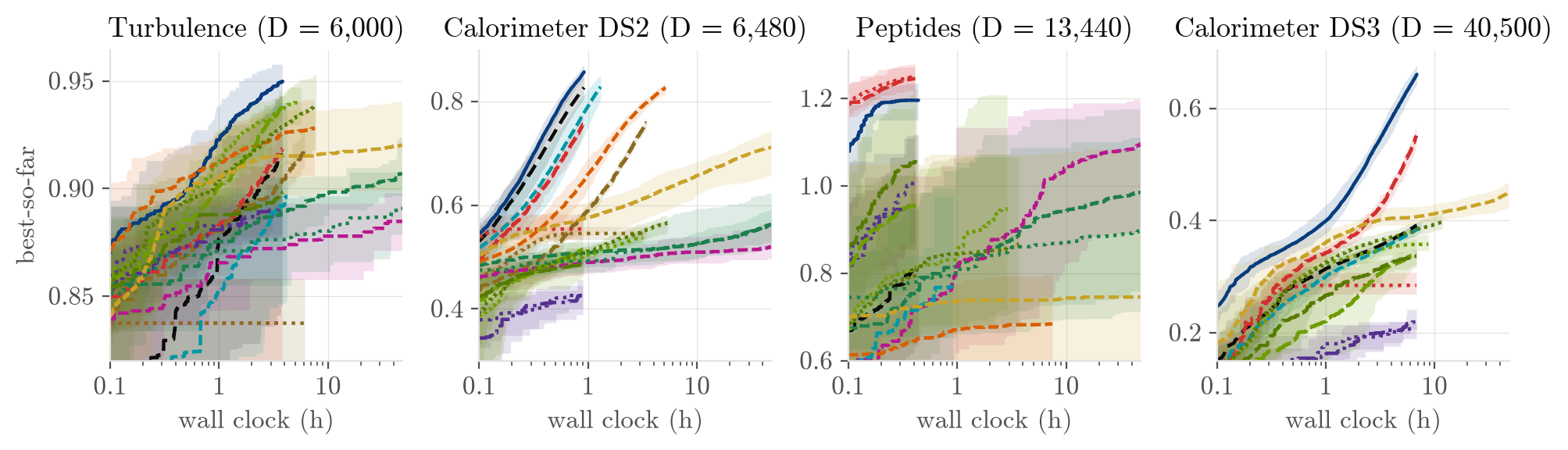}
\vspace{-1.2em}\\
\caption{
\textbf{Scientific latent-space optimisation.}
\emph{Table:} best observed value, reported as mean and
$5$th--$95$th percentile range over ten seeds; bold: within two standard
errors of the best. $^{\dagger}$: truncated at 48\,h, reaching
$N=303$--$1{,}132$.
\emph{Plots:} best value against evaluation count and wall-clock time.
}
\label{fig:c5-domains}
\end{figure}

\paragraph{Scientific latent-space optimisation.}

We evaluate SABLE on four latent-space optimisation problems from pretrained
scientific generative models, with $D$ from $6{,}000$ to $40{,}500$ (Appendix~\ref{app:benchmark-definitions}): Lagrangian turbulence, antimicrobial peptides,
and CaloDiffusion calorimeter showers at two detector resolutions.
SABLE uses $M=300$, marginalised variance scales,
and the prescribed exploration schedule. 
For the non-smooth peptide objective, whose variation
resembles noise, we use $\lambda_{\mathrm{ref}}=10$
(Appendix~\ref{app:defaults}).
Each method gets $3{,}000$ evaluations per run over ten runs.
Figure~\ref{fig:c5-domains} compares SABLE with model-free search,
CMA-ES variants, GP-based BO methods, and no-schedule ablations against evaluation count and wall-time.
In both views, SABLE attains the highest mean best
observed value on every problem except Peptides, where sep-CMA-ES leads. On Calorimeter
DS3 ($D=40{,}500$), SABLE reaches $0.66$ compared with $0.55$ for sep-CMA-ES.
As $D$ grows 6.25-fold, the schedule's gain over
no scheduling rises from $0.03$ on DS2 to $0.27$ on DS3.

To ensure a fair comparison, we give baselines access to the same typical-radius spherical domain as SABLE. This constraint can benefit the baselines by preventing unconstrained optimisers from proposing atypical latent norms that can severely degrade generation quality (see Section~\ref{sec:model}). The ambient-space CMA-ES ablation quantifies this effect, showing that removing the spherical constraint substantially worsens optimisation performance.

The comparison includes random and orthogonal search, separable and full-covariance CMA-ES,
Vanilla BO, Linear BO, TuRBO, MSR, BAxUS,
O3 with CMA-ES and TuRBO optimisers \citep{willis2025sample},
the method of \citet{costofsampling} without our posterior
compression or schedule, and our compressed-posterior,
no-schedule ablation.
Vanilla BO, Linear BO, MSR, and BAxUS do not reach the budget within 48 hours and are reported at the evaluations completed by then. GP baselines other than BAxUS are omitted on Calorimeter DS3, and full-covariance 
CMA-ES on Peptides and DS3, where costs are prohibitive. Full configurations are given in Appendix~\ref{app:baselines}.

\paragraph{Scalability.}
Computational differences become substantial at large budgets.
Figure~\ref{fig:c3-rank} shows that at $N=D=10{,}000$,
the $M=300$ variant requires just $0.0055$\,s
per decision, versus $79.5$\,s for the
uncapped representation underlying \citet{costofsampling},
a roughly 14,000-fold speedup in decision time.
On scientific tasks, Figure~\ref{fig:c5-domains} reports
$0.027$\,s per decision for SABLE on the $D=6{,}480$
calorimeter problem, compared with $1.1$\,s per objective
evaluation. By contrast, Vanilla BO, Linear BO, MSR,
and BAxUS require $430$--$900$\,s per decision at the
largest $N$ reached within $48$\,h ($N=303$--$1{,}132$).
Our compressed no-schedule ablation matches or exceeds
\citet{costofsampling} on every task while requiring 
approximately 90 times
less decision time. 
SABLE's wall-clock trajectories closely mirror its
evaluation-count performance.

\begin{wrapfigure}{r}{0.5\textwidth}
    \vspace{-\intextsep}
    \centering
    \includegraphics[width=\linewidth]{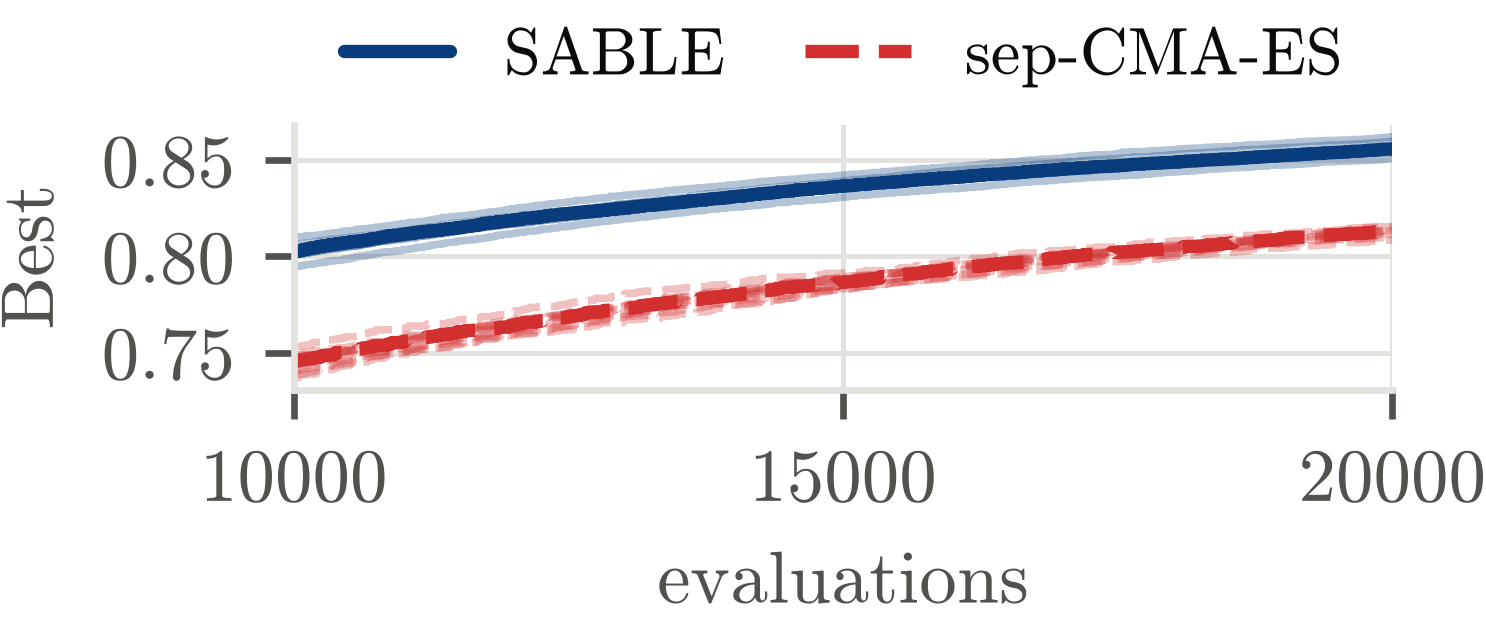}
    \vspace{-0.85em}
    \caption{
    \textbf{High-budget DS3 optimisation}. Best observed
    value against evaluations; thick lines are medians and thin lines the individual runs.
    }
    \label{fig:c9-ds3}
\end{wrapfigure}
\paragraph{Extending the budget at bounded cost.}
Because SABLE's per-round cost is bounded independently of $N$
for $N\leq D$ once the compression budget $M$ is reached,
it can run at budgets well beyond those of the GP baselines.
On Calorimeter DS3 ($D=40{,}500$) we extend the budget
to $N=20{,}000$ and compare against sep-CMA-ES, the only competitive baseline at this scale.
This exceeds the scientific-task budget more
than sixfold. SABLE leads on all ten runs at every budget
(Figure~\ref{fig:c9-ds3}), and it
matches sep-CMA-ES's final best observed value at approximately half the budget.

\section{Conclusion}
SABLE addresses high-dimensional latent-space optimisation
under limited evaluation budgets by combining data-sized
acquisition optimisation, controlled off-span exploration,
and bounded-cost model refitting through pseudo-observation
recompression, without restricting queries to a learned
subspace. Across synthetic and scientific benchmarks,
SABLE achieves strong sample efficiency up to $D=40{,}500$,
with millisecond-scale decisions in our timing experiments.

The method assumes isotropic latent geometry and a useful
affine component in the objective. Its benefits are more
limited when higher-order or quantised structure dominates,
as illustrated by the Peptides benchmark.
Future work includes alternative exploration schedules,
batch acquisition, and higher-order working models.

\bibliography{references_affine}

@article{li2024synthetic,
  title     = {Synthetic Lagrangian turbulence by generative diffusion models},
  author    = {Li, Tianyi and Biferale, Luca and Bonaccorso, Fabio and Scarpolini, Martino Andrea and Buzzicotti, Michele},
  journal   = {Nature Machine Intelligence},
  volume    = {6},
  number    = {4},
  pages     = {393--403},
  year      = {2024},
  publisher = {Nature Publishing Group},
}

@inproceedings{
chen2024ampdiffusion,
title={{AMP}-Diffusion: Integrating Latent Diffusion with Protein Language Models for Antimicrobial Peptide Generation},
author={Tianlai Chen and Pranay Vure and Rishab Pulugurta and Pranam Chatterjee},
booktitle={NeurIPS 2023 Generative AI and Biology (GenBio) Workshop},
year={2023},
}

@article{torres2025generative,
  title     = {Generative latent diffusion language modeling yields anti-infective synthetic peptides},
  author    = {Torres, Marcelo D. T. and Chen, Leo Tianlai and Wan, Fangping and Chatterjee, Pranam and de la Fuente-Nunez, Cesar},
  journal   = {Cell Biomaterials},
  volume    = {1},
  number    = {9},
  pages     = {100183},
  year      = {2025},
}

@article{lin2023esm2,
  title     = {Evolutionary-scale prediction of atomic-level protein structure with a language model},
  author    = {Lin, Zeming and Akin, Halil and Rao, Roshan and Hie, Brian and Zhu, Zhongkai and Lu, Wenting and Smetanin, Nikita and Verkuil, Robert and Kabeli, Ori and Shmueli, Yaniv and dos Santos Costa, Allan and Fazel-Zarandi, Maryam and Sercu, Tom and Candido, Salvatore and Rives, Alexander},
  journal   = {Science},
  volume    = {379},
  number    = {6637},
  pages     = {1123--1130},
  year      = {2023},
}

@article{eisenberg1984moment,
  title     = {Analysis of membrane and surface protein sequences with the hydrophobic moment plot},
  author    = {Eisenberg, David and Schwarz, Emil and Komaromy, Matthew and Wall, Roger},
  journal   = {Journal of Molecular Biology},
  volume    = {179},
  number    = {1},
  pages     = {125--142},
  year      = {1984},
}

@article{yeaman2003mechanisms,
  title     = {Mechanisms of antimicrobial peptide action and resistance},
  author    = {Yeaman, Michael R. and Yount, Nannette Y.},
  journal   = {Pharmacological Reviews},
  volume    = {55},
  number    = {1},
  pages     = {27--55},
  year      = {2003},
}

@article{amram2023calodiffusion,
  title     = {Denoising diffusion models with geometry adaptation for high fidelity calorimeter simulation},
  author    = {Amram, Oz and Pedro, Kevin},
  journal   = {Physical Review D},
  volume    = {108},
  number    = {7},
  pages     = {072014},
  year      = {2023},
  publisher = {American Physical Society},
}

@dataset{calochallenge_dataset2,
  title     = {Fast Calorimeter Simulation Challenge 2022 --- Dataset 2},
  author    = {Giannelli, Michele Faucci and Kasieczka, Gregor and Krause, Claudius and Nachman, Benjamin and Salamani, Dalila and Shih, David and Zaborowska, Anna},
  year      = {2022},
  publisher = {Zenodo},
  note      = {https://calochallenge.github.io/homepage/},
}

@dataset{calochallenge_dataset3,
  title     = {Fast Calorimeter Simulation Challenge 2022 --- Dataset 3},
  author    = {Giannelli, Michele Faucci and Kasieczka, Gregor and Krause, Claudius and Nachman, Benjamin and Salamani, Dalila and Shih, David and Zaborowska, Anna},
  year      = {2022},
  publisher = {Zenodo},
  note      = {https://calochallenge.github.io/homepage/},
}

@inproceedings{song2021ddim,
  title     = {Denoising Diffusion Implicit Models},
  author    = {Song, Jiaming and Meng, Chenlin and Ermon, Stefano},
  booktitle = {International Conference on Learning Representations (ICLR)},
  year      = {2021},
  url       = {https://arxiv.org/abs/2010.02502},
}

@article{lu2022dpmsolverpp,
  title   = {{DPM-Solver++}: Fast Solver for Guided Sampling of Diffusion Probabilistic Models},
  author  = {Lu, Cheng and Zhou, Yuhao and Bao, Fan and Chen, Jianfei and Li, Chongxuan and Zhu, Jun},
  journal = {Machine Intelligence Research},
  year    = {2025},
  issue={22},
  pages={730-751},
}

@article{watson2023rfdiffusion,
  author  = {Watson, Joseph L. and Juergens, David and Bennett, Nathaniel R. and Trippe, Brian L. and Yim, Jason and Eisenach, Helen E. and Ahern, Woody and Borst, Andrew J. and Ragotte, Robert J. and Milles, Lukas F. and Wicky, Basile I. M. and Hanikel, Nikita and Pellock, Samuel J. and Courbet, Alexis and Sheffler, William and Wang, Jue and Venkatesh, Preetham and Sappington, Isaac and V{\'a}zquez Torres, Susana and Lauko, Anna and De Bortoli, Valentin and Mathieu, Emile and Ovchinnikov, Sergey and Barzilay, Regina and Jaakkola, Tommi S. and DiMaio, Frank and Baek, Minkyung and Baker, David},
  title   = {De novo design of protein structure and function with {RFdiffusion}},
  journal = {Nature},
  volume  = {620},
  number  = {7976},
  pages   = {1089--1100},
  month   = aug,
  year    = {2023},
  issn    = {1476-4687},
  doi     = {10.1038/s41586-023-06415-8}
}

@article{das2021accelerated,
  author  = {Das, Payel and Sercu, Tom and Wadhawan, Kahini and Padhi, Inkit and Gehrmann, Sebastian and Cipcigan, Flaviu and Chenthamarakshan, Vijil and Strobelt, Hendrik and Dos Santos, Cicero and Chen, Pin-Yu and others},
  title   = {Accelerated antimicrobial discovery via deep generative models and molecular dynamics simulations},
  journal = {Nature Biomedical Engineering},
  volume  = {5},
  number  = {6},
  pages   = {613--623},
  year    = {2021}
}

@article{chenthamarakshan2023accelerating,
  author  = {Chenthamarakshan, Vijil and Hoffman, Samuel C. and Owen, C. David and Lukacik, Petra and Strain-Damerell, Claire and Fearon, Daren and Malla, Tika R. and Tumber, Anthony and Schofield, Christopher J. and Duyvesteyn, Helen M. E. and others},
  title   = {Accelerating drug target inhibitor discovery with a deep generative foundation model},
  journal = {Science Advances},
  volume  = {9},
  number  = {25},
  year    = {2023}
}

@inproceedings{tripp2020sampleefficient,
  author    = {Tripp, Austin and Daxberger, Erik and Hern{\'a}ndez-Lobato, Jos{\'e} Miguel},
  title     = {Sample-efficient optimization in the latent space of deep generative models via weighted retraining},
  booktitle = {Advances in Neural Information Processing Systems},
  volume    = {33},
  pages     = {11259--11272},
  year      = {2020}
}

@inproceedings{maus2022local,
  author    = {Maus, Natalie and Jones, Haydn and Moore, Juston and Kusner, Matt J. and Bradshaw, John and Gardner, Jacob},
  title     = {Local latent space Bayesian optimization over structured inputs},
  booktitle = {Advances in Neural Information Processing Systems},
  volume    = {35},
  pages     = {34505--34518},
  publisher = {Curran Associates, Inc.},
  year      = {2022}
}

@inproceedings{moss2025return,
  author    = {Moss, Henry and Ober, Sebastian W. and Diethe, Tom},
  title     = {Return of the latent space {COWBOYS}: Re-thinking the use of {VAEs} for Bayesian optimisation of structured spaces},
  booktitle = {Proceedings of the 42nd International Conference on Machine Learning},
  series    = {Proceedings of Machine Learning Research},
  volume    = {267},
  pages     = {44956--44970},
  publisher = {PMLR},
  year      = {2025}
}

@book{garnett2023bayesian,
  author    = {Garnett, Roman},
  title     = {Bayesian Optimization},
  publisher = {Cambridge University Press},
  year      = {2023}
}

@inproceedings{eriksson2019scalable,
  author    = {Eriksson, David and Pearce, Michael and Gardner, Jacob and Turner, Ryan D. and Poloczek, Matthias},
  title     = {Scalable global optimization via local Bayesian optimization},
  booktitle = {Advances in Neural Information Processing Systems},
  volume    = {32},
  year      = {2019}
}

@inproceedings{hvarfner2024vanilla,
  author    = {Hvarfner, Carl and Hellsten, Erik Orm and Nardi, Luigi},
  title     = {Vanilla Bayesian optimization performs great in high dimensions},
  booktitle = {41st International Conference on Machine Learning},
  volume    = {235},
  pages     = {20793--20817},
  publisher = {PMLR},
  year      = {2024}
}

@book{cover2006elements,
  author    = {Cover, Thomas M. and Thomas, Joy A.},
  title     = {Elements of Information Theory},
  edition   = {2},
  publisher = {Wiley-Interscience},
  year      = {2006},
  isbn      = {9780471748816}
}

@book{vershynin2026hdp,
  author    = {Vershynin, Roman},
  title     = {High-Dimensional Probability: An Introduction with Applications in Data Science},
  series    = {Cambridge Series in Statistical and Probabilistic Mathematics},
  publisher = {Cambridge University Press},
  edition   = {2},
  year      = {2026}
}

@article{nalisnick2019typicality,
  author  = {Nalisnick, Eric and Matsukawa, Akihiro and Teh, Yee Whye and Lakshminarayanan, Balaji},
  title   = {Detecting out-of-distribution inputs to deep generative models using typicality},
  journal = {arXiv preprint arXiv:1906.02994},
  year    = {2019}
}

@inproceedings{wang2013bayesian,
author = {Wang, Ziyu and Zoghi, Masrour and Hutter, Frank and Matheson, David and De Freitas, Nando},
title = {Bayesian optimization in high dimensions via random embeddings},
year = {2013},
isbn = {9781577356332},
publisher = {AAAI Press},
booktitle = {Proceedings of the Twenty-Third International Joint Conference on Artificial Intelligence},
pages = {1778–1784},
numpages = {7},
location = {Beijing, China},
series = {IJCAI '13}
}

@InProceedings{kandasamy2015high,
  title = 	 {High Dimensional Bayesian Optimisation and Bandits via Additive Models},
  author = 	 {Kandasamy, Kirthevasan and Schneider, Jeff and Poczos, Barnabas},
  booktitle = 	 {32nd International Conference on Machine Learning},
  pages = 	 {295--304},
  year = 	 {2015},
  volume = 	 {37},
  publisher =    {PMLR},
}

@InProceedings{eriksson2021high,
  title = 	 {High-dimensional {Bayesian} optimization with sparse axis-aligned subspaces},
  author =       {Eriksson, David and Jankowiak, Martin},
  booktitle = 	 {Proceedings of the Thirty-Seventh Conference on Uncertainty in Artificial Intelligence},
  pages = 	 {493--503},
  year = 	 {2021},
  volume = 	 {161},
  series = 	 {Proceedings of Machine Learning Research},
  month = 	 {27--30 Jul},
  publisher =    {PMLR},
}

@inproceedings{
doumont2026,
title={We Still Don{\textquoteright}t Understand High-Dimensional Bayesian Optimization},
author={Colin Doumont and Donney Fan and Natalie Maus and Jacob R. Gardner and Henry Moss and Geoff Pleiss},
booktitle={The 29th International Conference on Artificial Intelligence and Statistics},
year={2026},
}

@inproceedings{papenmeier2025understanding,
  title = 	 {Understanding High-Dimensional {B}ayesian Optimization},
  author =       {Papenmeier, Leonard and Poloczek, Matthias and Nardi, Luigi},
  booktitle = 	 {42nd International Conference on Machine Learning},
  pages = 	 {47902--47923},
  year = 	 {2025},
  volume = 	 {267},
  series = 	 {Proceedings of Machine Learning Research},
  publisher =    {PMLR},
}

@inproceedings{
samuel2023normguided,
title={Norm-guided latent space exploration for text-to-image generation},
author={Dvir Samuel and Rami Ben-Ari and Nir Darshan and Haggai Maron and Gal Chechik},
booktitle={Thirty-seventh Conference on Neural Information Processing Systems},
year={2023},
}

@article{costofsampling,
title={Search at the Cost of Sampling: Nearly-Instant Latent Space Bayesian Optimization},
author={Donney Fan and Colin Doumont and Aleksandra Kalisz and Paul Duckworth and Jacob R. Gardner and Henry Moss and Geoff Pleiss},
journal = {arXiv preprint arXiv:2609.19476}, 
year=2026,
}

@InProceedings{krishnamoorthy2023,
  title = 	 {Diffusion Models for Black-Box Optimization},
  author =    {Krishnamoorthy, Siddarth and Mashkaria, Satvik Mehul and Grover, Aditya},
  booktitle = 	 {40th International Conference on Machine Learning},
  pages = 	 {17842--17857},
  year = 	 {2023},
  volume = 	 {202},
  publisher =    {PMLR},
}

@inproceedings{
fan2023,
title={Reinforcement Learning for Fine-tuning Text-to-Image Diffusion Models},
author={Ying Fan and Olivia Watkins and Yuqing Du and Hao Liu and Moonkyung Ryu and Craig Boutilier and Pieter Abbeel and Mohammad Ghavamzadeh and Kangwook Lee and Kimin Lee},
booktitle={Thirty-seventh Conference on Neural Information Processing Systems},
year={2023},
url={https://openreview.net/forum?id=8OTPepXzeh}
}

@inproceedings{
denker2025,
title={Iterative Importance Fine-tuning of Diffusion Models},
author={Alexander Denker and Shreyas Padhy and Francisco Vargas and Johannes Hertrich},
booktitle={Frontiers in Probabilistic Inference: Learning meets Sampling},
year={2025},
url={https://openreview.net/forum?id=HLaFozI6It}
}

@inproceedings{
yun2025,
title={Posterior Inference with Diffusion Models for High-dimensional Black-box Optimization},
author={Taeyoung Yun and Kiyoung Om and Jaewoo Lee and Sujin Yun and Jinkyoo Park},
booktitle={Forty-second International Conference on Machine Learning},
year={2025},
}

@inproceedings{Gruver2023,
 author = {Gruver, Nate and Stanton, Samuel and Frey, Nathan and Rudner, Tim G. J. and Hotzel, Isidro and Lafrance-Vanasse, Julien and Rajpal, Arvind and Cho, Kyunghyun and Wilson, Andrew G},
 booktitle = {Advances in Neural Information Processing Systems},
 doi = {10.52202/075280-0547},
 pages = {12489--12517},
 publisher = {Curran Associates, Inc.},
 title = {Protein Design with Guided Discrete Diffusion},
volume = {36},
 year = {2023}
}

@inproceedings{
yuan2025paretoflow,
title={ParetoFlow: Guided Flows in Multi-Objective Optimization},
author={Ye Yuan and Can Chen and Christopher Pal and Xue Liu},
booktitle={The Thirteenth International Conference on Learning Representations},
year={2025},
}

@inproceedings{ros2008simple,
  title={A simple modification in CMA-ES achieving linear time and space complexity},
  author={Ros, Raymond and Hansen, Nikolaus},
  booktitle={International conference on parallel problem solving from nature},
  pages={296--305},
  year={2008},
  organization={Springer}
}

@article{gomez2018automatic,
  title={Automatic chemical design using a data-driven continuous representation of molecules},
  author={G{\'o}mez-Bombarelli, Rafael and Wei, Jennifer N and Duvenaud, David and Hern{\'a}ndez-Lobato, Jos{\'e} Miguel and S{\'a}nchez-Lengeling, Benjam{\'\i}n and Sheberla, Dennis and Aguilera-Iparraguirre, Jorge and Hirzel, Timothy D and Adams, Ryan P and Aspuru-Guzik, Al{\'a}n},
  journal={ACS central science},
  volume={4},
  number={2},
  pages={268},
  year={2018}
}

@inproceedings{kusner2017grammar,
  title={Grammar variational autoencoder},
  author={Kusner, Matt J and Paige, Brooks and Hern{\'a}ndez-Lobato, Jos{\'e} Miguel},
  booktitle={International conference on machine learning},
  pages={1945--1954},
  year={2017},
  organization={PMLR}
}

@article{binois2022survey,
  title={A survey on high-dimensional Gaussian process modeling with application to Bayesian optimization},
  author={Binois, Mickael and Wycoff, Nathan},
  journal={ACM Transactions on Evolutionary Learning and Optimization},
  volume={2},
  number={2},
  pages={1--26},
  year={2022},
  publisher={ACM New York, NY}
}

@article{griffiths2020constrained,
  title={Constrained Bayesian optimization for automatic chemical design using variational autoencoders},
  author={Griffiths, Ryan-Rhys and Hern{\'a}ndez-Lobato, Jos{\'e} Miguel},
  journal={Chemical science},
  volume={11},
  number={2},
  pages={577--586},
  year={2020},
  publisher={The Royal Society of Chemistry}
}

@inproceedings{hansen1996adapting,
  title={Adapting arbitrary normal mutation distributions in evolution strategies: The covariance matrix adaptation},
  author={Hansen, Nikolaus and Ostermeier, Andreas},
  booktitle={Proceedings of IEEE international conference on evolutionary computation},
  pages={312--317},
  year={1996},
  organization={IEEE}
}

@article{willis2025sample,
  title   = {Sample-Efficient Optimisation over the Outputs of Generative Models},
  author  = {Willis, Samuel
             and Duckworth, Paul
             and Simons, Jack
             and Kalisz, Aleksandra
             and Sinkovics, Krisztina
             and Ghenassia, Noam
             and Surana, Shikha
             and Oldroyd, Henry T.
             and Stere, Alexandru I.
             and Margineantu, Dragos D.
             and Ek, Carl Henrik
             and Moss, Henry
             and Bodin, Erik},
  journal = {arXiv preprint arXiv:2509.23800},
  year    = {2025}
}

@inproceedings{bodin2025linear,
  title={Linear combinations of latents in generative models: subspaces and beyond},
  author={Bodin, Erik and Stere, Alexandru and Margineantu, Dragos and Ek, Carl and Moss, Henry},
  booktitle={International Conference on Learning Representations},
  volume={2025},
  pages={8014--8043},
  year={2025}
}

@article{balandat2020botorch,
  title={BoTorch: A framework for efficient Monte-Carlo Bayesian optimization},
  author={Balandat, Maximilian and Karrer, Brian and Jiang, Daniel and Daulton, Samuel and Letham, Ben and Wilson, Andrew G and Bakshy, Eytan},
  journal={Advances in neural information processing systems},
  volume={33},
  pages={21524--21538},
  year={2020}
}

@article{papenmeier2022increasing,
  title={Increasing the scope as you learn: Adaptive bayesian optimization in nested subspaces},
  author={Papenmeier, Leonard and Nardi, Luigi and Poloczek, Matthias},
  journal={Advances in Neural Information Processing Systems},
  volume={35},
  pages={11586--11601},
  year={2022}
}
\bibliographystyle{iclr2027_conference}

\clearpage
\appendix

\section{Affine working model and posterior construction}\label{app:affine-foundations}

\subsection{Spherical affine decomposition}\label{app:objective-projection}
Let \(\nu_D\) denote the uniform probability measure on
\(\mathbb S^{D-1}\) and 
let $\langle u,v\rangle_{L^2(\nu_D)}:=\int_{\Sph}u(\bm s)v(\bm s)\,d\nu_D(\bm s)$. Rotational symmetry of $\nu_D$ gives
\[
\int_{\Sph}\bm s\,d\nu_D(\bm s)=\bm0,
\qquad
\int_{\Sph}\bm s\bm s^\top\,d\nu_D(\bm s)=\frac{1}{D}I_D.
\]
Consequently, the $L^2(\nu_D)$ projection of $F$ onto the constant functions is
\[
\beta_0=\int_{\Sph}F(\bm s)\,d\nu_D(\bm s),
\]
and its projection onto the degree-one space $\{\bm a^\top\bm s:\bm a\in\R^D\}$ has coefficient
\[
\bm v_\star
=
D\int_{\Sph}F(\bm s)\bm s\,d\nu_D(\bm s).
\]
Define $\Delta(\bm s):=F(\bm s)-\beta_0-\bm v_\star^\top\bm s$. Then
\[
\int_{\Sph}\Delta(\bm s)\,d\nu_D(\bm s)=0,
\qquad
\int_{\Sph}\Delta(\bm s)\bm s\,d\nu_D(\bm s)=\bm0.
\]
Hence $\Delta$ is orthogonal in $L^2(\nu_D)$ to every constant function and every degree-one function $\bm a^\top\bm s$. Therefore
\[
F(\bm s)=\beta_0+\bm v_\star^\top\bm s+\Delta(\bm s)
\]
is precisely the orthogonal decomposition of $F$ into its degree-zero, degree-one, and degree-at-least-two spherical-harmonic components. The decomposition is unique in $L^2(\nu_D)$ and does not require the higher-order residual $\Delta$ to be small.

\subsection{Prior rotational symmetry}\label{app:prior-rotation}
For every orthogonal matrix $R\in O(D)$, the isotropic prior satisfies $R\bm v\overset{d}{=}\bm v$. The affine likelihood is equivariant under the simultaneous transformation $\bm s_i\mapsto R\bm s_i$ and $\bm v\mapsto R\bm v$. Thus, before observing data, the working model has no preferred coefficient orientation; conditioning can introduce directional preference only through the geometry of the evaluated design.

\subsection{Analytical treatment of the intercept}\label{app:intercept-marginalisation}
Let $\bm\theta=(\sigma_1^2,\sigma_\varepsilon^2)$. Integrating the Gaussian slope coefficient $\bm v$ gives
\[
\bm y_N\mid\beta,\bm\theta
\sim
\mathcal N\!\left(\beta\one_N,C_N(\bm\theta)\right),
\qquad
C_N(\bm\theta):=\sigma_1^2X_NX_N^\top+\sigma_\varepsilon^2I_N.
\]
This covariance enters the marginal likelihood for the variance parameters and intercept. For fixed $\bm\theta$, define $q_N(\bm\theta)=\one_N^\top C_N(\bm\theta)^{-1}\one_N$ and
\[
\widehat\beta(\bm\theta)
=
\frac{\one_N^\top C_N(\bm\theta)^{-1}\bm y_N}
{\one_N^\top C_N(\bm\theta)^{-1}\one_N}.
\]
Under the flat prior, $\beta\mid\bm\theta,\bm y_N\sim\mathcal N(\widehat\beta(\bm\theta),q_N(\bm\theta)^{-1})$. Integrating $\beta$ gives
\[
\begin{split}
p(\bm y_N\mid\bm\theta)
&\propto |C_N(\bm\theta)|^{-1/2}q_N(\bm\theta)^{-1/2}\\
&\quad\times\exp\!\left\{-\frac12
(\bm y_N-\widehat\beta\one_N)^\top
C_N(\bm\theta)^{-1}
(\bm y_N-\widehat\beta\one_N)\right\},
\end{split}
\]
where the dependence of $\widehat\beta$ on $\bm\theta$ is suppressed inside the exponent. Hence $p(\bm\theta\mid\bm y_N)\propto p(\bm y_N\mid\bm\theta)p(\sigma_1^2)p(\sigma_\varepsilon^2)$. The factor $q_N(\bm\theta)^{-1/2}$ distinguishes analytical marginalisation from profiling out the intercept.

\subsection{Hyperposterior discretisation}\label{app:atoms}

We place log-normal priors on the signal and discrepancy scales,
\[
\log\sigma_1
\sim
\mathcal N(\log\bar\sigma_1,h_1^2),
\qquad
\log\sigma_\varepsilon
\sim
\mathcal N(\log\bar\sigma_\varepsilon,h_\varepsilon^2).
\]
Here $\bar\sigma_1,\bar\sigma_\varepsilon$ are prior medians and
$h_1,h_\varepsilon$ are log-scale widths.

To keep the prior scale of the global affine variation comparable across
ambient dimensions, we parameterise the signal prior through
\(
\hat\sigma_1:=\sigma_1\sqrt D .
\)
Indeed, under the isotropic slope prior,
$\|\bm v\|$ is of order $\sigma_1\sqrt D$, while
$\bm v^\top\bm s$ has standard deviation $\sigma_1$ for any fixed
$\bm s\in\Sph$. We therefore specify a dimension-independent prior median
$\hat{\bar\sigma}_1$ and set
\(
\bar\sigma_1=\tfrac{\hat{\bar\sigma}_1}{\sqrt D}.
\)

We approximate the hyperposterior on a fixed
$n_g\times n_g$ grid in
$(\log\sigma_1,\log\sigma_\varepsilon)$.
Each axis spans $\eta$ prior standard deviations around its prior median;
the adopted implementation extends the upper end of the discrepancy axis
slightly further, as specified in Appendix~\ref{app:defaults}.
For each grid point, the marginal likelihood is evaluated using
Appendix~\ref{app:intercept-marginalisation}.

These evaluations can be shared efficiently across the grid.
Let
\(
X_NX_N^\top
=
U\operatorname{diag}(\ell_1,\ldots,\ell_N)U^\top,
\)
and define
$\tilde{\bm o}:=U^\top\one_N$ and
$\tilde{\bm y}:=U^\top\bm y_N$.
Then the quadratic forms
$\one_N^\top C_N^{-1}\one_N$,
$\one_N^\top C_N^{-1}\bm y_N$,
$\bm y_N^\top C_N^{-1}\bm y_N$,
and $\log|C_N|$
are each obtained from $\mathcal O(N)$ sums over the eigenvalues
$\ell_j$. For the fixed grid used here, these likelihood evaluations consist of
scalar eigenbasis calculations once the Gram spectrum is available.
In the streaming implementation the spectrum is maintained incrementally;
the resulting per-round refitting complexity is given in
Section~\ref{sec:complexity}.

Multiplying the grid likelihoods by the prior densities and normalising gives
a discrete approximation to $p(\bm\theta\mid\mathcal D_N)$.
For acquisition, we compress this distribution to $K$ equally weighted
hyperparameter atoms using deterministic stratified resampling.
If $W_j$ denotes the cumulative posterior mass in a fixed ordering of the
grid cells, atom $k$ is assigned to the first cell satisfying
\(
W_j\geq (k-\tfrac12)/K.
\)
Each atom $\bm\theta_k$ retains the corresponding conditional intercept
posterior
\[
\beta\mid\bm\theta_k,\mathcal D_N
\sim
\mathcal N\!\left(
\widehat\beta(\bm\theta_k),
q_N(\bm\theta_k)^{-1}
\right).
\]
The default prior and discretisation parameters are collected in
Appendix~\ref{app:defaults}. The same construction can be used with other
smooth prior densities on the log-scales.

\subsection{Recommendation readout}\label{app:readout}

The query selected for the next evaluation need not coincide with the current
estimate of the maximiser. We therefore distinguish the next query
$\bm s_{N+1}$ from a recommendation $\hat{\bm s}$ returned when a point
estimate of the optimum is required.

We form the recommendation from the hyperparameter-marginalised posterior
mean of the affine slope,
\[
\overline{\bm m}_N
=
\frac1K\sum_{k=1}^K
\bm m_N(\bm\theta_k),
\qquad
\hat{\bm s}
=
\frac{\overline{\bm m}_N}
{\|\overline{\bm m}_N\|},
\]
where $\bm m_N(\bm\theta_k)$ is the conditional posterior slope mean from
Appendix~\ref{app:data-informed-geometry}.
Thus $\hat{\bm s}$ maximises the marginal posterior mean of the affine
working surrogate over the sphere.

If $\overline{\bm m}_N=\bm0$, as can occur before the data identify any
slope direction, the posterior-mean recommendation is undefined and we
return the best observed query instead.

\section{Posterior geometry under finite data}
\label{app:posterior-geometry}

\subsection{Data-informed coefficient subspace}
\label{app:data-informed-geometry}

\begin{proof}[Proof of Proposition~\ref{prop:explored-span}]
Let
\[
M_N
:=
I_N-\frac{1}{N}\one_N\one_N^\top,
\qquad
\bm\mu_N
:=
\frac{1}{N}\sum_{i=1}^N\bm s_i .
\]
The rows of $M_NX_N$ are the centred evaluated directions
$(\bm s_i-\bm\mu_N)^\top$. Hence
\[
A_N
:=
X_N^\top M_NX_N
=
\sum_{i=1}^N
(\bm s_i-\bm\mu_N)(\bm s_i-\bm\mu_N)^\top,
\]
and
\[
\mathcal S_N
:=
\operatorname{range}(A_N)
=
\operatorname{span}
\{\bm s_i-\bm\mu_N:i=1,\ldots,N\}.
\]
Because the centred directions sum to zero,
\[
r_N:=\dim(\mathcal S_N)=\rank(A_N)\leq N-1.
\]

Fix
$\bm\theta=(\sigma_1^2,\sigma_\varepsilon^2)$.
Under the affine working model,
\[
\bm y_N\mid\beta,\bm v,\bm\theta
\sim
\mathcal N
\left(
\beta\one_N+X_N\bm v,\,
\sigma_\varepsilon^2I_N
\right).
\]
Integrating the flat intercept gives, up to a factor independent of
$\bm v$,
\[
p(\bm y_N\mid\bm v,\bm\theta)
\propto
\exp\!\left\{
-\frac{1}{2\sigma_\varepsilon^2}
(\bm y_N-X_N\bm v)^\top
M_N
(\bm y_N-X_N\bm v)
\right\}.
\]
Combining this likelihood with
$\bm v\mid\sigma_1^2\sim
\mathcal N(\bm0,\sigma_1^2I_D)$
therefore gives
\[
\bm v\mid\mathcal D_N,\bm\theta
\sim
\mathcal N
\bigl(
\bm m_N(\bm\theta),
\Sigma_N(\bm\theta)
\bigr),
\]
where
\[
\Sigma_N(\bm\theta)
=
\left(
\sigma_1^{-2}I_D
+
\sigma_\varepsilon^{-2}A_N
\right)^{-1},
\qquad
\bm m_N(\bm\theta)
=
\sigma_\varepsilon^{-2}
\Sigma_N(\bm\theta)
X_N^\top M_N\bm y_N .
\]

Since $A_N$ is symmetric positive semidefinite,
\[
\R^D
=
\mathcal S_N\oplus\mathcal S_N^\perp,
\qquad
\mathcal S_N^\perp=\ker(A_N).
\]
Moreover,
$X_N^\top M_N\bm y_N\in\mathcal S_N$, and
$\Sigma_N(\bm\theta)$ preserves both
$\mathcal S_N$ and $\mathcal S_N^\perp$.
It follows that
$\bm m_N(\bm\theta)\in\mathcal S_N$.

Let $P_N$ denote the orthogonal projector onto $\mathcal S_N$.
Since $A_N$ vanishes on $\mathcal S_N^\perp$, the posterior precision
there remains equal to the prior precision. Thus
\[
\Sigma_N(\bm\theta)
=
\sigma_1^2(I_D-P_N)
+
P_N\Sigma_N(\bm\theta)P_N .
\]
This proves the stated posterior decomposition.
In particular, for any $\bm s\in\Sph$,
\[
\begin{aligned}
\Var[
\bm v^\top\bm s
\mid
\mathcal D_N,\bm\theta
]
&=
\sigma_1^2
\|(I_D-P_N)\bm s\|^2 +
(P_N\bm s)^\top
\Sigma_N(\bm\theta)
(P_N\bm s).
\end{aligned}
\]
Hence every unit direction
$\bm s\in\mathcal S_N^\perp$
has conditional posterior mean zero and slope variance $\sigma_1^2$,
exactly as under the prior.
\end{proof}

\subsection{General residual-covariance extension}
\label{app:general-residual-covariance}
The rank bound also holds for a general positive-definite
Gaussian residual covariance $\Omega_N$. Let
$W_N:=\Omega_N^{-1}$. Integrating the flat intercept
prior gives the residual precision
\[
\widetilde W_N
=
W_N-W_N\one_N
(\one_N^\top W_N\one_N)^{-1}
\one_N^\top W_N.
\]
Writing
\[
\bm u_N
=
\frac{W_N^{1/2}\one_N}
{\sqrt{\one_N^\top W_N\one_N}},
\]
we have $\widetilde W_N = W_N^{1/2} (I_N-\bm u_N\bm u_N^\top) W_N^{1/2}$,
which is positive semidefinite of rank $N-1$.
Consequently, the data-dependent contribution to the
slope posterior precision,
\(
\mathcal I_N:=X_N^\top\widetilde W_NX_N,
\)
has rank at most $N-1$. The posterior mean lies in
$\operatorname{range}(\mathcal I_N)$, and the posterior
retains its prior covariance on the orthogonal complement.
Thus, general Gaussian residual covariance may alter
the informed directions, but not their rank bound.

\subsection{Post-data rotational symmetry}
\label{app:post-data-rotation}

\begin{corollary}[Residual rotational symmetry after observing data]
\label{cor:post-data-rotation}
Let $P_N$ be the orthogonal projector onto $\mathcal S_N$.
Conditional on
$\bm\theta=(\sigma_1^2,\sigma_\varepsilon^2)$,
let $R\in O(D)$ act as the identity on $\mathcal S_N$ and arbitrarily on
$\mathcal S_N^\perp$. Then
\[
R\bm v\mid\mathcal D_N,\bm\theta
\overset{d}{=}
\bm v\mid\mathcal D_N,\bm\theta.
\]
Consequently, any two subspaces of the same dimension contained in
$\mathcal S_N^\perp$ are posterior-equivalent under the affine working
model. The same invariance holds after marginalising the scalar variance
parameters.
\end{corollary}

\begin{proof}[Proof of Corollary~\ref{cor:post-data-rotation}]
By Proposition~\ref{prop:explored-span}, for every
fixed $\bm\theta$,
\[
\bm m_N(\bm\theta)\in\mathcal S_N,
\qquad
\Sigma_N(\bm\theta)
=
\sigma_1^2(I_D-P_N)
+
P_N\Sigma_N(\bm\theta)P_N.
\]
Any orthogonal $R$ fixing $\mathcal S_N$ therefore satisfies
\[
R\bm m_N(\bm\theta)=\bm m_N(\bm\theta),
\qquad
R\Sigma_N(\bm\theta)R^\top=\Sigma_N(\bm\theta).
\]
This establishes invariance of the conditional Gaussian
posterior. Since the same invariance holds for every
$\bm\theta$, it is preserved under hyperparameter
marginalisation.
\end{proof}

\subsection{Effective informed dimension and uncertainty contrast}
\label{app:effective-informed-dimension}

Let $a_{1,N},\ldots,a_{r_N,N}>0$ be the nonzero
eigenvalues of $A_N$, with orthonormal eigenvectors
$\bm u_{j,N}$. For fixed
$\bm\theta=(\sigma_1^2,\sigma_\varepsilon^2)$,
write $\lambda=\sigma_1^2/\sigma_\varepsilon^2$.
The posterior slope variance along $\bm u_{j,N}$ is
\[
\Var[\bm u_{j,N}^\top\bm v\mid
\mathcal D_N,\bm\theta]
=
\frac{\sigma_1^2}{1+\lambda a_{j,N}},
\]
giving the relative variance reduction
\begin{equation}
c_{j,N}
=
\frac{\lambda a_{j,N}}{1+\lambda a_{j,N}}.
\label{eq:cjn}
\end{equation}

Directions in $\mathcal V_N^\perp$ retain conditional
slope variance $\sigma_1^2$, whereas data-informed
directions have variance reductions $c_{j,N}$.
This uncertainty contrast increases with $\lambda$
and influences EI's implicit off-span exploration.
However, higher-order objective variation may be
absorbed into $\sigma_\varepsilon^2$, making the
fitted $\lambda$ unreliable for controlling exploration.
This motivates the separate schedule of
Section~\ref{sec:exploration_schedule_sub}.
The coefficients $c_{j,N}$ also determine the
posterior updates in the compression
analysis of Appendix~\ref{app:low-rank-posterior-proof}.

\section{Sequential acquisition in data-sized coordinates}
\label{app:acquisition-machinery}

This appendix gives the computational details underlying the sequential
acquisition of Section~\ref{sec:method}. We first express the predictive
distribution in the design-span basis, then construct candidate directions
for EI and score them at the prescribed off-span norm, and finally reconstruct
the selected query in the ambient latent space. The motivation and derivation
of the VoI schedule are given in
Appendix~\ref{app:voi-schedule}.

\subsection{Working-basis representation}
\label{app:working-basis}

This appendix constructs the design-span basis used in
Section~\ref{sec:datasized-bo} and derives the reduced predictive moments
used by EI.

\subsubsection{Working basis}

Let $ X_NX_N^\top = U\operatorname{diag}(\ell_1,\ldots,\ell_N)U^\top$,
with
$\ell_1,\ldots,\ell_{d_N}>0$ and
$\ell_{d_N+1}=\cdots=\ell_N=0$.
An orthonormal basis
$V\in\R^{D\times d_N}$ of
$\mathcal V_N=\operatorname{range}(X_N^\top)$ is
\[
V_{\cdot j}
=
\ell_j^{-1/2}X_N^\top U_{\cdot j},
\qquad
j=1,\ldots,d_N.
\]

Define
\[
\tilde{\bm o}:=U^\top\one_N,
\qquad
\tilde{\bm y}:=U^\top\bm y_N,
\]
and, for hyperparameter atom
$\bm\theta_k=(\sigma_{1,k}^2,\sigma_{\varepsilon,k}^2)$,
\[
\psi_{k,j}
:=
\sigma_{1,k}^2\ell_j+\sigma_{\varepsilon,k}^2,
\qquad
j=1,\ldots,N.
\]
The intercept quantities of
Appendix~\ref{app:intercept-marginalisation} are then
\[
q_N(\bm\theta_k)
=
\sum_{j=1}^N
\frac{\tilde o_j^2}{\psi_{k,j}},
\qquad
\widehat\beta(\bm\theta_k)
=
q_N(\bm\theta_k)^{-1}
\sum_{j=1}^N
\frac{\tilde o_j\tilde y_j}{\psi_{k,j}}.
\]

Let $\tilde{\bm m}_k := V^\top\bm m_N(\bm\theta_k)$,
whose components are
\[
\tilde m_{k,j}
=
\sigma_{1,k}^2\ell_j^{1/2}
\frac{
\tilde y_j-\widehat\beta(\bm\theta_k)\tilde o_j
}{
\psi_{k,j}
},
\qquad
j=1,\ldots,d_N.
\]
Define $\tilde R_k = \operatorname{diag}(\bm d_k) - \kappa_k\bm b_k\bm b_k^\top$,
where, for $j=1,\ldots,d_N$,
\[
d_{k,j}
=
\frac{\sigma_{1,k}^4\ell_j}{\psi_{k,j}},
\qquad
b_{k,j}
=
\frac{\ell_j^{1/2}\tilde o_j}{\psi_{k,j}},
\]
and
\[
\kappa_k
=
\frac{\sigma_{1,k}^4}{q_N(\bm\theta_k)},
\qquad
\tilde{\bm g}_k
=
\frac{\kappa_k}{\sigma_{1,k}^2}\bm b_k.
\]
The matrix $\tilde R_k$ is diagonal minus rank one, so linear systems
involving $\tilde R_k+\nu\mathbf I$ can be evaluated using
Sherman--Morrison rather than a separate dense factorisation for each
candidate.  

\subsubsection{Predictive mean and variance}

Recall the decomposition
$\bm s=V\bm w+t\bm e$ from \eqref{eq:decomposition}.
The posterior slope mean lies in $\mathcal V_N$, so the off-span component
does not contribute to the predictive mean:
\[
m_k(\bm s)
=
\widehat\beta(\bm\theta_k)
+\tilde{\bm m}_k^\top\bm w
=:\tilde m_k(\bm w).
\]

For the predictive variance, the posterior covariance blocks of
$(\beta,\bm v)$ are
\[
\operatorname{Var}(\beta\mid\mathcal D_N,\bm\theta_k)
=q_N(\bm\theta_k)^{-1},
\qquad
\operatorname{Cov}(\bm v,\beta\mid\mathcal D_N,\bm\theta_k)
=-V\tilde{\bm g}_k,
\]
and
\[
\operatorname{Var}(\bm v\mid\mathcal D_N,\bm\theta_k)
=
\sigma_{1,k}^2I_D
-
V\tilde R_kV^\top.
\]
Therefore,
\begin{align*}
s_k^2(\bm s)
&=
\operatorname{Var}(\beta+\bm v^\top\bm s
\mid\mathcal D_N,\bm\theta_k)
\\
&=
q_N(\bm\theta_k)^{-1}
+\sigma_{1,k}^2\|\bm s\|^2
-2\tilde{\bm g}_k^\top\bm w
-\bm w^\top\tilde R_k\bm w
\\
&=
q_N(\bm\theta_k)^{-1}
+\sigma_{1,k}^2
-2\tilde{\bm g}_k^\top\bm w
-\bm w^\top\tilde R_k\bm w =: \tilde s_k^2 (\bm w),
\end{align*}
where we used $V^\top\bm s=\bm w$ and
$\|\bm s\|^2=1$. Thus the conditional predictive distribution at any query on the sphere
can be seen as a function of the $d_N\leq N$ design-span coordinates $\bm w$ only.

\subsection{Candidate directions}
\label{app:pencil}

For a fixed hyperparameter atom and fixed norm $\|\bm w\|$, the first-order
condition for EI gives a one-parameter family of candidate directions.
At a candidate $\bm w$, define
\[
\zeta_k
:=
\frac{m_k(\bm w)-\tau_N}{s_k(\bm w)},
\qquad
\varrho_k
:=
\frac{\phi(\zeta_k)}
{s_k(\bm w)\Phi(\zeta_k)}.
\]
An interior stationary direction satisfies
\begin{equation}
\bm w
\ \propto\
(\tilde R_k+\nu\mathbf I)^{-1}
\left(
\tilde{\bm m}_k-\varrho_k\tilde{\bm g}_k
\right),
\qquad
\nu>-\lambda_{\min}(\tilde R_k).
\label{eq:pencil}
\end{equation}
Because $\varrho_k$ depends on the predictive mean and variance at the
resulting candidate, we resolve it by a short fixed-point iteration for each
value of $\nu$. We evaluate $\nu$ on a grid above the spectral endpoint and
also include the posterior-mean direction
$\tilde{\bm m}_k/\|\tilde{\bm m}_k\|$ whenever it is defined.

\Eqref{eq:pencil} characterises stationary directions for the
single-atom EI problem. A stationary point of the hyperparameter-marginalised
acquisition need not be stationary for any individual atom, because the atom
gradients can cancel. We therefore use the union of the per-atom families as
a candidate construction rather than claiming that it exhausts the stationary
points of marginal EI. Every candidate is subsequently rescored using the
full marginal acquisition \eqref{eq:marginal-ei}. Thus the approximation is
in candidate generation, not in acquisition evaluation.

\subsection{Acquisition scoring and query reconstruction}
\label{app:norm-rays}

For hyperparameter marginalisation, let
$\bm\theta_1,\ldots,\bm\theta_K$ denote the
equally weighted posterior atoms of
Appendix~\ref{app:atoms}.
We use a common incumbent defined by the
marginal posterior mean:
\[
\bar m_N(\bm w)
:=
\frac{1}{K}\sum_{k=1}^K
\tilde m_k(\bm w),
\qquad
\tau_N
:=
\max_{1\leq i\leq N}
\bar m_N(\bm w_i).
\]
Writing
$\gamma_{N,k}(\bm w)
:=\tilde m_k(\bm w)-\tau_N$,
the marginal acquisition is
\begin{equation}
\label{eq:marginal-ei}
\begin{aligned}
\widetilde{\EI}_N^{\mathrm{marg}}(\bm w)
=
\frac{1}{K}\sum_{k=1}^K
\Bigg[
&\gamma_{N,k}(\bm w)
\Phi\left(
\frac{\gamma_{N,k}(\bm w)}
{\tilde s_k(\bm w)}
\right)+
\tilde s_k(\bm w)
\phi\left(
\frac{\gamma_{N,k}(\bm w)}
{\tilde s_k(\bm w)}
\right)
\Bigg].
\end{aligned}
\end{equation}
When hyperparameters are marginalised,
$\widetilde{\EI}_N$ in
\eqref{eq:scheduled-ei} denotes
$\widetilde{\EI}_N^{\mathrm{marg}}$.

Under the prescribed schedule, each unit candidate direction
$\bm u$ from Appendix~\ref{app:pencil} is scored at
$\bm w=\sqrt{1-t_N^2}\,\bm u$ using the full
hyperparameter-marginalised EI. We select the candidate
with the largest acquisition value.
For the unconstrained-EI comparison, we instead score each
direction along the ray $\bm w=\alpha\bm u$,
$\alpha\in[0,1]$, using grids uniform in both $\alpha$
and $t=\sqrt{1-\alpha^2}$. The latter resolves small
off-span norms. Grid sizes are given in
Appendix~\ref{app:defaults}.

Given the selected $\bm w^\star$, the query is reconstructed as
\begin{equation}
\bm s_{N+1}
=
V\bm w^\star+t_N\bm e,
\label{eq:reconstruct}
\end{equation}
where, for $t_N>0$, $\bm e\in\mathcal V_N^\perp$
is the next unused direction from the orthogonal-frame
ledger (Appendix~\ref{app:frame-ledger}).
When $t_N=0$, no off-span direction is required.

\section{Value-of-information schedule for off-span exploration}
\label{app:voi-schedule}

The acquisition formulation of Section~\ref{sec:datasized-bo}
allows any prescribed off-span exploration schedule.
Here we derive the value-of-information (VoI)
proxy used in our experiments.

\subsection{Schedule definition}

The schedule trades off immediate exploitation of learned
directions against the value of exploring new orthogonal
components over the remaining evaluation horizon.
It is controlled by a reference-resolution parameter
$\lambda_{\mathrm{ref}}$, chosen independently of the fitted
variance scales. Larger $\lambda_{\mathrm{ref}}$ represents
greater information gained from a given off-span allocation.

For stages $N=0,\ldots,N_{\max}-1$, let
$h_N:=N_{\max}-N-1$ denote the number of
evaluations remaining after query $\bm s_{N+1}$,
and let $R_N$ denote the cumulative resolved
reference variance from the preceding $N$ queries,
with $R_0=0$.
Treating successive off-span
queries as resolving orthogonal slope components, their contributions add, and $R_N$ is the expected squared norm of the
corresponding posterior mean, so $\sqrt{R_N}$ is a root-mean-square proxy
for the exploitable slope magnitude.
We parameterise the informativeness of a candidate off-span norm $t \in [0,1]$ by $\lambda_{\mathrm{ref}}t^2$; a scalar Gaussian update with unit prior variance then resolves variance
\(
\delta_{\lambda_{\mathrm{ref}}}(t)
:=
\lambda_{\mathrm{ref}}t^2/
(1+\lambda_{\mathrm{ref}}t^2).
\)
The remaining in-span norm is $\sqrt{1-t^2}$, so increasing $t$ trades
current exploitation against directional information that can be used over
the remaining $h_N$ evaluations. We therefore prescribe
\begin{equation}
t_N
:=
\argmax_{t\in[0,1]}
\left\{
\sqrt{1-t^2}\sqrt{R_N}
+
h_N\sqrt{R_N+\delta_{\lambda_{\mathrm{ref}}}(t)}
\right\},
\qquad
R_{N+1}=R_N+\delta_{\lambda_{\mathrm{ref}}}(t_N).
\label{eq:voi-schedule-main}
\end{equation}
The first term represents exploitation by the current query, while the
second represents the value of the enlarged learned slope over the remaining
evaluation horizon. 

The recursion depends only on $(N_{\max},\lambda_{\mathrm{ref}})$ and can
be precomputed. It starts with $t_0=1$, decreases monotonically
over the evaluation horizon, and reaches
$t_{N_{\max}-1}=0$.
In practice, we use a fixed large
$\lambda_{\mathrm{ref}}$ for all smooth objectives, where small perturbations
remain informative, and a smaller $\lambda_{\mathrm{ref}}$ for noisy or
quantised objectives, where fine perturbations may be obscured or produce no
change. The local-resolvability analysis in
Appendix~\ref{sec:local-resolvability} motivates this distinction and our
experimental values. Once fine perturbations are well resolved,
the large value of $\lambda_{\mathrm{ref}}$ is of secondary
importance: the off-span norm depends only weakly on it.

\subsection{Scalar reference model}

Consider a unit-variance reference slope coefficient
$\xi\sim\mathcal N(0,1)$ associated with an off-span direction. A query with
off-span norm $t$ is represented by the scalar observation
\begin{equation}
r=t\xi+\epsilon,
\qquad
\epsilon\sim\mathcal N(0,\lambda_{\mathrm{ref}}^{-1}),
\label{eq:voi-reference-observation}
\end{equation}
where $\lambda_{\mathrm{ref}}>0$ is the reference directional
signal-to-noise ratio of a fully off-span query. The corresponding
signal-to-noise ratio is therefore $\lambda_{\mathrm{ref}}t^2$.

Gaussian conjugacy gives
\[
\Var(\xi\mid r)
=
\frac{1}{1+\lambda_{\mathrm{ref}}t^2}.
\]
Hence the fraction of reference prior variance resolved by the observation is
\begin{equation}
\delta_{\lambda_{\mathrm{ref}}}(t)
:=
1-\Var(\xi\mid r)
=
\frac{\lambda_{\mathrm{ref}}t^2}
{1+\lambda_{\mathrm{ref}}t^2}.
\label{eq:voi-resolved-increment}
\end{equation}
By the law of total variance, $\E_r\!\left[\E(\xi\mid r)^2\right] = \delta_{\lambda_{\mathrm{ref}}}(t)$,
so the same quantity is the expected squared posterior mean acquired along the reference direction.

Successive off-span queries are treated as probing orthogonal reference
coefficients. Let $\bm\mu_N^{\mathrm{ref}}$ collect their scalar posterior
means after the preceding $N$ reference updates, and define $R_N := \E\bigl[\|\bm\mu_N^{\mathrm{ref}}\|^2\bigr]$, where
$R_0=0$.
Because the reference directions are orthogonal, their contributions add.
A query with off-span norm $t_N$ therefore gives
\begin{equation}
R_{N+1}
=
R_N+\delta_{\lambda_{\mathrm{ref}}}(t_N).
\label{eq:voi-R-update}
\end{equation}
Thus $R_N$ is an auxiliary state of the reference schedule: it is the
expected squared norm of the collection of resolved reference slope
components, not the squared norm of the fitted posterior mean used by EI.

\subsection{Value-of-information objective}
We use $\sqrt{R_N}$ as a root-mean-square proxy for the currently
exploitable reference slope magnitude.

Allocating norm $t$ off-span leaves in-span norm $\sqrt{1-t^2}$. The
current-query exploitation term is therefore approximated by
\(
\sqrt{1-t^2}\sqrt{R_N}.
\)
After the off-span observation, the resolved reference mass becomes
$R_N+\delta_{\lambda_{\mathrm{ref}}}(t)$. Under an explore-then-commit
approximation, each of the $h_N$ remaining evaluations is valued at the
updated root-mean-square magnitude. This gives the VoI proxy
\begin{equation}
J_N(t)
:=
\sqrt{1-t^2}\sqrt{R_N}
+
h_N
\sqrt{
R_N+\delta_{\lambda_{\mathrm{ref}}}(t)
},
\qquad
t\in[0,1].
\label{eq:voi-objective}
\end{equation}
Maximising this objective gives
\[
t_N
=
\argmax_{t\in[0,1]}J_N(t),
\]
together with the recursion
\eqref{eq:voi-R-update}, which is exactly the schedule stated in
\eqref{eq:voi-schedule-main}.

The approximation has three deliberate simplifications. First,
$\sqrt{R_N}$ is a root-mean-square proxy rather than the exact expected
exploitation value. Second, the calculation values the information acquired
by the current off-span component over the remaining horizon but does not
model further information acquisition by those later evaluations. Third,
the scalar reference observation ignores residual uncertainty in the in-span
contribution. These choices make the schedule depend only on
$(N_{\max},\lambda_{\mathrm{ref}})$ and not on fitted hyperparameters or
realised responses.

\subsection{Schedule properties}\label{sec:schedule_properties}

The VoI rule automatically anneals the prescribed off-span
allocation. As reference information accumulates and the remaining horizon
shrinks, the optimal off-span norm decreases until exploration terminates.

\begin{proposition}[Monotone decay and stopping]
\label{prop:voi-monotone}
Fix $\lambda_{\mathrm{ref}}>0$ and $N_{\max}\geq2$,
and let $(t_N,R_N)$ be generated by
\eqref{eq:voi-schedule-main} for
$N=0,\ldots,N_{\max}-1$,
with $R_0=0$ and $h_N:=N_{\max}-N-1$.
Then
\[
t_0=1,\qquad
t_{N+1}\leq t_N
\quad (0\leq N<N_{\max}-1).
\]
For every $N$ with $R_N>0$,
\begin{equation}
t_N>0
\quad\Longleftrightarrow\quad
R_N<h_N\lambda_{\mathrm{ref}}.
\label{eq:voi-stopping}
\end{equation}
Moreover, $R_N$ increases strictly while $t_N>0$ and remains constant
once $t_N=0$. Hence off-span exploration cannot restart, and in particular
$t_{N_{\max}-1}=0$.
\end{proposition}

\begin{proof}
Set $x=t^2$ and write
\[
J_N(x)
=
\sqrt{R_N}\sqrt{1-x}
+
h_N
\sqrt{
R_N+
\frac{\lambda_{\mathrm{ref}}x}
     {1+\lambda_{\mathrm{ref}}x}
},
\qquad x\in[0,1].
\]
For $R_N>0$, $J_N$ is strictly concave in $x$ and therefore has a unique
maximiser. At $N=0$, $R_0=0$ and $h_0>0$,
so $J_0$ is increasing and $t_0=1$.

For $R_N>0$,
\[
J_N'(x)
=
-\frac{\sqrt{R_N}}{2\sqrt{1-x}}
+
\frac{
h_N\lambda_{\mathrm{ref}}
}{
2(1+\lambda_{\mathrm{ref}}x)^2
\sqrt{
R_N+\lambda_{\mathrm{ref}}x/(1+\lambda_{\mathrm{ref}}x)
}
}.
\]
For fixed $x$, this derivative decreases with $R_N$ and increases with
$h_N$. Since $ R_{N+1}=R_N+\delta_{\lambda_{\mathrm{ref}}}(t_N)\geq R_N$ and $h_{N+1}=h_N-1$,
we have $J_{N+1}'(x)\leq J_N'(x)$. The unique maximiser therefore cannot
move to the right, which gives $t_{N+1}\leq t_N$.

At $x=0$,
\[
J_N'(0)
=
\frac{h_N\lambda_{\mathrm{ref}}-R_N}{2\sqrt{R_N}},
\]
so strict concavity gives \eqref{eq:voi-stopping}. Finally,
$\delta_{\lambda_{\mathrm{ref}}}(t)>0$ exactly when $t>0$, so $R_N$
increases while exploration is active and is constant thereafter. Since
$h_N$ continues to decrease, once \eqref{eq:voi-stopping} fails it cannot
become active again. The final round has $h_{N_{\max}-1}=0$,
and hence $t_{N_{\max}-1}=0$.
\end{proof}

When \eqref{eq:voi-stopping} holds, the maximiser is interior and satisfies
\[
\frac{\sqrt{R_N}}
{\sqrt{1-t_N^2}}
=
\frac{
h_N\lambda_{\mathrm{ref}}
}{
(1+\lambda_{\mathrm{ref}}t_N^2)^2
\sqrt{
R_N+\delta_{\lambda_{\mathrm{ref}}}(t_N)
}
}.
\]
The left-hand side increases and the right-hand side decreases in $t_N^2$,
so the solution is unique and can be obtained by one-dimensional bisection.
Since the recursion depends only on $(N_{\max},\lambda_{\mathrm{ref}})$,
the complete schedule can be precomputed.

\begin{corollary}[Persistence of off-span exploration]
\label{cor:voi-persistence}
Let $N_\star$ be the first round with $t_{N_\star}=0$. Then
\[
N_\star
\geq
\left\lceil
\frac{1+\lambda_{\mathrm{ref}}}
     {2+\lambda_{\mathrm{ref}}}
(N_{\max}-1)
\right\rceil.
\]
\end{corollary}

\begin{proof}
Each query resolves at most
$\delta_{\lambda_{\mathrm{ref}}}(1)
=\lambda_{\mathrm{ref}}/(1+\lambda_{\mathrm{ref}})$, so
\[
R_N
\leq
N\frac{\lambda_{\mathrm{ref}}}
{1+\lambda_{\mathrm{ref}}}.
\]
Combining this with the stopping requirement
$R_N\geq h_N\lambda_{\mathrm{ref}}$ gives the result.
\end{proof}

Thus larger $\lambda_{\mathrm{ref}}$ places a later lower bound on when
off-span exploration can terminate. In the high-reference-SNR regime this
permits very small, but necessarily persistent, off-span allocations.

\subsection{Interpretation of the reference resolution parameter}

The reference parameter $\lambda_{\mathrm{ref}}$ controls
the assumed resolution of fresh off-span probes, independently
of the fitted signal--discrepancy ratio.
By \eqref{eq:voi-resolved-increment}, increasing
$\lambda_{\mathrm{ref}}$ credits a given off-span allocation
with more resolved variance, without necessarily increasing
$t_N$. Corollary~\ref{cor:voi-persistence} further shows
that larger references delay the earliest possible
termination of off-span exploration.
The reference remains external to the fitted model because
higher-order objective structure can inflate the discrepancy
scale, making the fitted variance ratio unreliable
for controlling exploration. For intuition, consider an interior solution with
$\lambda_{\mathrm{ref}}t_N^2\gg1$, 
$t_N^2\ll1$ and
$R_N\gg\delta_{\lambda_{\mathrm{ref}}}(t_N)$.
The first-order condition then gives
\[
t_N
\approx
\left(
\frac{h_N}{\lambda_{\mathrm{ref}}R_N}
\right)^{1/4}.
\]
Thus $t_N$ decreases as reference information accumulates
and increases with the remaining horizon; a large reference
also permits small but persistent off-span probes.
This approximation is interpretive: the schedule uses
the original one-dimensional objective.
Changing $N_{\max}$ requires recomputing the schedule,
which is a finite-horizon proxy rather than an
anytime or Bayes-optimal policy.

\subsection{Local resolvability and reference resolution regimes}
\label{sec:local-resolvability}

The characteristic reference scale is
$\chi_{\mathrm{ref}}:=\lambda_{\mathrm{ref}}^{-1/2}$:
larger $\lambda_{\mathrm{ref}}$ assumes finer directional
resolution, rather than estimating the objective's
physical signal-to-noise ratio.
Small perturbations remain informative for smooth objectives,
but may be obscured by observation noise or produce no
change under discrete decoding or local plateaux.

\subsubsection{Characterising local resolvability}

We use two complementary measures to distinguish these behaviours. The
statistic $Q$ measures how objective changes vary as the perturbation size is
reduced, while $Z$ measures how often a perturbation leaves the objective
exactly unchanged. Together they indicate whether directional information
remains resolved at fine perturbation scales or is only revealed by broader
perturbations.

We define these quantities using paired local perturbations on the sphere.
Let $\bm s_a\in\mathbb S^{D-1}$ be an anchor and let
$\bm e_{aj}\perp\bm s_a$ be a unit direction. For perturbation scale
$\chi\in(0,1)$, define
\[
    \bm s_{aj}^{\pm}(\chi)
    =
    \sqrt{1-\chi^2}\,\bm s_a
    \pm
    \chi\,\bm e_{aj},
    \qquad
    d_{aj}(\chi)
    =
    \frac{
        F(\bm s_{aj}^{+}(\chi))
        -
        F(\bm s_{aj}^{-}(\chi))
    }{2}.
\]
The two perturbed points remain on $\mathbb S^{D-1}$ and lie at equal
angular distance $\arcsin(\chi)$ from the anchor. Aggregating their
directional contrasts over $A$ anchors and $J$ directions gives the
root-mean-square local response
\[
    P(\chi)
    =
    \left[
        \frac{1}{AJ}
        \sum_{a=1}^{A}\sum_{j=1}^{J}
        d_{aj}(\chi)^2
    \right]^{1/2}.
\]

The statistic
\[
    Q
    =
    \frac{P(\chi_{\mathrm{s}})/\chi_{\mathrm{s}}}
         {P(\chi_{\ell})/\chi_{\ell}}
\]
compares the objective change per unit perturbation at a fine and a broader
scale. For a smooth deterministic objective, reducing the perturbation size
also reduces the objective change proportionally, giving $Q\approx1$.
Stochastic variation does not shrink with the perturbation size in the same
way, so it contributes relatively more at $\chi_{\mathrm{s}}$ and tends to
increase
$Q$ above one.
We therefore use $Q$ as a measure of fine-scale resolvability rather than
as a direct estimate of noise magnitude.

A different failure of local resolvability occurs when the objective is
locally constant. This can arise, for example, when a continuous latent
perturbation does not change the generated output, or when several generated outputs receive exactly the same objective value. In such cases, moving a
small distance in a fresh direction can produce no observable change in the
objective, so that direction is unresolved at that perturbation scale.
We measure this behaviour by the exact-zero fraction
\[
    Z(\chi)
    =
    \frac{1}{AJ}
    \sum_{a=1}^{A}\sum_{j=1}^{J}
    \mathbb I\!\left\{d_{aj}(\chi)=0\right\}.
\]
Thus $Z(\chi)$ is the fraction of probed directions for which the paired
perturbation produces exactly no change in the objective. 

Fine off-span probes preserve more norm for exploitation,
but are useful only when directional changes remain resolved.
A large $\lambda_{\mathrm{ref}}$ assumes such fine-scale
resolution, whereas a smaller reference reflects the need
for broader perturbations under noise or local plateaux.
We use $Q$ and $Z$ to distinguish two fixed
reference-resolution regimes instead of estimating
$\lambda_{\mathrm{ref}}$ continuously.
Values $Q\approx1$ and $Z(\chi_{\mathrm{s}})=0$ support
the fine-resolution regime; 
elevated $Q$ is consistent with fine-scale changes
being obscured by stochastic variation,
while positive $Z(\chi_{\mathrm{s}})$ indicates
unresolved changes due to plateaux or quantisation.
Either supports the coarse-resolution regime.

When the local-resolvability regime is not known from the objective
characteristics, these statistics provide an empirical diagnostic using $A=5$ independently drawn anchors
and $J=2$ fresh directions per anchor, with
\(
\chi_{\mathrm{s}}=0.00625,
\)
and 
\(
\chi_{\ell}=0.05.
\)
Evaluating paired perturbations at both scales requires $4AJ=40$ objective
evaluations. In the experiments reported in this paper, the regime is assigned
from known objective characteristics; the $Q$--$Z$ analysis below is used to
validate these assignments rather than as an additional calibration step in
each optimisation run.
We validated this diagnostic on the objectives of
Appendix~\ref{sec:benchmarks} using $2{,}000$ random $A=5$, $J=2$
sub-designs per objective. Resolved objectives gave median $Q\approx1$, with
$Z(\chi_{\mathrm{s}})=0$ on every draw. Noise-dominated objectives gave
median $Q\approx8$, as expected when the local response becomes approximately
independent of perturbation scale, for which
$Q\approx\chi_{\ell}/\chi_{\mathrm{s}}=8$. Locally unresolved objectives
were instead identified by $Z(\chi_{\mathrm{s}})>0$. Across all $2{,}000$
random sub-designs per objective, the diagnostic outcomes for objectives
assigned to the fine- and coarse-resolution regimes did not overlap.

\subsubsection{Reference-resolution regimes}

\paragraph{Fine-resolution regime.}
For an active interior allocation $t_N\in(0,1)$,
the stationarity condition in Appendix \ref{sec:schedule_properties} implies,
for fixed $R_N>0$ and $h_N>0$ as
$\lambda_{\mathrm{ref}}\to\infty$,
\begin{equation}
t_N\to0,
\qquad
\lambda_{\mathrm{ref}}t_N^2\to\infty.
\label{eq:high-lambda-limits}
\end{equation}
Indeed, if $t_N$ remained bounded away from zero,
the marginal information gain would vanish while
the marginal exploitation cost remained positive.
Conversely, if $\lambda_{\mathrm{ref}}t_N^2$
remained bounded, the marginal information gain
would diverge while the exploitation cost remained finite,
contradicting stationarity in either case.
Consequently,
$\delta_{\lambda_{\mathrm{ref}}}(t_N)\to1$
and $\sqrt{1-t_N^2}\to1$.
Applying these limits to the stationarity condition gives
\[
t_N^4
\sim
\frac{h_N}
{\lambda_{\mathrm{ref}}\sqrt{R_N(R_N+1)}},
\]
so $t_N=O(\lambda_{\mathrm{ref}}^{-1/4})$.
Moreover,
\[
1-\delta_{\lambda_{\mathrm{ref}}}(t_N)
\sim
\frac{[R_N(R_N+1)]^{1/4}}
{\sqrt{h_N\lambda_{\mathrm{ref}}}}
=
O(\lambda_{\mathrm{ref}}^{-1/2}).
\]
Thus the reference-information deficit vanishes as
$\lambda_{\mathrm{ref}}^{-1/2}$, while the off-span norm
varies only with its fourth root.
This diminishing sensitivity motivates using
$\lambda_{\mathrm{ref}}=10^5$ as a fixed representative
of the fine-resolution regime rather than a tuned optimum.

\paragraph{Coarse-resolution regime.}
We choose a fixed reference that credits a maximally
off-span probe with substantial information.
Since $\delta_{\lambda_{\mathrm{ref}}}(t)$ increases
with $t$ and $t_0=1$, we require
\[
\delta_{\lambda_{\mathrm{ref}}}(1)
=
\frac{\lambda_{\mathrm{ref}}}{1+\lambda_{\mathrm{ref}}}
\geq0.9.
\]
This gives $\lambda_{\mathrm{ref}}\geq9$, and we use the convenient decade
value $\lambda_{\mathrm{ref}}=10$. The criterion ensures that, although
smaller probes receive less reference-information credit in the coarse-resolution
regime, a sufficiently broad probe can still substantially resolve a fresh
direction.
The value $\lambda_{\mathrm{ref}}=10$ is not intended as a finely calibrated optimum. Indeed,
\[
\frac{d\log \delta_{\lambda_{\mathrm{ref}}}(1)}
     {d\log \lambda_{\mathrm{ref}}}
=
\frac{1}{1+\lambda_{\mathrm{ref}}},
\]
so moderate 
changes in $\lambda_{\mathrm{ref}}$ have only a weak effect on the
property used to define the coarse regime. We consequently treat
$\lambda_{\mathrm{ref}}=10$ as a fixed representative of that regime rather
than as a tuned objective-specific value.

\section{Bounded-cost model refitting}
\label{app:bounded-refitting}

The data-informed posterior rank can grow with the evaluation horizon even
though it is independent of the ambient latent dimension. We first quantify
the effect of discarding weak data-induced posterior updates and derive
finite-sample stability bounds for EI. We then use the same spectral
ordering to motivate a bounded-state observation representation for online
refitting.

\subsection{Ideal spectral truncation}
\label{app:low-rank-posterior-proof}

Recall from Appendix~\ref{app:effective-informed-dimension} the nonzero
eigenpairs
$(a_{j,N},\bm u_{j,N})$, $j=1,\ldots,r_N$, of $A_N$, ordered so that
\(
a_{1,N}\geq\cdots\geq a_{r_N,N}>0,
\)
and the corresponding relative posterior variance reductions $c_{j,N}$
defined in \eqref{eq:cjn}. Conditional on
$\bm\theta=(\sigma_1^2,\sigma_\varepsilon^2)$, the exact posterior
covariance is
\[
\Sigma_N
=
\sigma_1^2
\left(
I_D
-
\sum_{j=1}^{r_N}
c_{j,N}\bm u_{j,N}\bm u_{j,N}^\top
\right).
\]

For $0\leq r\leq r_N$, retain the $r$ strongest updates and define
\[
\bm m_N^{(r)}
:=
\sum_{j=1}^r
\bm u_{j,N}\bm u_{j,N}^\top\bm m_N,
\qquad
\Sigma_N^{(r)}
:=
\sigma_1^2
\left(
I_D
-
\sum_{j=1}^r
c_{j,N}\bm u_{j,N}\bm u_{j,N}^\top
\right).
\]
Thus the retained directions preserve their posterior updates, while the
discarded directions revert to their prior mean and variance. Define
\[
T_{N,r}
:=
\sum_{j=r+1}^{r_N}c_{j,N},
\qquad
c_{r_N+1,N}:=0.
\]

\begin{proposition}[Posterior error under spectral truncation]
\label{prop:low-rank-posterior}
For every $0\leq r\leq r_N$,
\[
\Sigma_N^{(r)}-\Sigma_N
=
\sigma_1^2
\sum_{j=r+1}^{r_N}
c_{j,N}\bm u_{j,N}\bm u_{j,N}^\top
\succeq0.
\]
Consequently,
\[
\left\|
\Sigma_N^{(r)}-\Sigma_N
\right\|_{\mathrm{op}}
=
\sigma_1^2c_{r+1,N},
\qquad
\tr\!\left(
\Sigma_N^{(r)}-\Sigma_N
\right)
=
\sigma_1^2T_{N,r}.
\]
Under the affine working model at fixed design and hyperparameters,
\[
\E\!\left[
\|\bm m_N-\bm m_N^{(r)}\|^2
\right]
=
\sigma_1^2T_{N,r}.
\]
\end{proposition}

\begin{proof}
The covariance identities follow directly from the exact and truncated
spectral representations.

For the posterior mean,
\(
\bm m_N
=
\sigma_\varepsilon^{-2}
\Sigma_NX_N^\top M_N\bm y_N.
\)
Under the affine working model at fixed design,
\(
M_N\bm y_N
=
M_NX_N\bm v+M_N\bm\varepsilon_N,
\)
with
$\bm v\sim\mathcal N(\bm0,\sigma_1^2I_D)$ and
$\bm\varepsilon_N\sim
\mathcal N(\bm0,\sigma_\varepsilon^2I_N)$.
Diagonalising in the eigenbasis of $A_N$ gives
\[
\Cov(\bm m_N)
=
\sigma_1^2
\sum_{j=1}^{r_N}
c_{j,N}\bm u_{j,N}\bm u_{j,N}^\top.
\]
Projecting onto the discarded eigenspace and taking the trace yields
$\E\|\bm m_N-\bm m_N^{(r)}\|^2=\sigma_1^2T_{N,r}.$
\end{proof}

Proposition~\ref{prop:low-rank-posterior} controls the two posterior
quantities affected by truncation: the discarded posterior-mean component
through $T_{N,r}$ and the largest omitted covariance update through
$c_{r+1,N}$. We now propagate these errors to the predictive distribution
and hence to EI. This yields both realised-data bounds on the
acquisition functions and expected bounds on the loss from optimising their
rank-$r$ approximations.

For a fixed realised response vector, the full and truncated predictive
distributions differ only through
$(\bm m_N,\Sigma_N)$ and
$(\bm m_N^{(r)},\Sigma_N^{(r)})$, respectively. After intercept
marginalisation,
\[
\mu_N(\bm s)
=
\overline y_N
+
(\bm s-\bm\mu_N)^\top\bm m_N,
\qquad
s_N^2(\bm s)
=
\frac{\sigma_\varepsilon^2}{N}
+
(\bm s-\bm\mu_N)^\top
\Sigma_N
(\bm s-\bm\mu_N),
\]
with $\mu_N^{(r)}$ and $s_N^{(r)}$ obtained by replacing
$(\bm m_N,\Sigma_N)$ with
$(\bm m_N^{(r)},\Sigma_N^{(r)})$.

\begin{proposition}[Low-rank acquisition stability]
\label{prop:low-rank-acquisition}
For fixed $\mathcal D_N$ and $\bm\theta$,
\[
\sup_{\bm s\in\Sph}
\left|
\operatorname{EI}_N(\bm s)
-
\operatorname{EI}_N^{(r)}(\bm s)
\right|
\leq
2\|\bm m_N-\bm m_N^{(r)}\|
+
2\sigma_1
\sqrt{\frac{c_{r+1,N}}{2\pi}}.
\]

If the full and truncated acquisitions are maximised over the same feasible
query domain, then under the affine working model at fixed design and
hyperparameters,
\begin{equation}
\E\!\left[
\operatorname{EI}_N(\bm s_N^\star)
-
\operatorname{EI}_N(\bm s_{N,r}^\star)
\right]
\leq
4\sigma_1
\left(
\sqrt{T_{N,r}}
+
\sqrt{\frac{c_{r+1,N}}{2\pi}}
\right).
\label{eq:ei-low-rank-bound}
\end{equation}
\end{proposition}

\begin{proof}
Let $\tau_N=\max_i\mu_N(\bm s_i)$ and
$\tau_N^{(r)}=\max_i\mu_N^{(r)}(\bm s_i)$
denote the respective posterior-mean incumbents.
Since
\(
\mu_N(\bm s)-\tau_N
=
\bm s^\top\bm m_N
-\max_i\bm s_i^\top\bm m_N,
\)
and analogously for the truncated posterior,
\[
\left|
(\mu_N(\bm s)-\tau_N)
-
(\mu_N^{(r)}(\bm s)-\tau_N^{(r)})
\right|
\leq
2\|\bm m_N-\bm m_N^{(r)}\|.
\]
Proposition~\ref{prop:low-rank-posterior} similarly gives $|s_N^{(r)}(\bm s)-s_N(\bm s)| \leq 2\sigma_1\sqrt{c_{r+1,N}}$.
Gaussian EI is Lipschitz in its predictive mean with constant $1$ and in
its predictive standard deviation with constant $1/\sqrt{2\pi}$, yielding
the EI bound.
If two acquisitions differ uniformly by at most $\varepsilon$, maximising
one loses at most $2\varepsilon$ under the other. Finally,
Proposition~\ref{prop:low-rank-posterior} and Jensen's inequality give $\E\|\bm m_N-\bm m_N^{(r)}\| \leq \sigma_1\sqrt{T_{N,r}}$,
which yields \eqref{eq:ei-low-rank-bound}.
\end{proof}

These ideal truncation results provide an a priori retention principle:
directions associated with larger eigenvalues of $A_N$ produce stronger
posterior updates for every variance ratio and are therefore retained
preferentially. The actual bounded-state algorithm, however, repeatedly
compresses an observation-space representation so that the variance
hyperparameters can continue to be refitted; we describe that representation
next.

\subsection{Compressed observation state}
\label{app:reservoir-motivation}

Let $M$ denote the compression budget. The live state consists of
pseudo-observations
$(\check X,\check{\bm y},\check{\bm o})$, where
$\check X\in\R^{m\times D}$ and
$\check{\bm y},\check{\bm o}\in\R^m$.
The vector $\check{\bm o}$ replaces the original ones vector after
observation-space compression.

The compression is constructed to preserve
\begin{equation}
\begin{aligned}
\check X^\top(\check{\bm y}-\beta\check{\bm o})
&=
X_N^\top(\bm y_N-\beta\one_N)
\qquad \text{for all }\beta,\\
\|\check{\bm y}\|^2
&=
\|\bm y_N\|^2,\\
\check{\bm o}^\top\check{\bm y}
&=
\one_N^\top\bm y_N,\\
\|\check{\bm o}\|^2
&=
N.
\end{aligned}
\label{eq:invariants}
\end{equation}
These identities preserve the slope--response cross-products and scalar
residual quantities required for intercept marginalisation. The true
observation count and acquisition incumbent are tracked separately from the
pseudo-observations.

\paragraph{Orthogonal observation-space compression.}
\label{app:reservoir-validity}
If $W$ has orthonormal columns, replacing
$(X,\bm y,\bm o)$ by
$(W^\top X,W^\top\bm y,W^\top\bm o)$
is an exact change of coordinates on the retained observation subspace.
Gaussian residuals remain white under this transformation, so the affine
likelihood and posterior formulas remain valid. The ones vector becomes the
general coordinate vector $W^\top\bm o$, which is why
$\check{\bm o}$ must be carried explicitly.

\subsection{Recompression and evidence correction}
\label{app:reservoir-truncation}

When recompression is triggered, we diagonalise its Gram matrix
and retain a bounded collection of eigendirections. Leading eigendirections
are retained preferentially because they correspond to the strongest
posterior variance reductions for every fixed variance ratio. A sparse
selection from the remaining spectrum preserves some memory of weaker
directions introduced by continued off-span exploration.

The discarded spectrum is additionally summarised by a small collection of
augmentation rows. Within each spectral bin, these rows are constructed
from inverse-moment vectors of the observation and intercept coordinates,
\[
\Lambda^{-k}\check{\bm y},
\qquad
\Lambda^{-k}\check{\bm o},
\qquad
k=0,\ldots,\mu-1,
\]
restricted to that bin, where $\Lambda$ contains the corresponding
discarded Gram eigenvalues. The zeroth moments enforce the exact statistics
in \eqref{eq:invariants}, while higher moments improve the approximation of
the shrinkage-weighted quantities entering the posterior and marginal
likelihood across hyperparameter values. The numbers of retained
directions, spectral bins, and moments used in the experiments are given in
Appendix~\ref{app:experimental-setup}.

The pseudo-state alone does not preserve the full determinant contribution
to the marginal likelihood because discarded Gram eigenvalues are no longer
represented explicitly. We therefore maintain a hyperparameter-dependent
log-determinant correction. If $g_j^{\mathrm{drop}}$ are eigenvalues removed
from the live state and $g_i^{\mathrm{aug}}$ are eigenvalues represented by
the augmentation rows, we update
\begin{equation}
\Delta_{\det}(\bm\theta)
\mathrel{+}=
\sum_{j\,\mathrm{drop}}
\log(\sigma_1^2g_j^{\mathrm{drop}}+\sigma_\varepsilon^2)
-
\sum_{i\,\mathrm{aug}}
\log(\sigma_1^2g_i^{\mathrm{aug}}+\sigma_\varepsilon^2).
\label{eq:logdet-acc}
\end{equation}
This corrects the leading determinant bias that would otherwise arise from
evaluating the evidence using only the compressed state.

\subsection{Stability under bounded-cost refitting}
\label{app:bounded-refitting-stability}

We analyse the stability of the posterior and acquisition
function under repeated recompression. We first consider
fixed variance hyperparameters and a common hyperparameter
distribution, before allowing the hyperparameter weights
to change after refitting on the compressed state.

\subsubsection{Stability with fixed hyperparameter weights}
\label{app:fixed-weight-stability}

Unlike the ideal spectral truncation analysed in
Appendix~\ref{app:low-rank-posterior-proof},
bounded-cost refitting repeatedly compresses an
observation-space representation as new observations
arrive. We characterise the resulting approximation
in terms of the accumulated error in the design
Gram matrix.

At round $N$, let
\[
M_N
:=
I_N-\frac1N\one_N\one_N^\top,
\qquad
A_N
:=
X_N^\top M_NX_N,
\qquad
\bm b_N
:=
X_N^\top M_N\bm y_N .
\]
For the current pseudo-observation state
$(\check X_N,\check{\bm y}_N,\check{\bm o}_N)$ with $m$ rows, define
\[
\check M_N
:=
\mathbf I_m-\frac1N
\check{\bm o}_N\check{\bm o}_N^\top,
\qquad
\check A_N
:=
\check X_N^\top\check M_N\check X_N.
\]
Because $\|\check{\bm o}_N\|^2=N$, $\check M_N$ is the orthogonal
projector that removes the compressed intercept direction.

The invariants of \eqref{eq:invariants} imply $ \check X_N^\top\check M_N\check{\bm y}_N = X_N^\top M_N\bm y_N = \bm b_N$.
Thus the centred slope--response statistic is preserved exactly through
recompression. Moreover, writing
\[
G_N:=X_N^\top X_N,
\qquad
\check G_N:=\check X_N^\top\check X_N,
\]
and using
\(
\check X_N^\top\check{\bm o}_N=X_N^\top\one_N
\),
we obtain
\(
A_N-\check A_N
=
G_N-\check G_N.
\)
Hence, conditional on the variance hyperparameters, the only approximate
sufficient statistic entering the slope posterior is the centred design Gram
matrix $\check A_N$; the centred slope--response statistic is preserved
exactly.

\begin{proposition}[Accumulation of streaming Gram error]
\label{prop:streaming-gram-error}
Let
\(
E_N:=A_N-\check A_N.
\)
Suppose the pseudo-observation representation is initially exact, and let
$J_N$ denote the number of recompressions performed up to round $N$.
For recompression $j$, let $\check X_j^-$ and $\check X_j^+$ denote the
pseudo-design matrices immediately before and after recompression, and define
the local Gram defect
\[
D_j
:=
(\check X_j^-)^\top\check X_j^-
-
(\check X_j^+)^\top\check X_j^+.
\]
Then
$E_N
=
\sum_{j=1}^{J_N}D_j$, and 
consequently,
\begin{equation*}
\rho_N
:=
\|E_N\|_{\mathrm{op}}
\leq
\sum_{j=1}^{J_N}\|D_j\|_{\mathrm{op}}.
\end{equation*}
\end{proposition}

\begin{proof}
Appending a raw observation adds the same rank-one
matrix to the exact and pseudo-state Gram matrices,
leaving $E_N$ unchanged. Recompression $j$ changes
the pseudo-state Gram matrix by $-D_j$, so $E_N$
increases by $D_j$. Summing over recompressions
proves $E_N
=
\sum_{j=1}^{J_N}D_j$; the norm bound
follows by the triangle inequality.
\end{proof}

Proposition~\ref{prop:streaming-gram-error} separates the approximation
introduced by recompression from subsequent data arrival. Later observations
can change the posterior consequence of an existing Gram perturbation, but
they do not enlarge the perturbation itself: additional Gram error enters
only when another recompression is performed. The result does not depend on
how the post-compression pseudo-observations are constructed and therefore
applies directly to the spectral retention and augmentation procedure used
by the bounded-state algorithm.

\begin{proposition}[Conditional posterior stability under streaming compression]
\label{prop:streaming-posterior}
Fix
\(
\bm\theta=(\sigma_1^2,\sigma_\varepsilon^2)
\)
and write
\(
\lambda:=\sigma_1^2/\sigma_\varepsilon^2.
\)
Define
\[
R_N:=(I_D+\lambda A_N)^{-1},
\qquad
\check R_N:=(I_D+\lambda\check A_N)^{-1}.
\]
The exact and compressed conditional slope posteriors satisfy $ \Sigma_N=\sigma_1^2R_N$ and  $\check{\Sigma}_N=\sigma_1^2\check R_N$.
Because $\bm b_N$ is preserved exactly, we have $\bm m_N=\lambda R_N\bm b_N$ and  $\check{\bm m}_N=\lambda\check R_N\bm b_N$.
Then
\begin{equation}
\|\Sigma_N-\check{\Sigma}_N\|_{\mathrm{op}}
\leq
\sigma_1^2\lambda\rho_N,
\label{eq:streaming-cov-bound}
\end{equation}
and
\begin{equation}
\|\bm m_N-\check{\bm m}_N\|
\leq
\lambda\rho_N\|\check{\bm m}_N\|.
\label{eq:streaming-mean-bound}
\end{equation}
\end{proposition}

\begin{proof}
The resolvent identity gives $R_N-\check R_N = -\lambda R_NE_N\check R_N $.
Since $A_N$ and $\check A_N$ are positive semidefinite, $\|R_N\|_{\mathrm{op}}\leq1$ and $\|\check R_N\|_{\mathrm{op}}\leq1$.
Hence $\|R_N-\check R_N\|_{\mathrm{op}} \leq \lambda\rho_N$,
and multiplication by $\sigma_1^2$ gives
\eqref{eq:streaming-cov-bound}.
For the posterior mean,
\[
\bm m_N-\check{\bm m}_N
=
\lambda(R_N-\check R_N)\bm b_N
=
-\lambda R_NE_N\check{\bm m}_N,
\]
which gives \eqref{eq:streaming-mean-bound}. 
\end{proof}

\begin{corollary}[Predictive and EI stability under streaming compression]
\label{cor:streaming-ei}
For fixed $\bm\theta$, let
$\mu_N(\bm s)$ and $s_N(\bm s)$ denote the exact predictive mean and
standard deviation, and let
$\check\mu_N(\bm s)$ and $\check s_N(\bm s)$ denote their compressed
counterparts. Then, uniformly over $\bm s\in\Sph$,
\begin{align}
|\mu_N(\bm s)-\check\mu_N(\bm s)|
&\leq
2\lambda\rho_N\|\check{\bm m}_N\|,
\label{eq:streaming-pred-mean}\\
|s_N^2(\bm s)-\check s_N^2(\bm s)|
&\leq
4\sigma_1^2\lambda\rho_N,
\label{eq:streaming-pred-var}\\
|s_N(\bm s)-\check s_N(\bm s)|
&\leq
2\sigma_1\sqrt{\lambda\rho_N}.
\label{eq:streaming-pred-sd}
\end{align}
Let
\[
\tau_N:=\max_{1\leq i\leq N}\mu_N(\bm s_i),
\qquad
\check\tau_N:=\max_{1\leq i\leq N}\check\mu_N(\bm s_i),
\]
and let $\operatorname{EI}_N$ and
$\check{\operatorname{EI}}_N$ denote Gaussian EI computed from the exact
and compressed predictive distributions and their corresponding thresholds.
Then
\begin{equation}
\sup_{\bm s\in\Sph}
\left|
\operatorname{EI}_N(\bm s)
-
\check{\operatorname{EI}}_N(\bm s)
\right|
\leq
\varepsilon_N(\bm\theta),
\label{eq:streaming-ei-uniform}
\end{equation}
where
\begin{equation*}
\varepsilon_N(\bm\theta)
:=
4\lambda\rho_N\|\check{\bm m}_N\|
+
\frac{2\sigma_1}{\sqrt{2\pi}}
\sqrt{\lambda\rho_N}.
\end{equation*}
If
\[
\bm s_N^\star\in
\argmax_{\bm s\in\Sph}\operatorname{EI}_N(\bm s),
\qquad
\check{\bm s}_N^\star
\in
\argmax_{\bm s\in\Sph}\check{\operatorname{EI}}_N(\bm s),
\]
then
\begin{equation}
\operatorname{EI}_N(\bm s_N^\star)
-
\operatorname{EI}_N(\check{\bm s}_N^\star)
\leq
2\varepsilon_N(\bm\theta).
\label{eq:streaming-ei-optimisation}
\end{equation}
\end{corollary}

\begin{proof}
The invariants in \eqref{eq:invariants} ensure that
$\overline y_N$ and $\bm\mu_N$ are preserved exactly
under compression. Hence the exact and compressed
predictive means are
\[
\mu_N(\bm s)
=
\overline y_N
+
(\bm s-\bm\mu_N)^\top\bm m_N,
\qquad
\check\mu_N(\bm s)
=
\overline y_N
+
(\bm s-\bm\mu_N)^\top\check{\bm m}_N.
\]
Similarly,
\[
s_N^2(\bm s)
=
\frac{\sigma_\varepsilon^2}{N}
+
(\bm s-\bm\mu_N)^\top
\Sigma_N
(\bm s-\bm\mu_N),
\]
with the compressed expression obtained by replacing
$\Sigma_N$ by $\check{\Sigma}_N$.
Since $\bm s$ and all evaluated designs have unit norm,
$\|\bm\mu_N\|\leq1$ and therefore
$\|\bm s-\bm\mu_N\|\leq2$.
Equations~\eqref{eq:streaming-pred-mean} and
\eqref{eq:streaming-pred-var} follow from
Proposition~\ref{prop:streaming-posterior}.
For nonnegative $a,b$, $|\sqrt a-\sqrt b| \leq \sqrt{|a-b|}$,
which gives \eqref{eq:streaming-pred-sd}.
The maximum of two functions differs by at most their uniform difference, so $
|\tau_N-\check\tau_N| \leq 2\lambda\rho_N\|\check{\bm m}_N\|$.
Gaussian EI is $1$-Lipschitz in the predictive mean,
$1$-Lipschitz in the improvement threshold, and
$1/\sqrt{2\pi}$-Lipschitz in the predictive standard deviation.
Combining these properties with
\eqref{eq:streaming-pred-mean}--\eqref{eq:streaming-pred-sd}
gives \eqref{eq:streaming-ei-uniform}.
Finally, uniformly perturbing an objective by at most
$\varepsilon_N(\bm\theta)$ changes the value attained by its maximiser under
the original objective by at most $2\varepsilon_N(\bm\theta)$, giving
\eqref{eq:streaming-ei-optimisation}.
\end{proof}

\paragraph{Hyperparameter-marginalised EI with fixed weights.}
The preceding corollary treats one fixed variance-hyperparameter value.
The acquisition used by SABLE instead averages conditional EI over
hyperparameter atoms and defines its incumbent from the corresponding
marginal posterior mean. The same streaming argument extends to this setting
provided the exact and compressed calculations use the same hyperparameter
distribution.

Let $\pi$ be any common distribution over
$\bm\theta=(\sigma_1^2,\sigma_\varepsilon^2)$, including the finite
equal-weight atom distribution used by the algorithm, and define
\[
\overline\mu_N(\bm s)
:=
\E_{\bm\theta\sim\pi}
\bigl[\mu_N(\bm s;\bm\theta)\bigr],
\qquad
\check{\overline\mu}_N(\bm s)
:=
\E_{\bm\theta\sim\pi}
\bigl[\check\mu_N(\bm s;\bm\theta)\bigr].
\]
Let
\[
\overline\tau_N
:=
\max_{1\leq i\leq N}
\overline\mu_N(\bm s_i),
\qquad
\check{\overline\tau}_N
:=
\max_{1\leq i\leq N}
\check{\overline\mu}_N(\bm s_i),
\]
and define the corresponding marginal EI acquisitions by
\[
\overline{\operatorname{EI}}_N(\bm s)
:=
\E_{\bm\theta\sim\pi}
\left[
\operatorname{EI}
\bigl(
\mu_N(\bm s;\bm\theta),
s_N(\bm s;\bm\theta);
\overline\tau_N
\bigr)
\right],
\]
with $\check{\overline{\operatorname{EI}}}_N$ defined analogously.

\begin{corollary}[Marginal EI stability for a common hyperparameter distribution]
\label{cor:streaming-marginal-ei}
For any common distribution $\pi$ over the variance hyperparameters,
\begin{equation}
\sup_{\bm s\in\Sph}
\left|
\overline{\operatorname{EI}}_N(\bm s)
-
\check{\overline{\operatorname{EI}}}_N(\bm s)
\right|
\leq
\overline\varepsilon_N(\pi),
\label{eq:streaming-marginal-ei-bound}
\end{equation}
where
\begin{equation*}
\overline\varepsilon_N(\pi)
:=
4\rho_N\,
\E_{\bm\theta\sim\pi}
\left[
\lambda(\bm\theta)
\|\check{\bm m}_N(\bm\theta)\|
\right]
+
\frac{2\sqrt{\rho_N}}{\sqrt{2\pi}}\,
\E_{\bm\theta\sim\pi}
\left[
\sigma_1(\bm\theta)\sqrt{\lambda(\bm\theta)}
\right].
\end{equation*}
\end{corollary}

\begin{proof}
\Eqref{eq:streaming-pred-mean} gives
\[
\sup_{\bm s\in\Sph}
\left|
\overline\mu_N(\bm s)
-
\check{\overline\mu}_N(\bm s)
\right|
\leq
2\rho_N\,
\E_{\bm\theta\sim\pi}
\left[
\lambda(\bm\theta)
\|\check{\bm m}_N(\bm\theta)\|
\right].
\]
Hence the same bound holds for
$|\overline\tau_N-\check{\overline\tau}_N|$.
Applying the Lipschitz properties of Gaussian EI inside the expectation,
then using \eqref{eq:streaming-pred-sd}, yields
\eqref{eq:streaming-marginal-ei-bound}.
\end{proof}

\subsubsection{Stability with refitted hyperparameter weights}
\label{app:refitted-hyperparameter-weights}

The hyperposterior of Appendix~\ref{app:atoms} is evaluated on a fixed
finite grid before its $K$-atom acquisition approximation. 
Let $\pi$ and $\check\pi$ denote the normalised
hyperparameter weights obtained from the exact and
compressed marginal likelihoods for the same $N$
observations, using identical positive prior masses
on a common finite grid
$\Theta=\{\bm\theta_1,\ldots,\bm\theta_J\}$.
All grid points have strictly positive variance
parameters.
For any probability vector $p$ on $\Theta$, define
the marginal posterior-mean incumbents
\[
\overline\tau_N^p
:=
\max_{1\le i\le N}
\sum_j p_j\mu_N(\bm s_i;\bm\theta_j),
\qquad
\check{\overline\tau}_N^p
:=
\max_{1\le i\le N}
\sum_j p_j\check\mu_N(\bm s_i;\bm\theta_j).
\]
The corresponding marginal EI acquisitions are
\begin{align*}
\overline{\operatorname{EI}}_N^p(\bm s)
:=
\sum_j p_j
\operatorname{EI}\bigl(
\mu_N(\bm s;\bm\theta_j),
s_N(\bm s;\bm\theta_j);
\overline\tau_N^p
\bigr),\\
\check{\overline{\operatorname{EI}}}_N^p(\bm s)
:=
\sum_j p_j
\operatorname{EI}\bigl(
\check\mu_N(\bm s;\bm\theta_j),
\check s_N(\bm s;\bm\theta_j);
\check{\overline\tau}_N^p
\bigr).
\end{align*}
Both thresholds use the same $N$ historically
evaluated queries.
For the compressed posterior, set
\begin{align*}
B_{\mu,N}
&:=\max_{1\le i\le N}\!\left(
\max_j\check\mu_N(\bm s_i;\bm\theta_j)
-\min_j\check\mu_N(\bm s_i;\bm\theta_j)\right),
\\
B_{\mathrm{EI},N}^{\pi}
&:=\sup_{\bm s\in\Sph}\!\left(
\max_j e_{N,j}^{\pi}(\bm s)
-\min_j e_{N,j}^{\pi}(\bm s)\right),
\end{align*}
where
$e_{N,j}^{\pi}(\bm s):=
\operatorname{EI}(
\check\mu_N(\bm s;\bm\theta_j),
\check s_N(\bm s;\bm\theta_j);
\check{\overline\tau}_N^{\pi})$.
\begin{proposition}[Marginal-EI stability under compressed refitting]
\label{prop:refitted-grid-ei}
Let $\ell_j$ and $\check\ell_j$ denote the exact
and compressed log marginal likelihoods used
to construct $\pi$ and $\check\pi$, respectively.
Define
\[
\delta_{\ell,N}
:=
\max_j(\check\ell_j-\ell_j)
-
\min_j(\check\ell_j-\ell_j).
\]
Write
\(
d_{\mathrm{TV}}(\pi,\check\pi)
:=
\frac12\sum_{j=1}^J
|\pi_j-\check\pi_j|.
\)
Then
\begin{equation}
d_{\mathrm{TV}}(\pi,\check\pi)
\leq
\tanh\left(\frac{\delta_{\ell,N}}4\right).
\label{eq:refitted-grid-tv}
\end{equation}
Under the pseudo-observation
invariants~\eqref{eq:invariants}, the corresponding
marginal EI acquisitions satisfy
\begin{equation}
\sup_{\bm s\in\Sph}
\left|
\overline{\operatorname{EI}}_N^\pi(\bm s)
-
\check{\overline{\operatorname{EI}}}_N^{\check\pi}(\bm s)
\right|\leq
\overline\varepsilon_N(\pi)
+
\tanh\left(\frac{\delta_{\ell,N}}4\right)
\left(B_{\mu,N}+B_{\mathrm{EI},N}^{\pi}\right),
\label{eq:refitted-grid-ei-bound}
\end{equation}
where $\overline\varepsilon_N(\pi)$ is the
fixed-weight compression bound of
Corollary~\ref{cor:streaming-marginal-ei}.
\end{proposition}
\begin{proof}
Let $r_j=\exp(\check\ell_j-\ell_j)$,
$a=\min_j r_j$, $b=\max_j r_j$, and
$m=\sum_j\pi_jr_j$. Normalisation gives
$\check\pi_j=\pi_jr_j/m$.
If $a=b$, the weights coincide. Otherwise,
convexity of $r\mapsto|r-m|$ on $[a,b]$ gives
\begin{align*}
d_{\mathrm{TV}}(\pi,\check\pi)
=\frac{1}{2m}\sum_j\pi_j|r_j-m|
\leq\frac{(m-a)(b-m)}{m(b-a)}\leq\frac{\sqrt b-\sqrt a}{\sqrt b+\sqrt a}
=\tanh\left(\frac{\delta_{\ell,N}}4\right).
\end{align*}
For the acquisition bound, insert the
compressed marginal EI with exact weights:
\begin{align*}
\left|
\overline{\operatorname{EI}}_N^\pi(\bm s)
-
\check{\overline{\operatorname{EI}}}_N^{\check\pi}(\bm s)
\right|\leq
\left|
\overline{\operatorname{EI}}_N^\pi(\bm s)
-
\check{\overline{\operatorname{EI}}}_N^\pi(\bm s)
\right| +
\left|
\check{\overline{\operatorname{EI}}}_N^\pi(\bm s)
-
\check{\overline{\operatorname{EI}}}_N^{\check\pi}(\bm s)
\right|.
\end{align*}
Corollary~\ref{cor:streaming-marginal-ei}
bounds the first term uniformly by
$\overline\varepsilon_N(\pi)$,
since both acquisitions use the same
hyperparameter weights $\pi$.
For probability vectors $p,q$ and values $h_j$,
\[
\left|\sum_j(p_j-q_j)h_j\right|
\leq d_{\mathrm{TV}}(p,q)
(\max_jh_j-\min_jh_j).
\]
Applying this inequality to the historical
posterior means, together with the
$1$-Lipschitz property of the maximum, gives
\[
\left|
\check{\overline\tau}_N^\pi
-
\check{\overline\tau}_N^{\check\pi}
\right|
\leq
d_{\mathrm{TV}}(\pi,\check\pi)B_{\mu,N}.
\]
Changing the averaging weights at the common
threshold $\check{\overline\tau}_N^\pi$
contributes at most
$d_{\mathrm{TV}}(\pi,\check\pi)
B_{\mathrm{EI},N}^{\pi}$.
Since Gaussian EI is $1$-Lipschitz in its
threshold, the combined weight perturbation
is bounded by
\[
d_{\mathrm{TV}}(\pi,\check\pi)
\left(B_{\mu,N}+B_{\mathrm{EI},N}^{\pi}\right).
\]
Combining this with
\eqref{eq:refitted-grid-tv} proves
\eqref{eq:refitted-grid-ei-bound}.
\end{proof}
\paragraph{Implications for acquisition decisions.}
Define
\[
\varepsilon_N^{\mathrm{grid}}
:=
\overline\varepsilon_N(\pi)
+
\tanh\left(\frac{\delta_{\ell,N}}4\right)
(B_{\mu,N}+B_{\mathrm{EI},N}^{\pi}).
\]
If $\bm s^\star$ and $\check{\bm s}^\star$
maximise the exact and compressed grid-marginal
EI, respectively, over a common feasible domain
$\mathcal F\subseteq\Sph$, then
\[
\overline{\operatorname{EI}}_N^\pi(\bm s^\star)
-
\overline{\operatorname{EI}}_N^\pi(\check{\bm s}^\star)
\leq 2\varepsilon_N^{\mathrm{grid}}.
\]
On a common finite candidate set, the selected
candidate is unchanged if the exact acquisition
has a unique maximiser whose advantage over
every alternative exceeds
$2\varepsilon_N^{\mathrm{grid}}$.
\paragraph{Scope of the guarantees.}
The bounds separate posterior compression error
at fixed hyperparameter weights from the additional
effect of refitting those weights using approximate
evidence. They depend on the realised Gram error
$\rho_N$ and log-evidence error range
$\delta_{\ell,N}$ and do not establish that these
errors remain small under recompression.
The $K$-atom approximation and changes in the
candidate set arising from compression or
hyperparameter refitting are not covered.
The acquisition-decision guarantee applies to
a common feasible domain or candidate set,
but does not guarantee that the implemented
candidate search finds a global maximiser of
marginal EI over $\Sph$.

\subsection{Preserving fresh off-span directions}
\label{app:frame-ledger}

Because recompression can discard previously explored
directions, orthogonalising against the compressed state
alone could revisit them. SABLE instead draws initial
and off-span exploration directions, including
degenerate-acquisition fallbacks, from an orthonormal
frame of $D$ vectors.
Since every query is constructed from previously consumed
frame vectors, the next unused vector is orthogonal to
the complete query history, including directions
discarded from the compressed fitting state.
For bounded-state runs, the frame is generated on demand
from an implicitly represented seeded orthogonal transform,
replacing independent uniform draws with a fixed pseudorandom
orthogonal sequence. It requires $\mathcal O(D)$ additional
state; see Appendix~\ref{app:orthogonal-frame}.

For historical posterior-mean readouts, we maintain
$Q=X_N\check X^\top\in\R^{N\times m}$, giving
$X_N\bar{\bm m}=Q\bar{\bm c}$ when
$\bar{\bm m}=\check X^\top\bar{\bm c}$.
The live row count satisfies $m=\mathcal O(M)$.

The ledger is updated without revisiting past queries.
When a query $\bm s$ is appended, let
$\bm b=\check X\bm s$ and
$\Lambda=\check X\check X^\top$, the diagonal Gram matrix.
Since $\bm s$ lies in the compressed fitting span plus
a fresh frame direction orthogonal to previous queries,
$X_N\bm s=Q\Lambda^\dagger\bm b$,
where $\Lambda^\dagger$ denotes the pseudoinverse.
If $W$ diagonalises the bordered Gram matrix,
the ledger is updated alongside the fitting state:
\[
Q\longleftarrow
\begin{bmatrix}
Q & Q\Lambda^\dagger\bm b\\
\bm b^\top & 1
\end{bmatrix}W.
\]
Recompression applies a row transformation
$\check X\leftarrow T\check X$ and the corresponding
ledger update $Q\leftarrow QT^\top$.

Ledger maintenance costs $\mathcal O(NM^2)$ per round,
while computing the EI incumbent costs $\mathcal O(NM)$.
Acquisition candidate generation and scoring operate
in $\mathcal O(M)$-dimensional coordinates for fixed
hyperparameter-grid and candidate-set sizes; basis
construction and query reconstruction cost
$\mathcal O(MD)$.
For $N\leq D$, these costs are absorbed by
$\mathcal O(M^2D+M^3)$ time, and the ledger's
$\mathcal O(NM)$ storage by $\mathcal O(MD)$ memory.

\section{Orthogonal frames for off-span directions}
\label{app:orthogonal-frame}

The reconstruction in \eqref{eq:reconstruct} requires a unit
direction $\bm e\in\mathcal V_N^\perp$.
Since affine EI is invariant to the off-span orientation
(Appendix~\ref{app:post-data-rotation}), we can select
these directions from an orthogonal frame without
affecting the acquisition value.
By selecting successive unused columns, orthogonality
to all previous queries is guaranteed by construction.

\paragraph{Uncompressed posterior.}
Without compression, we use successive columns of a Haar
orthogonal frame, equivalently generated by sampling
Gaussian directions and orthogonalising against
previously used directions.
Retaining the explored span does not change the
$\mathcal O(ND)$ storage scaling of the uncompressed method.

\paragraph{Bounded-state posterior.}
To avoid storing the exploration history under compression,
we construct a seeded randomised DCT-II frame,
\[
\Psi=\Xi_L C\Xi_{L-1} C\cdots\Xi_1 C\Xi_0\Pi,
\]
where $C$ is the orthonormal DCT-II matrix, the $\Xi_i$
are independent random sign diagonals, and $\Pi$ is a
random permutation matrix.

Since $\Psi$ is orthogonal, its unused columns remain
orthogonal to all previous queries, even after
recompression. Although the frame is not Haar distributed,
its columns $\bm u_j$ have isotropic second moments:
\[
\E[\bm u_j]=0,
\qquad
\E[\bm u_j\bm u_j^\top]=D^{-1}I_D.
\]
Each column can be generated on demand using $L$
fast cosine transforms in $\mathcal O(LD\log D)$ time,
with $\mathcal O(D)$ stored state independent of $N$.
We use $L=3$ throughout.
The implementation requires only a fixed seed and a
counter identifying the next unused column.

\section{Benchmark definitions}\label{sec:benchmarks}

\subsection{ZonalFunctionBench: controlled objectives on the sphere}
\label{app:zonal-bench}

ZonalFunctionBench is a family of rotationally invariant benchmark
distributions on $S^{D-1}$ designed to control the amount of structure lying
outside the affine surrogate class. The construction is based on random
Gegenbauer features associated with the zonal kernel
\[
k(x,y)
=
\sum_{q\geq0}\sigma_q^2 P_q(x^\top y),
\]
where $P_q$ is the degree-$q$ Gegenbauer polynomial normalised so that
$P_q(1)=1$, and $\sigma_q^2$ controls the contribution of
spherical-harmonic degree $q$. Degrees $0$ and $1$ correspond respectively
to a constant and a linear form and therefore lie within the affine surrogate
class, whereas degrees $q\geq2$ introduce non-affine structure. Because the
feature directions are sampled uniformly on the sphere, the benchmark
distribution is rotationally invariant and introduces no preferred coordinate
system or axis-aligned structure.

We use four benchmark settings. The affine control contains only a degree-$1$
component. The remaining three settings are generated from degrees $0$--$3$, with
progressively more spectral power shifted from the degree-$1$ component
towards higher-order structure. For these three settings, the degree-$1$--$3$
weights $(\sigma_1^2,\sigma_2^2,\sigma_3^2)$ are
\[
(1,\tfrac12,\tfrac14)
\quad\text{(mild)},\qquad
(1,1,1)
\quad\text{(moderate)},\qquad
(\tfrac14,\tfrac12,1)
\quad\text{(severe)}.
\]
Thus the severe setting assigns only $1/7$ of its degree-$1$--$3$ spectral
power to the linear component. 
When noisy observations
are used, independent Gaussian noise with standard deviation
$\sigma_{\mathrm{obs}}$ is added to each evaluation.

\paragraph{Construction.}
For the mild, moderate, and severe settings, the degree-$1$--$3$ component is
drawn as
\[
f_{\mathrm{raw}}(x)
=
\sum_{q=1}^{3}\sum_{r=1}^{n_q}
\sqrt{\frac{\sigma_q^2 d_q}{n_q}}\,
\xi_{q,r}\,
P_q(x^\top w_{q,r}),
\qquad
w_{q,r}\sim\mathrm{U}(S^{D-1}),
\quad
\xi_{q,r}\sim\mathcal N(0,1),
\]
where $d_q$ is the dimension of the degree-$q$ spherical-harmonic space.
We use $n_{\mathrm{feat}}=\max(1500,D)$ random features in total,
allocated across degrees in proportion to $\sigma_q^2\sqrt{1-1/d_q}$, with
integer counts $n_1+n_2+n_3=n_{\mathrm{feat}}$ chosen by the
largest-remainder rule.

Since $\int_{S^{D-1}}P_q(x^\top w)^2\,d\nu_D(x)=1/d_q$, this scaling gives
$f_{\mathrm{raw}}$ sphere-averaged variance $\sum_q\sigma_q^2$. Its marginal
variance conditional on the orientations,
\[
s_{\mathrm{raw}}^2(x)
:=
\sum_{q=1}^{3}\sum_{r=1}^{n_q}
\frac{\sigma_q^2 d_q}{n_q}\,P_q(x^\top w_{q,r})^2,
\]
is not constant, however: with $n_q\ll d_q$ it rises to
$\sigma_q^2 d_q/n_q$ at each feature orientation, so the raw draw is
dominated by needle-like peaks there. We therefore normalise pointwise,
\[
\tilde f(x)
=
\frac{f_{\mathrm{raw}}(x)}
{s_{\mathrm{raw}}(x)\sqrt{\min(n_{\mathrm{feat}},D)}}.
\]
The first factor gives unit marginal variance at every $x$, and hence a
common scale across the three settings; the second keeps attainable maxima
on an $O(1)$ scale. At random points $s_{\mathrm{raw}}(x)$ is within a few
per cent of its sphere average, so the degree-one component of $\tilde f$
is aligned to well under one degree with that of $f_{\mathrm{raw}}$,
$\sum_{r=1}^{n_1}\sqrt{\sigma_1^2 d_1/n_1}\,\xi_{1,r}\,w_{1,r}$, which is
the direction $\bm v_\star$ used for the recovery diagnostic.

An independent degree-$0$ component is then added,
\[
c\sim\mathcal N(0,\sigma_0^2),
\qquad
\sigma_0
=
\sqrt{\frac{0.2}{\sum_{q=1}^{3}\sigma_q^2}}
\Big/\sqrt{\min(n_{\mathrm{feat}},D)},
\]
so that the constant enters with weight $0.2$ on the scale of the weights
$(\sigma_1^2,\sigma_2^2,\sigma_3^2)$. The resulting objective is
$f(x)=c+\tilde f(x)$; $c$ is drawn from a separate random-number generator,
so adding it does not alter the degree-$1$--$3$ landscape.

\subsection{Generative-model optimisation problems}
\label{app:benchmark-definitions}

We evaluate on four latent optimisation problems constructed from pretrained
scientific generative models. In each case, a pretrained generative model
$g:\R^D\to\mathcal X$ is paired with a black-box objective
$f:\mathcal X\to\R$, giving the spherical latent objective
\(
F(\bm s)=f(g(\sqrt{D}\bm s))
\)
of Section~\ref{sec:model}. The generative models are used without
retraining or fine-tuning for optimisation, and all objectives are maximised.
For each diffusion model, the latent variable is used as the terminal noise
and the reverse sampling procedure is fixed and deterministic, using either
DDIM with $\eta=0$ \citep{song2021ddim} or DPM-Solver++
\citep{lu2022dpmsolverpp}. Hence $g$, and therefore $F$, is deterministic.

The four problems span three scientific domains and latent dimensions from
$6{,}000$ to $40{,}500$. Two of the four form a matched calorimeter pair:
datasets~2 and~3 use the same generative-model family, physical target and
objective, but different detector discretisations. Their latent dimensions are
$6{,}480$ and $40{,}500$, respectively, providing a controlled comparison
across a $6.25$-fold increase in the number of latent voxel coordinates.

\begin{table}[ht]
\centering
\footnotesize
\begin{tabular}{l r l l l r}
\hline
\thead{Problem}
& \thead{$D$}
& \thead{Generative\\model}
& \thead{Generated\\object}
& \thead{Objective}
& \thead{Prior\\median} \\
\hline
Turbulence
& $6{,}000$
& diffusion-lagr
& velocity trajectory
& DNS distribution match
& $0.57$ \\
Peptides
& $13{,}440$
& AMP-Diffusion
& peptide sequence
& hydrophobic moment
& $0.34$ \\
Calorimeter DS2
& $6{,}480$
& CaloDiffusion
& shower voxel grid
& energy-deposition match
& $0.07$ \\
Calorimeter DS3
& $40{,}500$
& CaloDiffusion
& shower voxel grid
& energy-deposition match
& $0.05$ \\
\hline
\end{tabular}
\caption{
Generative-model optimisation problems.
``Prior median'' is the median objective value obtained from independent
draws of the generative model's latent prior
$\mathcal N(\bm0,I_D)$.
Datasets~2 and~3 form a matched calorimeter pair at different detector
discretisations.
}
\label{tab:bench-summary}
\end{table}

\paragraph{Lagrangian turbulence.}
We use diffusion-lagr \citep{li2024synthetic}, a one-dimensional diffusion
model for single-particle velocity trajectories in three-dimensional
homogeneous isotropic turbulence, at the published DM-3c
exponential-moving-average checkpoint. Sampling uses DPM-Solver++ with
$200$ steps. The latent has dimension
\(
D=2{,}000\times3=6{,}000
\)
and is reshaped into the three velocity components of a $2{,}000$-step
trajectory.

The objective measures agreement with the intermittent velocity-increment
statistics of direct numerical simulation. For lag $\tau$, let
$\Delta_\tau u(t)=u(t+\tau)-u(t)$ and let
$P^\tau_{\mathrm{gen}}$ and $P^\tau_{\mathrm{ref}}$ denote the corresponding
generated and reference histograms. Over
$L=\{1,5,10,25,50\}$, we maximise
\[
F
=
\max\!\left\{
-1,\,
1-
\frac{1}{|L|}
\sum_{\tau\in L}
\frac{
\operatorname{JSD}
\!\left(
P^\tau_{\mathrm{gen}}
\,\|\,P^\tau_{\mathrm{ref}}
\right)
}{
J_{\mathrm{gauss}}
}
\right\},
\]
where $J_{\mathrm{gauss}}=0.0391$ is the average divergence
for a Gaussian with matched mean and variance. Thus $F=1$ corresponds to
matching the reference increment distributions, whereas $F=0$ corresponds
to the Gaussian no-intermittency reference. Values below zero are truncated
at $-1$.

The reference distributions are constructed from $256$ ground-truth
trajectories using $101$ fixed histogram bins with $25\%$ padding beyond the
reference increment range; out-of-range values are accumulated in the end
bins. The resulting objective is non-differentiable. Prior draws give values
from approximately $0.26$ to $0.82$, with median $0.57$.

\paragraph{Antimicrobial peptides.}
We use the publicly released AMP-Diffusion checkpoint
\citep{chen2024ampdiffusion,torres2025generative}, a latent diffusion model
over per-residue protein-language-model embeddings built around ESM-2
\citep{lin2023esm2}. The checkpoint uses the $8$M-parameter ESM-2 variant
with $320$-dimensional embeddings. Sampling uses DDIM with $\eta=0$ for
$100$ steps.

The latent contains $42$ token positions,
\(
D=42\times320=13{,}440.
\)
After denoising, the embeddings are mapped through the language-model output
head and decoded by taking the arg-max over the $20$ standard amino acids at
each residue position. This discrete decoding makes the latent objective
piecewise constant and non-differentiable.

We maximise the Eisenberg mean hydrophobic moment for an $\alpha$-helix
\citep{eisenberg1984moment}, a standard measure of amphipathicity relevant
to antimicrobial peptides \citep{yeaman2003mechanisms}. For a sequence
$a_1,\ldots,a_n$,
\[
\mu_H
=
\frac{1}{n}
\left[
\left(
\sum_{k=1}^{n}H(a_k)\sin(k\delta)
\right)^2
+
\left(
\sum_{k=1}^{n}H(a_k)\cos(k\delta)
\right)^2
\right]^{1/2},
\qquad
\delta=100^\circ,
\]
where $H(\cdot)$ is the consensus hydrophobicity scale. The objective is
computed directly from the decoded sequence and requires no auxiliary
predictor. Its median under prior draws is $0.34$.

\paragraph{Calorimeter showers.}
The final two problems use CaloDiffusion
\citep{amram2023calodiffusion} on datasets~2 and~3 of the Fast Calorimeter
Simulation Challenge
\citep{calochallenge_dataset2,calochallenge_dataset3}. Both generate
electromagnetic calorimeter showers at a fixed incident energy of
$100$~GeV. The latent variable represents shower-to-shower variation, and
sampling uses DPM-Solver++ with $100$ steps.

For dataset~2, the detector contains $45$ longitudinal layers, $16$
azimuthal bins and $9$ radial bins, giving
\(
D_{\mathrm{DS2}}
=
45\times16\times9
=
6{,}480.
\)
For dataset~3, the same physical problem is represented using $45$
longitudinal layers, $50$ azimuthal bins and $18$ radial bins, giving
\(
D_{\mathrm{DS3}}
=
45\times50\times18
=
40{,}500.
\)
Dataset~3 therefore contains $6.25$ times as many latent voxel coordinates
as dataset~2.

The objective is evaluated on physical energy deposits. For dataset~2, a
normalised model output $v$ is mapped back to physical energy through
\[
E
=
\frac{
\sigma(1.9123\,v-12.8564)-\varepsilon_0
}{
1-2\varepsilon_0
}
\,2E_{\mathrm{inc}},
\qquad
\varepsilon_0=10^{-6},
\]
where $\sigma$ is the logistic function; negative sub-threshold values are
clamped to zero. Dataset~3 uses the corresponding checkpoint-specific
transformation with
$\sigma(1.1070\,v-13.4753)$ in place of
$\sigma(1.9123\,v-12.8564)$. Its network also uses a preconditioned
denoising head rather than the noise-prediction head of dataset~2; both
decoders are used according to their released implementations.

Let $p$ denote the resulting per-voxel energy distribution, normalised to
sum to one, and let $q$ denote a fixed target shower pattern. For both
datasets we maximise the cosine similarity
\(
F
=
\frac{\langle p,q\rangle}
     {\lVert p\rVert\,\lVert q\rVert},
\)
where
\[
q_{\ell\alpha r}
\propto
\exp\!\left(
-\frac{(\ell-\ell_0)^2}{2\sigma_\ell^2}
\right)
\exp\!\left(
-\frac{(r-r_0)^2}{2\sigma_r^2}
\right),
\]
with $\ell_0=25$, $r_0=11.625$, $\sigma_\ell=5$ and $\sigma_r=8$, and
uniform dependence on the azimuthal coordinate $\alpha$.

The score lies in $[0,1]$. Dataset~2 has prior median $0.071$ and
interquartile range $0.052$--$0.100$, while prior draws for dataset~3 are
concentrated around $0.05$.

\section{Experimental and implementation details}
\label{app:experimental-setup}

This appendix collects the numerical settings and implementation choices
used in the experiments, together with the complete
per-round fitting and acquisition algorithms. 

\subsection{Numerical settings}
\label{app:defaults}

\paragraph{Hyperparameter prior and discretisation.}
We place the amplitude prior on the global affine-variation scale
$\widehat\sigma_1:=\sqrt D\,\sigma_1$, with
$\widehat\sigma_1
\sim
\operatorname{LogNormal}(0,0.35^2)$.
Thus the corresponding prior on the per-direction scale has median
$\bar\sigma_1=D^{-1/2}$ and log-width $h_1=0.35$. The discrepancy prior is
\(
\sigma_\varepsilon
\sim
\operatorname{LogNormal}(\log 0.03,0.55^2),
\)
in output units and independently of $D$.

The hyperposterior grid of Appendix~\ref{app:atoms} uses $n_g=40$ points
per log-scale axis and half-width $\eta=3.5$ prior standard deviations.
For $\sigma_\varepsilon$, the upper grid boundary is extended when necessary
to include $\log\sigma_\varepsilon=0$; this enlarges the numerical support
but does not modify the prior density. The resulting discrete
hyperposterior is summarised by $K=8$ equal-weight atoms using the
deterministic stratified resampling of Appendix~\ref{app:atoms}.

\paragraph{Sequential acquisition.}
For each hyperparameter atom we use $12$ values of the stationary-family
parameter $\nu$ in Appendix~\ref{app:acquisition-machinery}, together with
the posterior-mean direction whenever it is nonzero. Under the prescribed
off-span schedule, every candidate direction is evaluated at
\(
\|\bm w\|=\sqrt{1-t_N^2}
\)
and scored using the full hyperparameter-marginalised EI. For the
unconstrained-EI comparison, the radial candidate set combines $25$ points
uniform in $\alpha=\|\bm w\|$ with $25$ points uniform in
$t=\sqrt{1-\alpha^2}$. These candidate-grid sizes affect only the numerical
optimisation of the acquisition; all candidates are scored using the full
hyperparameter-marginalised EI.

\paragraph{Reference-resolution schedule.}
The VoI schedule of \eqref{eq:voi-schedule-main} depends only on
the evaluation horizon $N_{\max}$ and the reference-resolution parameter
$\lambda_{\mathrm{ref}}$ and is precomputed before the run, independently of
the fitted variance scales. For SABLE, we use the fixed representatives of
the local-resolvability regimes described in
Appendix~\ref{sec:local-resolvability}: $\lambda_{\mathrm{ref}}=10^{5}$ for
the smooth controlled objectives (controlled ladder, turbulence, and both calorimeter problems),
and $\lambda_{\mathrm{ref}}=10$ for peptides, whose objective is quantised
through discrete sequence decoding. 

\paragraph{Bounded-state approximation.}
Unless otherwise stated, recompression uses $\mu=2$ inverse moments and at
most six spectral bins. Leading Gram eigendirections are retained
preferentially, supplemented by $J=\lfloor 2M/3\rfloor$ rank-spaced
spectral locations, following Appendix~\ref{app:reservoir-truncation}.
To avoid excessively frequent recompression, the live
representation is allowed to exceed $M$ by
\[
B=\min\left\{
M,\,
\max\left(
\left\lfloor\frac{M}{5}\right\rfloor,
4\mu n_b
\right)
\right\},
\]
where $n_b$ is the number of active spectral bins.
Recompression is triggered when the live row count
exceeds $M+B$.
Setting $M=\infty$ disables recompression and
recovers the uncapped representation.

\subsection{Complete per-round algorithms}
\label{app:full-algorithms}

Algorithms \ref{alg:fit} and \ref{alg:round-ei} summarise fitting and acquisition operations. The observation count $N$ is used by all
VoI schedules, while pseudo-row counts are used in the compressed
representation.

\begin{algorithm}[t]
\caption{\textsc{FitRound}: posterior fit and hyperposterior atoms}
\label{alg:fit}
\begin{algorithmic}[1]
\Require current state
$(\check X,\check{\bm y},\check{\bm o},Q)$;
log-determinant correction $\Delta_{\det}(\cdot)$;
true count $N$;
new raw evaluations;
compression budget $M$;
hyperparameter grid;
atom count $K$
\For{each new raw evaluation $(\bm s,y)$}
    \State append $(\bm s,y)$ to the live observation representation
    \State update the diagonal Gram basis by the arrowhead eigendecomposition
    \State apply the corresponding rotation to the ledger $Q$
        (Appendix~\ref{app:frame-ledger})
    \State $N\gets N+1$
\EndFor
\If{$M<\infty$ and the live row count exceeds $M+B$}
    \State retain the prescribed leading and stratified Gram eigendirections
    \State form the spectral-bin augmentation rows while preserving
    \eqref{eq:invariants}
    \State update $\Delta_{\det}(\bm\theta)$ on the hyperparameter grid using
    \eqref{eq:logdet-acc}
    \State re-diagonalise the compressed state, with $\mathcal O(M)$ rows
    \State update the ledger $Q\gets QT^\top$ using the
        combined recompression and basis transformation
\EndIf
\State evaluate the marginal likelihood on the
$(\log\sigma_1,\log\sigma_\varepsilon)$ grid using the live spectrum and
$\Delta_{\det}(\bm\theta)$
\State obtain $K$ equal-weight hyperparameter atoms by deterministic
stratified resampling
\State compute the atom-specific working-basis quantities required by
Appendix~\ref{app:working-basis}
\State compute the model incumbent $\tau_N$ using the
    historical posterior-mean ledger $Q$
\State \Return fitted state, $Q$, $N$, $\tau_N$,
    acquisition working basis, and atom-specific quantities
\end{algorithmic}
\end{algorithm}

\begin{algorithm}[t]
\caption{Sequential round: marginal EI with prescribed off-span exploration}
\label{alg:round-ei}
\begin{algorithmic}[1]
\Require output of Algorithm~\ref{alg:fit};
off-span norm $t_N$ from \eqref{eq:voi-schedule-main};
stationary-family candidate grid
\If{$t_N=1$}
    \State \Return the next unused orthogonal-frame direction
\EndIf
\State $\alpha_N\gets\sqrt{1-t_N^2}$
\State construct the candidate directions of
Appendix~\ref{app:acquisition-machinery} for every hyperparameter atom
\State append each nonzero posterior-mean direction
\If{nonzero in-span candidates exist}
    \State evaluate candidates at $\|\bm w\|=\alpha_N$
    under marginal EI \eqref{eq:marginal-ei}
    \State let $\bm w^\star$ be the highest-scoring candidate
\Else
    \State $\bm w^\star\gets\alpha_N\bm u$
    for any unit vector $\bm u\in\mathbb R^{d_N}$
\EndIf
\If{$t_N>0$}
    \State obtain the next unused off-span direction $\bm e$
    \State \Return $\bm s_{N+1}=V\bm w^\star+t_N\bm e$
\Else
    \State \Return $\bm s_{N+1}=V\bm w^\star$
\EndIf
\end{algorithmic}
\end{algorithm}

\subsection{Baseline configurations}
\label{app:baselines}

All baselines receive the same evaluation budget as SABLE and, unless marked
ambient, are evaluated on the same $\sqrt D$-radius sphere: a candidate
proposed in $\R^D$ is rescaled to the sphere before decoding, so the
objective is a function of direction only.

\paragraph{Random search and orthogonal design.}
Random search draws queries uniformly on the sphere. The orthogonal design
uses the columns of a Haar-random orthogonal frame scaled to the sphere, so
every pair of queries is exactly orthogonal and the budget probes $N$
distinct directions with no redundancy; it requires $N\leq D$.

\paragraph{Design-then-fit.}
The random and orthogonal design-then-fit baselines of
Figure~\ref{fig:c4-rank-lambda} spend $N-1$ evaluations on the
non-adaptive design, fit the uncompressed affine model with marginalised
variance scales to the result, and spend the final evaluation on the
posterior-mean recommendation of Appendix~\ref{app:readout}. Their first
$N-1$ queries coincide with the plain designs at the same seed.

\paragraph{Thompson sampling.}
Thompson sampling uses the same bounded-state posterior as SABLE
($M=300$, marginalised variance scales): at each step a slope is drawn from
the compressed posterior in the data-informed directions and from the prior
in their complement, and the query is the sphere point maximising the sampled
affine function. It prescribes no off-span norm.

\paragraph{CMA-ES.}
Separable CMA-ES \citep{ros2008simple} maintains a diagonal covariance and
costs $\mathcal O(D)$ per generation; full-covariance CMA-ES
\citep{hansen1996adapting} is run where a dense $D\times D$ covariance fits
in memory (turbulence and Calorimeter DS2). Both use the standard population
size $4+\lfloor3\ln D\rfloor$ and initial step size $\sigma_0=1$ in the ambient
parameterisation, with the search mean initialised at the origin so that the
first generation is a draw from $\mathcal N(\bm 0,I_D)$ and lands on
the typical shell.
Members of a generation are evaluated together and yielded in a
seed-shuffled order so that best-so-far curves carry no within-generation
ordering bias. The sphere variants rescale each candidate to the sphere before
evaluation; the ambient variants ($\R^D$) evaluate the raw candidate and are
free to leave the typical set.

\paragraph{Vanilla BO, Linear BO and TuRBO.}
Vanilla BO \citep{hvarfner2024vanilla}, high-dimensional linear BO
\citep{doumont2026}, and TuRBO~\citep{eriksson2019scalable} follow their
reference implementations in BoTorch~\citep{balandat2020botorch}. Vanilla BO uses a single-task GP with
an ARD RBF kernel, the dimension-scaled lengthscale prior of the reference,
standardised targets, and LogEI optimised by multi-start L-BFGS-B. Linear BO
uses the reference's linear kernel on inverse-stereographically projected
inputs with ARD lengthscales and LogEI. TuRBO uses a trust region with a
Mat\'ern-$5/2$ ARD kernel and Thompson sampling over $5000$ Sobol
candidates, each perturbing every coordinate independently with probability
$\min(20/D,1)$ ($20$ coordinates in expectation), with the reference hyperparameters.
Each opens with a $30$-point scrambled Sobol block. All three
search the unit cube and decode candidates on the sphere. To make $N=3000$
feasible at $D\geq6000$, hyperparameters are refit by L-BFGS-B whenever the
data has grown by $5\%$ since the last fit (about $94$ fits over a run),
warm-started from the previous values and capped at $200$ iterations; between
refits the model is conditioned exactly on all data under the last fitted
hyperparameters. Vanilla BO and linear BO nevertheless cost minutes per
decision by $N\approx1000$, so their runs are read at $48$\,h of wall clock
(Figure~\ref{fig:c5-domains}). These three baselines are omitted on Calorimeter
DS3.

\paragraph{MSR and BAxUS.}
The two remaining GP baselines are run as their reference implementations prescribe: neither the 30-point Sobol opening block nor the refit-on-growth cadence of the paragraph above is imposed on them, and each keeps its own initial design and refits at every step.
The only adaptation is the objective interface, which decodes each proposal on the sphere by the same map as the other GP
baselines while the optimiser keeps its own coordinates and the value it
received. MSR \citep{papenmeier2025understanding} opens with
$\lfloor 3\sqrt D\rfloor$ scrambled Sobol points and at every step builds a
fresh single-task GP with an ARD Mat\'ern-$5/2$ kernel, no lengthscale prior,
lengthscales initialised at $\sqrt D/10$, the reference's outputscale and noise
priors and a learned noise level (its real-world-benchmark setting), refits it
from scratch by L-BFGS-B, and maximises LogEI from RAASP starts (the
best observations perturbed on a random $\min(20/D,1)$ fraction of
coordinates, $\sigma=10^{-3}$; $5$ restarts over $512$ raw samples).
Refitting after every observation makes its per-decision cost grow as
$n^{2.6}$, so its runs stop at $48$\,h between $N=300$ and $800$. It is
omitted on Calorimeter DS3 due to cost. BAxUS
\citep{papenmeier2022increasing} uses the authors' implementation at the
paper's settings: a sparse $\pm1$ embedding whose target dimension starts at
the adjusted $d_0\in\{1,2,3\}$ ($1$, $2$, $3$ and $2$ on our four problems)
and is split fourfold whenever the trust region collapses, a $10$-point Latin
hypercube design per trust region, the TuRBO trust region (initial length
$0.8$, bounds $0.5^7$ and $1.6$, success tolerance $3$, the budget-derived
failure tolerance of the reference) with Thompson sampling over
$\min(100d,5000)$ Sobol perturbations, and a Mat\'ern-$5/2$ ARD GP refit at
every step from $100$ Latin-hypercube hyperparameter starts, the ten best
refined by $50$ Adam steps. Its ambient box is $[0,1]^D$, so the projected-up
point is decoded on the sphere as $\sqrt D\,S^{\top}y/\|S^{\top}y\|$. One deviation from the reference: above $21{,}201$ target
dimensions, reached only on Calorimeter DS3, the candidate perturbations come
from a Latin hypercube instead of the Sobol sequence, which has no direction
numbers there. Its decisions cost $30$--$50$\,s once the embedding has grown to
thousands of dimensions, so its runs stop at $48$\,h between $N=700$ and
$1{,}300$; it runs on all four problems.

\paragraph{O3.}
O3~\citep{willis2025sample} spans a cone of $K$ random latents
of the generative prior and searches the resulting $(K-1)$-dimensional unit
cube through its Knothe--Rosenblatt chart, decoding the rescaled combination
on the sphere. We report $K=90$, the largest setting of the reference, and
$K=300$, each searched by two optimisers. CMA-ES is full-covariance with
$\sigma_0=0.2$ and population $4+\lfloor 3\ln(K-1)\rfloor$. TuRBO uses the
configuration of the TuRBO baseline above, applied in the $(K-1)$-dimensional
cube. The settings are identical at $K=90$ and $K=300$; only the dimension of
the search cube changes. A sweep on the controlled objectives at $D=N=1000$
found a plateau over $K\in[24,90]$ and no gain at $K=300$.

\paragraph{No-schedule ablations.}
The row labelled \cite{costofsampling} is that method in its published form:
the affine model with MAP variance scales, the uncapped posterior, and
unconstrained EI, which prescribes no off-span norm. The ``no schedule'' row
under \emph{Ours} applies our bounded-state compression ($M=300$) and
marginalised variance scales but likewise leaves the off-span norm to EI.

\paragraph{Hardware.}
Campaign runs execute on a shared SLURM cluster as CPU tasks with $4$ cores each; cluster nodes are
dual-socket AMD EPYC 9754 machines (256 cores, 768 GB to 1 TB RAM).

\section{Additional experimental results}
\label{app:experiments}

This appendix provides additional experimental diagnostics,
ablation results, and baseline configurations for
Section~\ref{sec:experiments}.
Objective definitions and SABLE's numerical settings are given
in Appendices~\ref{app:zonal-bench},
\ref{app:benchmark-definitions}, and
\ref{app:experimental-setup}.

\subsection{Evaluation protocol}
\label{app:evaluation-protocol}

Every configuration is repeated over ten independent seeds.
Methods compared within a seed share the same objective instance
on controlled benchmarks; on deterministic generative-model
problems, seeds pair the methods' random streams.
For noisy observations, noise is drawn from a seed-determined
stream, while best-so-far values and recommendations are
scored on the noiseless objective.
Curves show means with 5th--95th percentile bands;
tables report means, with bolding conventions specified
in their captions.

\paragraph{Attained value.}
For controlled objectives, attained value is the better of the
best observed value and the objective value at the posterior-mean
recommendation $\hat{\bm s}$ (Appendix~\ref{app:readout}),
evaluated once after the optimisation run.
Model-free baselines report their best observed value, so this
convention grants model-based methods one additional evaluation.

\subsection{Variance-hyperparameter treatments}
\label{app:hyperparameter-details}

To isolate the effect of hyperparameter fitting, we compare
three treatments of the affine variance scales
$(\sigma_1,\sigma_\varepsilon)$, holding the acquisition
procedure, candidate construction, off-span directions,
and recommendation readout fixed.
MLE maximises the profile log marginal likelihood, with
the intercept profiled out and no prior on the variance scales.
MAP uses the prior and discretisation of
Appendix~\ref{app:defaults}, selecting the hyperposterior mode
with the intercept at its conditional mode.
Both provide a single hyperparameter atom.
Marginalisation instead uses $K=8$ equally weighted atoms
obtained by stratified resampling
(Appendix~\ref{app:atoms}) and averages the acquisition over them.

\subsection{Fitted uncertainty and off-span exploration}
\label{sec:exp-hyperparameters}

We cross the three variance-scale treatments with six
off-span exploration rules: no schedule, a model-derived
schedule, and the fixed-reference ladder
$\lambda_{\mathrm{ref}}\in\{10^0,10^1,10^3,10^5\}$.
The model-derived schedule replaces the fixed reference
in the VoI recursion for $t_N$ by
\[
\lambda^{\mathrm{model}}_N
:=
\frac{1}{K}\sum_{k=1}^K
\frac{\sigma_{1,k}^2}{\sigma_{\varepsilon,k}^2},
\]
with $K=1$ for point estimates. This tests whether fitted
uncertainty can control exploration; it is not part of SABLE.
The noise-free experiments use the uncapped posterior
at $N=D=1000$ with ten paired seeds.
Table~\ref{tab:c2-degree-one-full} reports the endpoints.

\paragraph{Without a schedule.}
MLE drives the fitted discrepancy scale to its numerical
floor, favouring off-span queries: its mean off-span norm
over the last $100$ steps is $0.66$ on the affine control
and $0.74$--$0.83$ on the non-affine objectives.
On the affine control, MLE recovers the degree-one
direction exactly and its recommendation attains the maximum,
but it recovers essentially none of the degree-one component
on non-affine objectives.
MAP and marginalisation behave similarly, indicating
that prior regularisation matters more here than
propagating hyperparameter uncertainty.
With these treatments, higher-order structure inflates
the fitted discrepancy scale, and unconstrained EI
reduces the off-span norm to $0.015$--$0.02$,
confining search to the span established early.

\paragraph{Fixed-reference schedules.}
These schedules prescribe $t_N$ through the finite-horizon
recursion with a fixed $\lambda_{\mathrm{ref}}$, independently
of the fitted variance scales.
With MAP or marginalisation, every fixed reference
improves degree-one recovery on the non-affine objectives.
Attained value increases with $\lambda_{\mathrm{ref}}$
up to $10^3$, with $10^3$ and $10^5$ statistically
indistinguishable. Smaller references keep the off-span
norm high for longer and sacrifice exploitation.
Under MLE, however, scheduling does not restore
degree-one recovery on non-affine objectives.

\paragraph{Model-derived reference.}
For MAP and marginalisation, the
``$\lambda_{\mathrm{ref}}$ from model'' rows perform
similarly to the best fixed reference on the affine control.
On non-affine objectives, higher-order structure inflates
the fitted discrepancy scale, yielding a small reference
and limited exploitation. Degree-one recovery matches
the best fixed reference, but attained value trails
$\lambda_{\mathrm{ref}}=10^3$ by $0.10$--$0.20$.
The fitted ratio is therefore unreliable for controlling
exploration under misspecification.

\paragraph{Noisy observations.}
Table~\ref{tab:c2-degree-one-noisy} repeats the
experiment with independent Gaussian observation
noise of standard deviation $\sigma_\varepsilon = 0.05$.
At $D=1000$, this obscures fine directional changes.
All treatments degrade, and the best fixed references
shift to $\lambda_{\mathrm{ref}}\in\{10^0,10^1\}$,
consistent with the unresolved regime of
Appendix~\ref{sec:local-resolvability}.
The model-derived reference becomes competitive
because the fitted discrepancy scale is now dominated
by observation noise rather than misspecification.
MLE again fails to recover the degree-one component,
with its attained values coming from the best
observed point rather than the recommendation.

\begin{table}[t]
\centering
\caption{\textbf{Complete schedule-by-treatment grid, noise-free}, at $N=D=1000$ with the
uncapped posterior. \emph{Left:} recovery of the known degree-one component,
$\cos(\hat{\bm s},\bm v_\star)$. \emph{Right:} attained value, the better of the best
observed point and the posterior-mean recommendation. Means over ten paired seeds.
Bold: the column best and every entry within two standard deviations of it.
Table~\ref{tab:c2-degree-one-main} shows the no-schedule, model-derived, and
$\lambda_{\mathrm{ref}}=10^{5}$ rows.}
\label{tab:c2-degree-one-full}
\footnotesize
\begin{tabular}{llcccccccc}
\toprule
& & \multicolumn{4}{c}{$\cos(\hat{\bm s},\bm v_\star)$}
  & \multicolumn{4}{c}{attained value} \\
\cmidrule(lr){3-6}\cmidrule(lr){7-10}
& & affine & mild & mod. & severe & affine & mild & mod. & severe \\
\midrule
 \multirow{3}{*}{no schedule} & MLE          & \textbf{1.00} & -0.00 & -0.01 & -0.01 & \textbf{1.00} & 0.48 & 0.32 & 0.16 \\
  & MAP          & 0.78 & 0.44 & 0.30 & 0.16 & 0.78 & 0.60 & 0.54 & 0.50 \\
  & Marg.        & 0.77 & 0.44 & 0.29 & 0.19 & 0.77 & 0.59 & 0.54 & 0.49 \\
\midrule
 \multirow{3}{*}{$\lambda_{\mathrm{ref}}$ from model} & MLE          & 0.13 & 0.00 & 0.00 & 0.01 & 0.49 & 0.49 & 0.36 & 0.23 \\
  & MAP          & 0.98 & \textbf{0.59} & \textbf{0.41} & \textbf{0.23} & 0.98 & 0.67 & 0.55 & 0.45 \\
  & Marg.        & 0.98 & \textbf{0.59} & \textbf{0.42} & \textbf{0.25} & 0.98 & 0.66 & 0.56 & 0.45 \\
\midrule
 \multirow{3}{*}{$\lambda_{\mathrm{ref}}=10^{0}$} & MLE          & 0.22 & -0.01 & 0.02 & 0.00 & 0.87 & 0.55 & 0.39 & 0.24 \\
  & MAP          & 0.86 & \textbf{0.55} & \textbf{0.38} & \textbf{0.23} & 0.86 & 0.61 & 0.50 & 0.43 \\
  & Marg.        & 0.87 & \textbf{0.55} & \textbf{0.39} & \textbf{0.23} & 0.87 & 0.60 & 0.52 & 0.41 \\
\midrule
 \multirow{3}{*}{$\lambda_{\mathrm{ref}}=10^{1}$} & MLE          & 0.12 & 0.00 & -0.01 & 0.00 & 0.96 & 0.60 & 0.42 & 0.29 \\
  & MAP          & 0.96 & \textbf{0.59} & \textbf{0.41} & \textbf{0.24} & 0.96 & 0.70 & 0.61 & 0.54 \\
  & Marg.        & 0.97 & \textbf{0.58} & \textbf{0.42} & \textbf{0.24} & 0.97 & 0.69 & 0.62 & 0.54 \\
\midrule
 \multirow{3}{*}{$\lambda_{\mathrm{ref}}=10^{3}$} & MLE          & \textbf{1.00} & -0.02 & -0.00 & -0.01 & \textbf{1.00} & 0.58 & 0.40 & 0.26 \\
  & MAP          & 0.97 & \textbf{0.58} & \textbf{0.40} & \textbf{0.24} & 0.98 & \textbf{0.75} & \textbf{0.68} & \textbf{0.65} \\
  & Marg.        & 0.98 & \textbf{0.59} & \textbf{0.41} & \textbf{0.25} & 0.98 & \textbf{0.76} & \textbf{0.70} & \textbf{0.65} \\
\midrule
 \multirow{3}{*}{$\lambda_{\mathrm{ref}}=10^{5}$} & MLE          & 0.90 & -0.00 & 0.01 & 0.00 & \textbf{1.00} & 0.55 & 0.37 & 0.24 \\
  & MAP          & \textbf{0.99} & \textbf{0.56} & \textbf{0.38} & \textbf{0.21} & \textbf{0.99} & \textbf{0.75} & \textbf{0.68} & \textbf{0.63} \\
  & Marg.        & \textbf{0.99} & \textbf{0.57} & \textbf{0.38} & \textbf{0.23} & \textbf{0.99} & \textbf{0.74} & \textbf{0.67} & \textbf{0.63} \\
\bottomrule
\end{tabular}
\end{table}

\begin{table}[t]
\centering
\caption{\textbf{Complete schedule-by-treatment grid, noisy observations}
($\sigma_\varepsilon=0.05$), at $N=D=1000$ with the uncapped posterior; layout and
bold rule as in Table~\ref{tab:c2-degree-one-full}. Best-so-far and recommendations
are scored on the noiseless objective.}
\label{tab:c2-degree-one-noisy}
\footnotesize
\begin{tabular}{llcccccccc}
\toprule
& & \multicolumn{4}{c}{$\cos(\hat{\bm s},\bm v_\star)$}
  & \multicolumn{4}{c}{attained value} \\
\cmidrule(lr){3-6}\cmidrule(lr){7-10}
& & affine & mild & mod. & severe & affine & mild & mod. & severe \\
\midrule
 \multirow{3}{*}{no schedule} & MLE          & 0.01 & 0.02 & 0.01 & 0.01 & 0.37 & 0.26 & 0.19 & 0.13 \\
  & MAP          & 0.14 & 0.10 & 0.07 & 0.06 & 0.15 & 0.15 & 0.15 & 0.16 \\
  & Marg.        & 0.16 & 0.11 & 0.08 & \textbf{0.05} & 0.16 & 0.16 & 0.16 & 0.15 \\
\midrule
 \multirow{3}{*}{$\lambda_{\mathrm{ref}}$ from model} & MLE          & -0.00 & -0.00 & -0.00 & -0.01 & \textbf{0.48} & \textbf{0.33} & \textbf{0.24} & 0.17 \\
  & MAP          & \textbf{0.36} & \textbf{0.27} & \textbf{0.19} & \textbf{0.11} & 0.36 & \textbf{0.30} & \textbf{0.25} & \textbf{0.21} \\
  & Marg.        & \textbf{0.36} & \textbf{0.26} & \textbf{0.18} & \textbf{0.11} & 0.36 & \textbf{0.29} & \textbf{0.24} & \textbf{0.22} \\
\midrule
 \multirow{3}{*}{$\lambda_{\mathrm{ref}}=10^{0}$} & MLE          & -0.02 & 0.01 & -0.00 & -0.00 & \textbf{0.44} & \textbf{0.32} & \textbf{0.23} & 0.15 \\
  & MAP          & \textbf{0.37} & \textbf{0.26} & \textbf{0.20} & \textbf{0.12} & 0.38 & \textbf{0.31} & \textbf{0.28} & \textbf{0.23} \\
  & Marg.        & \textbf{0.37} & \textbf{0.28} & \textbf{0.20} & \textbf{0.12} & 0.38 & \textbf{0.31} & \textbf{0.26} & \textbf{0.22} \\
\midrule
 \multirow{3}{*}{$\lambda_{\mathrm{ref}}=10^{1}$} & MLE          & -0.02 & -0.00 & -0.01 & -0.00 & \textbf{0.48} & \textbf{0.34} & \textbf{0.26} & 0.18 \\
  & MAP          & \textbf{0.32} & \textbf{0.24} & \textbf{0.18} & \textbf{0.10} & 0.33 & \textbf{0.31} & \textbf{0.29} & \textbf{0.26} \\
  & Marg.        & \textbf{0.33} & \textbf{0.25} & \textbf{0.18} & \textbf{0.11} & 0.34 & \textbf{0.31} & \textbf{0.28} & \textbf{0.27} \\
\midrule
 \multirow{3}{*}{$\lambda_{\mathrm{ref}}=10^{3}$} & MLE          & -0.02 & 0.03 & \textbf{0.02} & 0.00 & \textbf{0.47} & \textbf{0.32} & \textbf{0.22} & 0.16 \\
  & MAP          & 0.17 & 0.14 & 0.10 & 0.05 & 0.18 & 0.20 & 0.21 & \textbf{0.19} \\
  & Marg.        & 0.19 & 0.14 & 0.09 & 0.06 & 0.20 & 0.21 & 0.20 & 0.21 \\
\midrule
 \multirow{3}{*}{$\lambda_{\mathrm{ref}}=10^{5}$} & MLE          & -0.01 & -0.00 & -0.00 & 0.01 & \textbf{0.46} & \textbf{0.32} & 0.23 & 0.17 \\
  & MAP          & 0.14 & 0.11 & 0.10 & 0.05 & 0.15 & 0.17 & 0.17 & 0.17 \\
  & Marg.        & 0.15 & 0.11 & 0.08 & 0.04 & 0.16 & 0.16 & 0.16 & 0.17 \\
\bottomrule
\end{tabular}
\end{table}

\subsection{Endpoint values at $D=N=10{,}000$}
\label{app:c4-endpoint}

Table~\ref{tab:c4-schedule} tabulates the endpoints of the runs shown in
Figure~\ref{fig:c4-rank-lambda}, extended to the full reference ladder
$\lambda_{\mathrm{ref}}\in\{10^{1},10^{3},10^{5}\}$ at $M=300$ with
marginalised variance scales, together with the no-schedule ablation and the
comparisons. Thompson sampling and the two design-then-fit baselines use the
same posterior as our arms. The endpoint is the attained value of
Appendix~\ref{app:evaluation-protocol}, which the model-free baselines
cannot use, since they expose no recommendation.

\begin{table}[ht]
\centering
\caption{\textbf{Attained value at $D=N=10{,}000$, $M=300$}: the better of the best
observed point and the posterior-mean recommendation, mean and $5$--$95\%$ over ten
paired seeds. Bold: the column best and all within two standard deviations. The
model-free baselines expose no recommendation and report their best observed point.
Prescribing the off-span schedule beats every comparison on every non-affine objective.}
\label{tab:c4-schedule}
\footnotesize
\setlength{\tabcolsep}{4pt}%
\begin{tabular}{lcccc}
\toprule
& affine & mild & mod. & severe \\
\midrule
\multicolumn{5}{l}{\emph{Comparisons}} \\
Random design + posterior argmax & 0.93 {\scriptsize [0.92, 0.93]} & 0.46 {\scriptsize [0.45, 0.47]} & 0.31 {\scriptsize [0.29, 0.33]} & 0.17 {\scriptsize [0.16, 0.19]} \\
Orthogonal design + posterior argmax & \textbf{1.00} {\scriptsize [1.00, 1.00]} & 0.45 {\scriptsize [0.42, 0.47]} & 0.28 {\scriptsize [0.25, 0.31]} & 0.13 {\scriptsize [0.12, 0.14]} \\
sep-CMA-ES & 0.60 {\scriptsize [0.59, 0.61]} & 0.50 {\scriptsize [0.49, 0.51]} & 0.48 {\scriptsize [0.46, 0.49]} & 0.45 {\scriptsize [0.44, 0.46]} \\
Thompson sampling & 0.79 {\scriptsize [0.78, 0.79]} & 0.49 {\scriptsize [0.48, 0.51]} & 0.35 {\scriptsize [0.30, 0.38]} & 0.20 {\scriptsize [0.11, 0.27]} \\
\midrule
\multicolumn{5}{l}{\emph{Ours (marginalised, $M=300$)}} \\
no schedule & 0.53 {\scriptsize [0.52, 0.54]} & 0.44 {\scriptsize [0.44, 0.45]} & 0.40 {\scriptsize [0.38, 0.42]} & 0.36 {\scriptsize [0.35, 0.38]} \\
SABLE, $\lambda_{\mathrm{ref}} = 10^{1}$ & 0.95 {\scriptsize [0.94, 0.95]} & 0.64 {\scriptsize [0.63, 0.65]} & 0.57 {\scriptsize [0.57, 0.58]} & 0.50 {\scriptsize [0.49, 0.51]} \\
SABLE, $\lambda_{\mathrm{ref}} = 10^{3}$ & 0.96 {\scriptsize [0.95, 0.97]} & \textbf{0.69} {\scriptsize [0.68, 0.70]} & \textbf{0.65} {\scriptsize [0.64, 0.66]} & 0.61 {\scriptsize [0.60, 0.62]} \\
SABLE, $\lambda_{\mathrm{ref}} = 10^{5}$ & 0.97 {\scriptsize [0.96, 0.97]} & \textbf{0.70} {\scriptsize [0.70, 0.72]} & \textbf{0.66} {\scriptsize [0.65, 0.67]} & \textbf{0.62} {\scriptsize [0.61, 0.63]} \\
\bottomrule
\end{tabular}
\end{table}

\subsection{Low-rank posterior experiments}
\label{app:low-rank-experiments}
We compare compression budgets $M\in\{10,100,300,500\}$
with the uncapped posterior, holding all other settings fixed,
including $\lambda_{\mathrm{ref}}=10^5$ and marginalised
variance scales. The global-maximum reference in Figure~\ref{fig:c3-rank}
averages the maxima of ten paired objective draws. It is
exact for the affine control; for the Mercer objectives,
it is a lower bound obtained by gradient ascent on the sphere
(16 random restarts, 600 Adam steps, learning rate $0.02$).
The remaining headroom may therefore be understated.
The uncapped posterior uses the same paired draws and is
compared at $N=5{,}000$, the largest budget reached by all
ten seeds, while compressed runs continue to $N=D=10{,}000$.
Its per-decision cost at $N=D$ is measured separately
(Figure~\ref{fig:c3-rank}, right), without completing
an optimisation run at that budget.

\subsection{Objectives with negligible degree-one signal}
\label{app:zero-linear-signal}

We further test prescribed off-span exploration when the
objective has no degree-one signal in its spectral
construction. Using ZonalFunctionBench
(Appendix~\ref{app:zonal-bench}), we construct two
additional objectives with spectral weight on degrees
2 and 4 only (\emph{zero-even}), or degrees 2 and 3
only (\emph{zero-orth.}). The measured degree-one
share of the resulting objectives is below $10^{-4}$.

We use the setting of Table~\ref{tab:c2-degree-one-main}:
$D=N=1000$, the uncapped posterior, and ten paired
seeds. We compare MLE, MAP, and hyperparameter
marginalisation under standard EI, the model-derived
schedule, and the prescribed VoI schedule with
$\lambda_{\mathrm{ref}}=10^5$. Since the reference
degree-one signal is negligible, we report attained
value rather than directional recovery.

\begin{table}[t]
\centering
\caption{\textbf{Negligible degree-one signal.}
Attained values on two controlled objectives with
zero degree-one spectral weight and measured degree-one
share below $10^{-4}$. Means over ten paired seeds
at $D=N=1000$, without compression. Bold: the column
best and entries within two standard deviations.
}
\label{tab:zero-linear-signal}
\footnotesize
\begin{tabular}{llcc}
\toprule
Exploration rule & Hyperparameters
    & Zero-even & Zero-orth. \\
\midrule
\multirow{3}{*}{Standard EI (no schedule)}
    & MLE   & 0.11 & 0.13 \\
    & MAP   & 0.45 & 0.51 \\
    & Marg. & 0.43 & 0.52 \\
\midrule
\multirow{3}{*}{Model-derived schedule}
    & MLE   & 0.16 & 0.17 \\
    & MAP   & 0.40 & 0.43 \\
    & Marg. & 0.40 & 0.42 \\
\midrule
\multirow{3}{*}{VoI schedule, $\lambda_{\mathrm{ref}}=10^5$}
    & MLE   & 0.15 & 0.18 \\
    & MAP   & \textbf{0.59} & \textbf{0.65} \\
    & Marg. & \textbf{0.60} & \textbf{0.65} \\
\bottomrule
\end{tabular}
\end{table}

Table~\ref{tab:zero-linear-signal} shows that prescribed
exploration improves optimisation even when the
degree-one signal is negligible. Under MAP,
the attained value increases from $0.45$ to $0.59$
on zero-even and from $0.51$ to $0.65$ on zero-orth.,
relative to standard EI. Marginalisation gives similar
improvements, from $0.43$ to $0.60$ and from $0.52$
to $0.65$, respectively. The model-derived schedule
does not provide comparable gains.
These controls show that the empirical benefit of
prescribed off-span exploration is not confined to
recovering a degree-one component. They do not,
however, establish guarantees for arbitrary objectives
without affine structure.

\end{document}